\documentclass[11pt]{article} 

\usepackage[letterpaper,margin=1in]{geometry}
\usepackage{titling}
\pretitle{\vspace{-1.7cm}\begin{center}\LARGE}

\usepackage[T1]{fontenc}
\usepackage{microtype}
\usepackage{wrapfig}

\usepackage{titlesec}
\titleformat*{\paragraph}{\bfseries}

\usepackage{amsthm,amsmath,amssymb,amsfonts,amssymb,mathtools}
\usepackage{amsmath,amssymb,amsfonts,amssymb,mathtools}
\usepackage{xcolor}
\usepackage{empheq}
\usepackage{dsfont}
\usepackage{mathrsfs}
\usepackage{graphicx}
\usepackage{enumitem}
\usepackage{aliascnt} 
\usepackage{xspace}
\usepackage{booktabs}

\usepackage{caption}
\usepackage{tikz}
\usetikzlibrary{patterns}
\usetikzlibrary{arrows,shapes,automata,backgrounds,petri,positioning}
\usetikzlibrary{shadows}
\usetikzlibrary{calc}
\usetikzlibrary{spy}
\usetikzlibrary{angles, quotes}
\usetikzlibrary{matrix}
\usepackage{pgf,pgfplots}
\usepackage{pgfmath,pgffor}
\pgfplotsset{compat=1.17}

\usepackage{framed}
\usepackage{algorithm2e}
\usepackage[noend]{algpseudocode}
\SetAlCapFnt{\normalfont\small}\SetAlCapNameFnt{\unskip\itshape\small}

\definecolor[named]{ACMBlue}{cmyk}{1,0.1,0,0.1}
\definecolor[named]{ACMYellow}{cmyk}{0,0.16,1,0}
\definecolor[named]{ACMOrange}{cmyk}{0,0.42,1,0.01}
\definecolor[named]{ACMRed}{cmyk}{0,0.90,0.86,0}
\definecolor[named]{ACMLightBlue}{cmyk}{0.49,0.01,0,0}
\definecolor[named]{ACMGreen}{cmyk}{0.20,0,1,0.19}
\definecolor[named]{ACMPurple}{cmyk}{0.55,1,0,0.15}
\definecolor[named]{ACMDarkBlue}{cmyk}{1,0.58,0,0.21}

\usepackage[colorlinks,citecolor=blue,linkcolor=magenta,bookmarks=true]{hyperref}
\usepackage[nameinlink]{cleveref}
\crefname{ineq}{Inequality}{Inequality}
\creflabelformat{ineq}{#2{\upshape(#1)}#3}
\crefname{sub}{Subsection}{Subsection}
\creflabelformat{Subsection}{#2{\upshape(#1)}#3}
\crefname{sdp}{SDP}{SDP}
\creflabelformat{sdp}{#2{\upshape(#1)}#3}
\crefname{lp}{LP}{LP}
\creflabelformat{lp}{#2{\upshape(#1)}#3}
\usepackage{aliascnt}
\usepackage{cleveref}
\crefname{ineq}{Inequality}{Inequality}
\creflabelformat{ineq}{#2{\upshape(#1)}#3}
\crefname{sub}{Subsection}{Subsection}
\creflabelformat{Subsection}{#2{\upshape(#1)}#3}
\crefname{sdp}{SDP}{SDP}
\creflabelformat{sdp}{#2{\upshape(#1)}#3}
\crefname{lp}{LP}{LP}
\creflabelformat{lp}{#2{\upshape(#1)}#3}

\makeatletter
\newenvironment{Ualgorithm}[1][htpb]{\def\@algocf@post@ruled{\kern\interspacealgoruled\hrule  height\algoheightrule\kern3pt\relax}\def\@algocf@capt@ruled{under}\setlength\algotitleheightrule{0pt}\SetAlgoCaptionLayout{centerline}\begin{algorithm}[#1]}
{\end{algorithm}}
\makeatother

\newcommand{\CS}{Cauchy-Schwarz\xspace}

\newtheorem{theorem}{Theorem}[section]
\newtheorem{question}[theorem]{Question}

\newtheorem{lemma}[theorem]{Lemma}
\newtheorem{informal theorem}[theorem]{Theorem (informal statement)}

\newtheorem{proposition}[theorem]{Proposition}
\newtheorem{corollary}[theorem]{Corollary}
\newtheorem{claim}[theorem]{Claim}
\newtheorem{fact}[theorem]{Fact}

\newtheorem{remark}[theorem]{Remark}

\newtheorem{definition}[theorem]{Definition}
\newcommand{\eqdef}{\coloneqq}

\newcommand{\lp}{\left}
\newcommand{\rp}{\right}

\renewcommand\vec[1]{\mathbf{#1}}
\DeclareMathOperator*{\pr}{\mathbf{Pr}}
\DeclareMathOperator*{\E}{\mathbf{E}}
\newcommand{\proj}{\mathrm{proj}}

\newcommand{\normal}{\mathcal{N}}

\DeclareMathOperator*{\argmin}{argmin}

\newcommand{\tr}{\mathrm{tr}}
\newcommand{\bx}{\mathbf{x}}
\newcommand{\by}{\mathbf{y}}

\newcommand{\bw}{\mathbf{w}}

\newcommand{\err}{\mathrm{err}}

\newcommand{\R}{\mathbb{R}}

\newcommand{\Z}{\mathbb{Z}}
\newcommand{\N}{\mathbb{N}}
\newcommand{\eps}{\epsilon}

\newcommand{\poly}{\mathrm{poly}}
\newcommand{\polylog}{\mathrm{polylog}}
\newcommand{\var}{\mathbf{Var}}

\newcommand{\sgn}{\mathrm{sign}}
\newcommand{\sign}{\mathrm{sign}}

\newcommand{\opt}{\mathrm{opt}}
\newcommand{\D}{D}

\newcommand{\Ind}{\mathds{1}}
\newcommand{\1}{\Ind}

\newcommand{\littlesum}{\mathop{\textstyle \sum}}

\newcommand{\wt}{\widetilde}
\newcommand{\wh}{\widehat}

\newcommand{\x}{\vec x}
\newcommand{\z}{\vec z}

\newcommand{\Exn}{\E_{\x\sim \normal}}

\newcommand{\citet}{\cite}
\newcommand{\citep}{\cite}
\newcommand{\Pm}[1]{\mathrm{P}_{#1}}
\newcommand{\mderiv}{D_{\rho}}
\newcommand{\valb}{M}
\newcommand{\gradb}{L}
\newcommand{\OU}{Ornstein–Uhlenbeck\xspace}

\newcommand{\subvector}{h}

\title{Agnostically Learning Multi-index Models with Queries}
\author{
Ilias Diakonikolas\thanks{Supported by NSF Medium Award CCF-2107079,
NSF Award CCF-1652862 (CAREER), a Sloan Research Fellowship, and
a DARPA Learning with Less Labels (LwLL) grant.}\\
UW Madison\\
{\tt ilias@cs.wisc.edu}\\
\and
Daniel M. Kane\thanks{Supported in part by NSF Award CCF-2144298 (CAREER).}\\
UCSD\\
{\tt dakane@ucsd.edu }\\
\and
Vasilis Kontonis\thanks{This work was done at UW Madison, supported in part by NSF Award CCF-2144298 (CAREER).}\\
UW Madison\\
{\tt kontonis@wisc.edu }\\
\and
Christos Tzamos\thanks{Supported by NSF Award CCF-2144298 (CAREER).}\\
UW Madison\\
{\tt tzamos@wisc.edu}
\and
Nikos Zarifis\thanks{Supported in part by NSF Award CCF-1652862 (CAREER) 
and a DARPA Learning with Less Labels (LwLL) grant.}\\
UW Madison\\
{\tt zarifis@wisc.edu}\\
}
\author{
Ilias Diakonikolas\thanks{Supported by NSF Medium Award CCF-2107079,
		NSF Award CCF-1652862 (CAREER), and
		a DARPA Learning with Less Labels (LwLL) grant.}\\
UW Madison\\
{\tt ilias@cs.wisc.edu}\\
\and
Daniel M. Kane\thanks{Supported by NSF Medium Award CCF-2107547 and NSF Award CCF-1553288 (CAREER).}\\
UCSD\\
{\tt dakane@ucsd.edu }\\
\and
Vasilis Kontonis\thanks{This research was done at UW Madison. 
Supported in part by NSF Award CCF-2144298 (CAREER).}\\
UT Austin\\
{\tt vasilis@cs.utexas.edu } \\
\and
Christos Tzamos\thanks{This research was done at UW Madison. Supported by NSF Award CCF-2144298 (CAREER).}\\
UW Madison \& University of Athens\\
{\tt tzamos@wisc.edu}
\and
Nikos Zarifis\thanks{Supported by NSF Medium Award CCF-2107079, and a DARPA Learning with Less Labels (LwLL) grant.}\\
UW Madison\\
{\tt zarifis@wisc.edu}\\
}

\begin{document}

\maketitle

\begin{abstract}
We study the power of query access for the fundamental task of 
agnostic learning under the Gaussian distribution.
In the agnostic model, no assumptions are 
made on the labels of the examples and the goal is to compute a hypothesis 
that is competitive with the {\em best-fit} function in a known class, i.e., 
it achieves error $\opt+\eps$, where $\opt$ 
is the error of the best function in the class.
We focus on a general family of Multi-Index Models (MIMs), 
which are $d$-variate functions that depend only on few relevant directions, 
i.e., have the form $g(\vec W \x)$ for an 
unknown link function $g$ and a $k \times d$ matrix $\vec W$.
Multi-index models cover a wide range of commonly studied function classes, 
including real-valued function classes 
such as constant-depth neural networks with ReLU activations, 
and Boolean concept classes such as intersections of halfspaces.

Our main result shows that query access gives significant runtime improvements 
over random examples for agnostically learning 
both real-valued and Boolean-valued MIMs. 
Under standard regularity assumptions for the link function 
(namely, bounded variation or surface area), 
we give an agnostic query learner for MIMs 
with running time  $O(k)^{\poly(1/\eps)} \; \poly(d) $. 
In contrast, algorithms that rely only on random labeled examples 
inherently require $d^{\poly(1/\epsilon)}$ samples and runtime, 
even for the basic problem of agnostically 
learning a single ReLU or a halfspace. 
As special cases of our general approach, 
we obtain the following results:
\begin{itemize}[leftmargin=*]
\item For the class of depth-$\ell$, width-$S$ ReLU networks on
$\R^d$, our agnostic query learner runs in time $\poly(d) 2^{\poly(\ell S/\eps)}$.
This bound qualitatively matches the runtime of an algorithm by~\cite{CKM22} 
for the realizable PAC setting with random examples.

\item For the class of arbitrary intersections of $k$ halfspaces on $\R^d$, 
our agnostic query learner runs in time $\poly(d) \, 2^{\poly(\log (k)/\eps)}$. 
Prior to our work, no improvement over the agnostic PAC model complexity 
(without queries) was known, even for the case of a single halfspace. 
\end{itemize}
In both these settings, we provide evidence
that the $2^{\poly(1/\eps)}$ runtime dependence is required 
for proper query learners, even for agnostically learning a single ReLU or halfspace. 

In summary, our algorithmic result establishes a strong computational 
separation between 
the agnostic PAC and the agnostic PAC+Query models under the Gaussian distribution. 
Prior to our work, no such separation was known --- 
even for the special case of agnostically learning a single halfspace, 
for which it was an open problem first posed by Feldman~\cite{Feldman08}.
Our results are enabled by a general dimension-reduction technique 
that leverages query access to estimate gradients of (a smoothed version of) 
the underlying label function.
\end{abstract}

\setcounter{page}{0}
\thispagestyle{empty}

\newpage
\tableofcontents
\setcounter{page}{0}
\thispagestyle{empty}
\newpage

\section{Introduction}

\vspace{-0.2cm} 

\paragraph{PAC Learning with Queries}  
In Valiant's PAC learning model~\cite{Valiant:84,val84}, the learner is given access to random examples 
labeled according to an unknown function in a known concept class. The goal of the learner is to compute a 
hypothesis that is close to the target function with respect to a specified loss function\footnote{For Boolean functions, one typically uses the 0-1 loss, while for real-valued functions a typical choice is the $L_2$ loss.}. 
The standard PAC learning model is 
``passive'' in that the learning algorithm 
has no control over the selection of the training set. Interestingly, 
while this has become known as {\em the} PAC model, 
Valiant's landmark paper~\cite{val84} allowed queries (in addition to 
random samples), i.e., black-box access to the target function.  We will refer
to this as PAC+Query model.

A {\em query oracle}\footnote{In the special case of learning Boolean-valued functions, 
these are known as ``membership'' queries, as the answer to a query determines membership 
in the set of satisfying assignments of the target concept.} 
allows the learner to obtain the value of the target function on any 
desired point in the domain. PAC learning with access to a query oracle 
can be viewed as an ``active'' learning model, 
intuitively capturing the ability to perform experiments or the availability of expert advice.  A 
long line of research in computational learning theory has explored the power of queries 
in the context of PAC learning. This line of investigation has spanned 
the distribution-free versus distribution-specific settings 
and the realizable (i.e., clean label) setting
versus the agnostic (i.e., adversarial label noise) setting;  
see, e.g.,~\cite{Angluin:87, GoldreichLevin:89, 
KushilevitzMansour:93, Jackson:97} for some classical early works 
and~\cite{GKK:08, BlancLQT22} for some more recent results in this broad area. 
A conceptual message of this line of work is that, in the realizable setting, 
access to queries can be stronger than random samples 
(from a computational standpoint)
for a range of natural concept classes.

In addition to being a fundamental open question in learning theory, the general problem of understanding the effect of query access 
 in the {\em computational complexity} of learning
has received renewed attention over the past decade 
in the context of deep neural networks. A recent 
line of inquiry from the machine learning security community has studied {\em model extraction 
attacks} --- see, e.g.,~\cite{TramerZJRR16, SSG17, PapernotMGJCS17, MilliSDH19,  
JagielskiCBKP20, RolnickK20, JWZ20} and references therein --- 
where black-box query access to publicly deployed networks may 
allow efficient reconstruction of the hidden model -- thus exposing potential vulnerability of the 
deployed models. These practical applications served as a motivation for  
the design of the first computationally efficient learners for simple neural networks 
using query access to the target function~\cite{ChenKM21, DG22}. Importantly, the latter 
algorithmic results apply in the realizable PAC model under the Gaussian distribution.

\vspace{-0.2cm}

\paragraph{Multi-index Function Models (MIMs)} 
A common (semi)-parametric modeling assumption in high-
dimensional statistics is that the target function 
depends only on a few relevant directions. 
Specifically, multi-index models~\cite{Friedman:1980tu, Huber85-pp, Li91, HL93,xia2002adaptive, Xia08}
prescribe that the target function is of the form 
$f(\x) = g(\vec W \x)$ for a link function $g: \R^k \mapsto \R$ and a $k \times d$ weight matrix $\vec W$. In most settings, the link function $g$ is assumed to be unknown and satisfies certain smoothness properties.
Single-index models are the special case where
the target function depends only on a single hidden-direction 
$\vec w$, i.e., $f(\x) = g(\vec w \cdot \x)$ for some $g:\R \mapsto \R$ and  
$\vec w \in \R^d$~\cite{ichimura1993semiparametric, hristache2001direct,
hardle2004nonparametric, dalalyan2008new}.

Multi-index models capture a wide range of parametric models studied 
in the statistics and computer science literatures, 
including neural networks and classes of geometric Boolean functions 
(e.g., intersections of halfspaces).
An extensive recent line of work
\cite{janzamin2015beating,GeLM18,DudejaH18, BakshiJW19,GeKLW19, DKKZ20, ChenM20, 
damian2022neural, bietti2022learning, ChenDGKM23, CN23, DK23-nn} 
have studied the efficient learnability of (natural classes of) MIMs from random examples 
under well-behaved marginal distributions --- 
most notably under the Gaussian distribution on examples. 
The aforementioned works exclusively focus on the PAC model with random samples 
and the underlying algorithms succeed in the realizable setting 
(or in the presence of additive Gaussian label noise).

\paragraph{This Work: Agnostically Learning Multi-index Models with Queries}
Here we study the power of queries in the {\em agnostic} PAC 
model~\cite{Haussler:92, KSS:94} for a wide class of multi-index models. In the agnostic model, no assumptions are 
made on the labels of the examples and the goal is to compute a hypothesis 
that is competitive with the {\em best-fit} function in a known class.
This is a notoriously challenging model of learning 
with very few positive results in the distribution-free setting. For 
example, it is known that even {\em weak} (distribution-free) agnostic 
learning (i.e., outputting a hypothesis with non-trivial advantage over 
random) is computationally hard for very simple classes of single-index models with known link functions. These include linear threshold gates 
and single neurons with ReLU activations~\cite{Daniely16, DKMR22, DKMR22-massart, Tiegel22}\footnote{We note that these computational 
hardness results hold even with query access, as follows from~\cite{Feldman08}.}. 

In this work, we focus on the general problem of agnostically 
learning multi-index models under the standard Gaussian 
distribution using queries. At a high-level, our results 
also encompass the challenging setting where the link function is {\em unknown} 
and only require an average smoothness condition on the target function. Classes covered by our framework include 
real-valued function classes such as constant-depth neural networks with ReLU activations and Boolean concept classes 
such as intersections of halfspaces.
In summary, we are interested in the following question:

\begin{question} \label{qn:main}
Does {\em query access} affect the complexity of 
distribution-specific agnostic learning of multi-index models? 
In particular, does the availability of queries allow for {\em qualitatively} more 
efficient algorithms, compared to the vanilla random example setting?
\end{question}

\noindent The main contribution of this paper is a simple and general methodology 
that answers this question in the 
affirmative for a broad family of multi-index function models 
(including all the aforementioned examples). 

A special case of \Cref{qn:main} was explicitly asked 
--- in the Boolean setting --- 
for the class of Linear Threshold Functions 
by Feldman~\cite{Feldman08} 
and by Gopalan, Kalai, and Klivans~\cite{GKK08-open}
As a corollary of our approach, we answer this open question. 
Specifically, we provide a 
new query algorithm for agnostically learning halfspaces 
implying a super-polynomial separation between the two learning models 
(learning with random samples versus with queries), 
subject to standard cryptographic assumptions. 
In the following subsection, we describe our contributions in detail.

\subsection{Our Results} \label{ssec:results}

\paragraph{Problem Definition} 
Before we formally state our main results, 
we define the agnostic learning model with queries.
For concreteness, \Cref{def:mq-agnostic} 
concerns real-valued functions, 
where the accuracy is measured with respect to the $L_2$ loss. 
The definition for Boolean-valued concepts is essentially identical, where 
the $L_2$ loss is replaced by the 0-1 loss.

\begin{definition}[Agnostically Learning Real-valued Functions with Queries] 
\label{def:mq-agnostic} 
Fix $\epsilon \in (0, 1)$ and a class $\cal{C}$ of real-valued functions on $\R^d$.
The adversary picks a label function $y(\x) \in \R$ 
for every $\x \in \R^d$. 
The learner is allowed to either
draw $\x \sim \normal$ (sample access)
or select any desired point $\x \in \R^d$ (query access)
and obtain the value $y(\x)$.  
Let $N_s \in \Z_+$ be the number of samples and 
$N_{q} \in \Z_+$ the number of queries used by the learner. 
The goal of the learner is
to output a hypothesis $h: \R^d \to \R$ that, with high probability, 
has excess $L_2^2$ error at most $\eps$ (with respect to ${\cal C}$), i.e., 
it satisfies 
$\mathcal{E}_{2}(h, \mathcal{C}; y) \eqdef \E_{\x \sim \normal}[(h(\x) - y(\x))^2] -
\inf_{c \in \mathcal{C}}\E_{\x \sim \normal}[(c(\x) - y(\x))^2] \leq \epsilon \;.$
\end{definition}

\begin{remark}[Boolean-valued Functions] \label{rem:bv-def}
{\em In the boolean-valued setting, we focus on learning with respect to the 0-1 loss.  
That is, the goal of the learner is to output a hypothesis $h: \R^d \mapsto \{ \pm 1\}$
with excess 0-1 error at most $\eps$, i.e., 
$
\mathcal{E}_{0/1}(h, \mathcal{C}; y) \eqdef \pr_{\x \sim \normal}[h(\x) \neq y(\x)] -\inf_{ c \in \mathcal{C}}\pr_{\x \sim \normal}[c(\x) \neq y(\x)] \leq \epsilon \,.
$}
\end{remark}

\subsubsection{Agnostically Learning Real-valued Multi-index Models} 

We start by describing the family of multi-index models for which 
our results are applicable. Roughly speaking, our algorithmic approach can be used
to agnostically learn any family of multi-index models $\mathcal{C}$ 
such that any function in $\mathcal{C}$ has ``bounded variation'', 
in the sense that the $L_2$-norm of its gradient
is bounded with respect to the standard normal. 
We remark that similar ``smoothness'' assumptions, i.e., 
that $f$ belongs in a Sobolev space,
are standard (and necessary) in non-parametric and semi-parametric regression~\cite{Tsybakov08}.
Under this assumption, we show that
there exists an \emph{efficient dimension-reduction} scheme that yields
a ``fixed parameter tractable'' agnostic learner significantly improving over
the best known algorithmic results in the agnostic PAC setting with random examples. 

We are now ready to formally define the semi-parametric 
class of MIMs that we consider in this work.  In the following definition, 
we require that the target function is bounded in $L_4$-norm 
(with respect to the standard normal distribution) 
and also that the norm of its gradient is bounded in $L_2$-norm.  
\begin{definition}[Bounded Variation Multi-index Models]
\label{def:bounded-variation-concepts}
Fix $L, M>0$ and $k \in \Z_+$. 
We define the class $\mathfrak R(\valb, \gradb, k)$ 
of continuous, (almost everywhere) differentiable real-valued functions
such that for every $f \in \mathfrak R(\valb, \gradb, k)$:
\begin{enumerate}[leftmargin=*]
\item 
It holds $(\E_{\x \sim \normal}[f^4(\x)])^{1/2} \leq \valb$ 
and 
$\E_{\x \sim \normal}[\|\nabla f(\x)\|_2^2] \leq\gradb$.
\item 
There exists a subspace $U$ of $\R^d$ of dimension at most $k$ such that $f$ depends only on $U$, i.e.,  for every $\x \in \R^d$ it holds that $f(\x) = f(\proj_U \x)$, 
where $\proj_U \x$ is the projection of $\x$ on $U$.
\end{enumerate}
\end{definition}
We will subsequently see that this is a very broad class of functions subsuming commonly studied 
classes such as multi-layer neural networks with ReLUs and other activations.

Our main result is an efficient algorithm that exploits the power of queries 
to significantly reduce the runtime of agnostically learning the semi-parametric class of \Cref{def:bounded-variation-concepts}. 

\begin{theorem}[Agnostic Query Learner for Real-valued Multi-index Models]
\label{intro-thm:non-proper-real-valued}
Fix the function class $\mathfrak R(M, L, k)$ given in 
\Cref{def:bounded-variation-concepts}.
There exists an algorithm that makes $N_q = \poly(d \valb \gradb/\eps)$ queries, 
draws $N_s = \poly(d\valb \gradb/\epsilon) + k^{\poly(\gradb, M,1/\eps)}$ random labeled examples, 
runs in time $\poly(N_s,N_q,d)$, and outputs a polynomial $h:\R^d\mapsto\R$ 
such that with high probability $h$ has $L_2^2$-excess error 
$\mathcal{E}_2(h, \mathfrak{R}(L,M, k); y) \leq \epsilon$.
\end{theorem}

\paragraph{Comparison with Sample-Based Algorithms}
As a corollary of  \Cref{intro-thm:non-proper-real-valued}, we establish a 
strong separation between the agnostic PAC+Query model and the agnostic PAC model (with random samples only).
We first compare with the best-known algorithm for agnostically PAC 
learning real-valued functions, which is the $L_2$-polynomial 
regression algorithm. To agnostically learn the class of 
\Cref{def:bounded-variation-concepts} to excess error $\eps$, 
one needs polynomials of degree $\poly(L,M,1/\eps)$, 
and thus $d^{\poly(L,M, 1/\epsilon)}$ samples and time are necessary.
\Cref{intro-thm:non-proper-real-valued} leverages the power of queries to 
efficiently reduce the dimensionality of the problem, 
and thus qualitatively improve the computational complexity of agnostic learning to 
$\poly(d) \, k^{\poly(L, M, 1/\epsilon)}$.

Given the assumption of \Cref{def:bounded-variation-concepts} that the target 
function depends only on an unknown $k$-dimensional subspace, 
it is natural to attempt some kind of dimension-reduction
technique in order to reduce the sample and computational complexity of learning. 
Such reductions are indeed often possible \emph{in the realizable setting} by using 
some form of PCA and then working in the obtained low-dimensional subspace; 
see, e.g., \cite{Vempala10}.

On the other hand, in the agnostic setting considered here, 
there is strong evidence that such dimension-reduction schemes, 
or any other runtime improvements whatsoever, are impossible 
using only sample access to the target function. 
Specifically, a recent line of work (see, e.g.,~\cite{DKPZ21, DKR23}) 
has shown that for agnostically learning real-valued 
MIMS (even very special cases thereof), 
the standard $L_2$-regression algorithm is qualitatively optimal 
computationally (e.g., under standard cryptographic assumptions) 
in the standard agnostic PAC model. 
This in particular implies that the best possible runtime without query access 
is $d^{\poly(1/\epsilon)}$. In fact, even for learning a single ReLU activation, which satisfies \Cref{def:bounded-variation-concepts} with $L, M = O(1)$
and $k = 1$, $d^{\poly(1/\epsilon)}$ samples and time 
are required~\cite{DKPZ21,DKR23}.
In contrast, \Cref{intro-thm:non-proper-real-valued} decouples the
dimension dependence from the dependence on $1/\epsilon$ 
and yields an algorithm with runtime 
$\poly(d) \, 2^{\poly(1/\epsilon)}$.

\paragraph{Concrete Applications}
\Cref{intro-thm:non-proper-real-valued} applies to a fairly general non-parametric class of 
functions. Here we provide specific applications to well-studied classes of neural networks.

\smallskip

\noindent \textbf{\em Single Non-Linear Gates.} 
The simplest case is that of agnostically learning a ReLU, i.e., a function of the form 
$f(\bx) = \mathrm{ReLU}\left( \bw \cdot \bx \right)$, 
where $\bw \in \R^d$ and $\mathrm{ReLU}(t) = \max \{ 0, t\}$. 
In the vanilla agnostic PAC setting, the complexity of this problem is $d^{\poly(1/\eps)}$(both upper and lower bounds). 
On the positive side, the $L_2$-polynomial regression 
algorithm has sample and computational complexity $d^{\Theta(\poly(1/\eps))}$. On the negative side, 
there is strong evidence that this complexity upper bound is 
qualitatively best possible, both for SQ 
algorithms~\cite{GGK20, DKZ20, DKPZ21} and under plausible cryptographic assumptions~\cite{DKR23}. 
Our agnostic query learner has complexity $\poly(d)\, 2^{\poly(1/\eps)}$, 
implying a super-polynomial separation between the two learning models.

\begin{corollary}[Agnostic Query Learning for ReLUs]\label{intro-cor:agnostic-relu-mq}
There exists an agnostic query learner for the class of ReLUs on $\R^d$ 
with running time $\poly(d) \, 2^{\poly(1/\eps)}$.
\end{corollary}

\Cref{intro-cor:agnostic-relu-mq} follows from \Cref{intro-thm:non-proper-real-valued} 
by observing that ReLUs satisfy 
\Cref{def:bounded-variation-concepts} for $k=1$ 
and $L, M = O(1)$ (assuming that the norm of the weight vector is bounded, i.e., $\|\bw\|_2 = O(1)$).

Note that selecting the excess error to be $\eps  = 1/\log^c(d)$, 
where $c>0$ is a small constant, 
the query algorithm of \Cref{intro-cor:agnostic-relu-mq} has 
$\poly(d)$ runtime. On the other hand,
the complexity of agnostic learning problem with random samples 
is super-polynomial in $d$ for {\em any} $\eps = o_d(1)$.

Finally, we note that \Cref{intro-cor:agnostic-relu-mq} holds 
for other link functions satisfying smoothness assumptions, 
e.g., sigmoidal activations of the form $t \mapsto 1/(1+\exp(-t))$.

\smallskip

\noindent \textbf{\em Single-index Models.} 
Our first application above assumed that the link function is known a priori.We next consider learning Single-index models (SIMs) 
with an {\em unknown} Lipschitz link
function $g:\R \mapsto \R$, i.e., $f(\x) = g(\vec w\cdot \x)$.
Classical results~\cite{kalai2009isotron,kakade2011efficient} 
gave efficient algorithms for this setting in the realizable PAC
setting (or with unbiased additive noise) under the additional assumption that $g$ is non-decreasing.
The agnostic setting was recently considered in \cite{gollakota2023agnostically} who gave an efficient algorithm
achieving error $O(\sqrt{\opt}) + \epsilon$ for 
distributions with bounded second moments 
(similarly assuming weight vectors of bounded $\ell_2$-norm).
Using \Cref{intro-thm:non-proper-real-valued}, we can leverage query 
access to provide optimal agnostic guarantees with essentially the 
same complexity as for the case of known link function.

\begin{corollary}[Agnostic Query Learning for Lipschitz SIMs]\label{intro-cor:agnostic-sim-mq}
There exists an agnostic query learner for the class of $L$-Lipschitz SIMs on $\R^d$, for $L = O(1)$,
with running time $\poly(d) \, 2^{\poly(1/\eps)}$.
\end{corollary}

\smallskip

\noindent \textbf{\em One-Hidden Layer ReLU Networks.} 
Our approach naturally extends to non-negative linear combinations (aka sums) of ReLUs, i.e., functions of the form
$f(\x) = \sum_{i=1}^k \alpha^{(i)} \mathrm{ReLU}(\vec w^{(i)} \cdot \x)$ for $k$ non-negative weights $\alpha^{(i)} \geq 0$ and weight vectors $\vec w^{(i)} \in \R^d$. 
Prior work \cite{janzamin2015beating,GeLM18,DKKZ20,diakonikolas2020small}
has studied this problem in the noiseless setting with random samples
under the Gaussian distribution --- 
with the best-known runtime being 
$\poly(d/\epsilon) \,  (k/\epsilon)^{O(\log^2 k)}$  
\citep{diakonikolas2020small}.  
Using \Cref{thm:non-proper-real-valued}, we obtain an agnostic query learner with complexity
$\poly(d)O(k)^{\poly(1/\epsilon)}$. 
To see this, we note that as long as $\E[f^2(\x)] = O(1)$ 
we also obtain that $\E[\|\nabla f(\x)\|_2^2] = O(1)$
which implies only an $O(k)^{\poly(1/\epsilon)}$ runtime overhead.

Our approach can also be applied to the more general class of 
(unconstrained) linear combinations of $k$ ReLUs, i.e., 
functions of the form 
$f(\x) = \sum_{i=1}^k \alpha^{(i)} \mathrm{ReLU}(\vec w^{(i)} \cdot \x)$. This is known \cite{DKKZ20,ChenDGKM23, CN23, DK23-nn} to be a 
more challenging class of functions to learn. 
In the noiseless setting, the best known runtime
for general linear combinations is 
$(d k/\epsilon)^{O(k)}$ \citep{DK23-nn}. Using \Cref{intro-thm:non-proper-real-valued}, we obtain an agnostic query learner 
with complexity $\poly(d) \, 2^{\poly(k/\eps)}$.

\begin{corollary}[Agnostic Query Learning for 1-Hidden Layer ReLU Networks]\label{intro-cor:agnostic-sums-of-relus-mq}
There exists an agnostic query learner for sums of $k$ ReLUs on 
$\R^d$ with running time $\poly(d) \, O(k)^{\poly(1/\eps)}$. 
For general linear combinations of ReLUs, the runtime is
$\poly(d) \, 2^{\poly(k/\eps)}$. 
\end{corollary}

\smallskip

\noindent \textbf{\em Bounded Depth Neural Networks.} 
Our non-parametric function class of \Cref{def:bounded-variation-concepts} 
includes deep ReLU networks with $\ell$ layers of width at most $S$.
More precisely, we assume that 
$f(\x)=\vec W_L\mathrm{ReLU}(\vec W_{L-1}\cdots \mathrm{ReLU}(\vec W_1\x))$, 
for matrices $\vec W_1\in\R^{k_1\times d},\ldots, \vec W_L\in\R^{k_L\times 1}$, 
with $\|\vec W_i\|_{op}\leq O(1)$ and $k_i\leq S$; 
see \Cref{def:deep-nets} for more details.
The running time of our algorithm for this class 
is $\poly(d )2^{\poly(\ell S/\eps)}$; 
see \Cref{thm:non-proper-real-deep-of-RelLU}.
We remark that a similar fixed-parameter tractability result for deep ReLU networks 
was recently shown in \cite{CKM22} for the realizable PAC setting (with access to random examples only).  
Our result exploits the power of queries to provide a learner 
with qualitatively similar running time
in the much more challenging agnostic setting.
We remark that the following result can be readily 
extended to other continuous activation functions, including 
sigmoids, LeakyReLUs, and combinations thereof.

\begin{corollary}[Agnostic Query Learning for Bounded-Depth Networks]\label{intro-cor:agnostic-bounded-depth-mq}
There exists an agnostic query learner for 
$\ell$-depth, $S$-width, ReLU networks on
$\R^d$ with running time $\poly(d) 2^{\poly(\ell S/\eps)}$.
\end{corollary}

\begin{table}
\centering
\caption{
Learning Real-Valued Functions using Queries: 
Running time comparisons of the best known PAC algorithms 
with our PAC+Queries technique (Influence PCA).
}
\label{tab:sample_complexity_comparison-real} 
\begin{tabular}{@{}lcc@{}}
\toprule
Function Class  &  PAC (without queries)  & {PAC+Queries}
\\ 
 & $L_2$ Regression & {\bf  Influence PCA (Ours)} \\
\midrule
Single ReLU  
& $d^{\poly(1/\eps)}$ & $\poly(d)\, 2^{\poly(1/\eps)}$ \\
Sum of $k$ ReLUs 
& $d^{\poly(1/\eps)}$ & $\poly(d) \, O(k)^{\poly(1/\eps)}$ \\
Linear Combinations of $k$ ReLUs 
& $d^{\poly(k/\eps)}$ & $\poly(d) \, 2^{\poly(k/\eps)}$ \\
Deep Networks with $\ell$-Layers, $S$-width
& $d^{\poly(\ell S/\eps)}$ &  $\poly(d) \, 2^{\poly( \ell S/\eps)}$ \\
Bounded Variation
& $d^{\poly(k,L,M,1/\eps)}$ & $\poly(d) \, 2^{\poly(k,L,M,1/\eps)}$ \\
\bottomrule
\end{tabular}
\end{table}

For a summary of our results for the above classes, we 
refer to \Cref{tab:sample_complexity_comparison-real}
(where for the $L_2$-regression algorithm we only assume random sample access).

\vspace{-0.2cm}

\paragraph{Proper versus Improper Learning} 
The hypothesis computed by algorithm of \Cref{intro-thm:non-proper-real-valued} 
is not necessarily in the target concept class. That is, the agnostic learner is {\em improper}. 
With some additional effort, our approach can be used to obtain {\em proper} learners. 
As a concrete example, for the class of ReLUs, we show the following:

\begin{theorem}[Proper Agnostic Query Learner of ReLUs]\label{intro-thm:proper-relu-mq}
There exists an algorithm that makes $\poly(d/\eps)$ queries, 
runs in time $\poly(d) \, 2^{\poly(1/\eps)}$, 
and properly agnostically learns the class of 
ReLUs on $\R^d$, i.e., it outputs a ReLU hypothesis
$h(\x) = \mathrm{ReLU}(\wh{\vec w} \cdot \x)$ with excess $L_2^2$ error at most $\epsilon$ with high probability.
\end{theorem}

We note that in addition to computing a ReLU hypothesis, 
the learner of \Cref{intro-thm:proper-relu-mq} 
uses $\poly(d/\eps)$ labeled examples (queries plus random examples), removing the extraneous 
$2^{\poly(1/\eps)}$ term in our generic result. 

It is natural to ask whether the $2^{\poly(1/\eps)}$ runtime dependence
in \Cref{intro-thm:proper-relu-mq} is inherent. We provide evidence that such a 
dependence may be necessary for {\em proper} learners. Specifically, we show (\Cref{proper lower bound theorem relu})
that if there exists a $\poly(d/\eps)$ agnostic proper learning for our problem, 
there exists a polynomial-time algorithm for the small-set 
expansion (SSE) problem~\cite{RaghavendraS10} (refuting the SSE 
hypothesis). This hardness result also extends to the Boolean 
class of halfspaces. Obtaining a computational lower bound
for improper learners is left as an interesting open problem.

\subsubsection{Agnostically Learning Boolean Multi-index Models}

We start by describing the family of Boolean functions for which 
our results are applicable. Roughly speaking, our algorithmic approach can be used
to agnostically learn any Boolean concept class $\mathcal{C}$ satisfying the following conditions: 
(i) $\mathcal{C}$ has bounded Gaussian surface area, 
(ii) it depends on an unknown low-dimensional subspace, 
and (iii) it is closed under translations. 
Under these assumptions, we similarly obtain
a ``fixed parameter tractable'' agnostic learner qualitatively 
improving over the agnostic PAC setting with random examples only. 

The Gaussian surface area of a Boolean function is the surface area of its decision boundary 
weighted by the Gaussian density (\Cref{def:GSA}). The Gaussian surface area of a concept class 
has played a significant role as a useful complexity measure in learning theory 
and related fields; see, e.g.,~\cite{KOS:08, Kane11, Neeman14, KTZ19, DeMN21}. A formal definition follows: 

\begin{definition}[Gaussian Surface Area]
\label{def:GSA}
  For a Borel set $A \subseteq \R^d$, its Gaussian surface area is defined by
  $
  \Gamma(A) \eqdef \liminf_{\delta \to 0} \frac{\normal(A_\delta \setminus A)}{\delta},
  $  where $A_\delta = \{x : \mathrm{dist}(x, A) \leq \delta\}$.
  For a Boolean function $f:\R^d \mapsto \{\pm 1\}$, we overload notation and define its Gaussian surface area to be the surface area of its positive region
  $K = \{ \x \in \R^d: f(\x) = +1\}$, i.e., $\Gamma(f) = \Gamma(K)$.
  For a class of Boolean concepts $\mathcal{C}$, we define $\Gamma(\mathcal{C}) \eqdef \sup_{f \in \mathcal{C}} \Gamma(f)$.
\end{definition}

We are ready to define the class of Boolean multi-index models 
for which our approach applies.

\begin{definition}[Bounded Surface Area, Low-Dimensional Boolean Concepts]
\label{def:bounded-surface-area-concepts}
Fix $\Gamma>0$ and $k \in \Z_+$. 
We define the class $\mathfrak B(\Gamma, k)$ of Boolean concepts with the following
properties:
\begin{enumerate}[leftmargin=*]
\item For every $f \in \mathfrak B(\Gamma, k)$, it holds 
$\Gamma(f_{\vec r}) \leq \Gamma$ for all $\vec r \in \R^d$,  where $f_{\vec r}(\x) = f(\x + \vec r)$.  
\item For every $f \in \mathfrak B(\Gamma, k)$, there exists a subspace $U$ of $\R^d$ of dimension at most $k$ such that $f$ depends only on $U$, i.e., 
for every $\x \in \R^d$ it holds $f(\x) = f(\proj_U \x)$. 
\end{enumerate}
\end{definition}

We remark that $\mathfrak{B}(\Gamma, k)$ is a general \emph{non-parametric class} 
that contains a range of natural and well-studied Boolean function classes.
For example, $\mathfrak{B}(\Omega(k), k)$ contains
arbitrary functions of $k$ halfspaces.

Our main positive result in this context is a query algorithm that agnostically learns
the class $\mathfrak{B}(\Gamma, k)$ with running time 
$\poly(d)k^{\poly(\Gamma/\epsilon)}$. In more detail,
we establish the following theorem: 

\begin{theorem}[Agnostic Learner for Boolean Multi-index Models]
\label{intro-thm:non-proper-geometric}
Fix the concept class $\mathfrak B(\Gamma, k)$ given in 
\Cref{def:bounded-surface-area-concepts}.
There exists an algorithm that makes $N_q = \poly(d/\eps)$ queries, 
draws $N_s=\poly(d/\eps)+O(k)^{\poly(\Gamma/\epsilon)}$ random labeled examples, 
runs in sample-polynomial time, 
and outputs a hypothesis $h: \R^d \to \{\pm 1\}$ with excess 0-1 error 
$\mathcal{E}_{0/1}(h, \mathfrak{B}(\Gamma, k); y) \leq \epsilon$.
\end{theorem}

\paragraph{Discussion} 
Some remarks are in order. We start by noting that, 
in the setting of \Cref{intro-thm:non-proper-geometric}, an exponential dependence on the parameter $\Gamma$ 
is {\em information-theoretically} necessary --- even with access to queries. Specifically, as shown in \cite{KOS:08}, there 
exists a Boolean concept class with Gaussian surface area $\Gamma$ 
(consisting of intersections of halfspaces) 
such that the total number of samples and queries required to obtain constant accuracy is $2^{\Omega(\Gamma)}$.

It is worth comparing \Cref{intro-thm:non-proper-geometric} with the best 
known algorithmic results in the standard agnostic PAC model (with random 
samples only). Klivans, O'Donnell and Servedio~\cite{KOS:08} showed that the 
$L_1$-polynomial regression algorithm of~\cite{KKMS:08} agnostically learns 
any concept class on $\R^d$ whose Gaussian surface area is at most $\Gamma>0$ 
with (sample and computational) complexity $d^{\poly(\Gamma/\eps)}$. 
Under the additional assumption that the concepts in the target class depend 
on an unknown $k$-dimensional subspace, for some parameter $k \ll d$, 
\Cref{intro-thm:non-proper-geometric} gives a significantly improved agnostic 
query algorithm with computational complexity 
$\poly(d) \,  k^{\poly(\Gamma/\eps)}$. 

For a concrete example, if the target class is the concept class consisting of any intersection of $\ell$ halfspaces, 
then we have that $k = \ell$ and $\Gamma = O(\sqrt{\log(\ell)})$~\cite{KOS:08}. 
So, as long as $\ell = O(1)$ or even $\ell = \polylog(d)$, query access allows us to obtain a super-polynomial complexity improvement.

\paragraph{Concrete Applications}
\Cref{intro-thm:non-proper-geometric} applies to a fairly general non-parametric class of functions. Here we provide specific applications to well-studied classes of Boolean functions. 

\smallskip

\noindent \textbf{\em Halfspaces.}
Arguably the simplest application is for the class of halfspaces.
A halfspace (or Linear Threshold Function) is any Boolean-valued function 
$f: \R^d \to \{ \pm 1\}$ of the form 
$f(\bx) = \sgn \left( \bw \cdot \bx -\theta \right)$, 
where $\bw \in \R^d$ is the weight vector and 
$\theta \in \R$ is the threshold. 
(The function $\sign: \R \to \{ \pm 1\}$ is defined as $\sgn(t)=1$ if $t \geq 0$ 
and $\sgn(t)=-1$ otherwise.) 
The problem of PAC learning halfspaces is a textbook problem in machine 
learning, whose history goes back to Rosenblatt's 
Perceptron algorithm~\cite{Rosenblatt:58}. As a corollary of \Cref{intro-thm:non-proper-geometric}, 
we obtain the following:

\begin{corollary}[Agnostic Query Learning of Halfspaces]
\label{intro-cor:agnostic-ltf-mq}
There exists an agnostic query learner for the class of halfspaces 
on $\R^d$ with running time $\poly(d) \, 2^{\poly(1/\eps)}$.
\end{corollary}

\Cref{intro-cor:agnostic-ltf-mq} follows from \Cref{intro-thm:non-proper-geometric} by observing that halfspaces satisfy \Cref{def:bounded-surface-area-concepts} for $k=1$ and $\Gamma \leq 1/\sqrt{2 \pi}$. 

As mentioned in the introduction, \Cref{intro-cor:agnostic-ltf-mq} answers an open question independently posed by Feldman~\cite{Feldman08} 
and by Gopalan, Kalai, and Klivans~\cite{GKK08-open}. 
Specifically, as we explain below, it implies a {\em super-polynomial} computational separation between agnostic query learning and agnostic learning with random samples for the class of halfspaces.

In the vanilla agnostic PAC setting, the complexity of this problem 
is $d^{\poly(1/\eps)}$; the upper bound
follows via the $L_1$-polynomial regression 
algorithm~\cite{KKMS:08} which has complexity $d^{\Theta(1/\eps^2)}$~\cite{DKN10} 
in this setting. The matching lower bound follows from a recent 
line of work, both in the SQ model~\cite{GGK20, DKZ20, DKPZ21} 
and under plausible cryptographic assumptions~\cite{DKR23, Tiegel22}.

\smallskip

\noindent \textbf{\em Functions of Halfspaces.} A more general concept class where our general approach is applicable is that consisting of all intersections (or arbitrary functions) of a bounded number of halfspaces. For the special case of intersections, we show:

\begin{corollary}[Agnostic Query Learning for Intersections of Halfspaces]\label{intro-cor:agnostic-int-ltf-mq}
There exists an agnostic query learner for intersections of $\ell$ halfspaces on $\R^d$ with running time $\poly(d) \, O(\ell)^{\poly(\log(\ell)/\eps)}$.
\end{corollary}

\Cref{intro-cor:agnostic-int-ltf-mq} follows from \Cref{intro-thm:non-proper-geometric} by observing that intersections of $\ell$ halfspaces satisfy \Cref{def:bounded-surface-area-concepts} for $k=\ell$ and that their Gaussian surface area is bounded above by $\Gamma  = O(\sqrt{\log(\ell)})$, as shown by Nazarov (see, e.g.,~\cite{KOS:08, chernozhukov2017detailed}).

Analogously to the case of a single halfspace, the complexity of the agnostic learning problem with random samples is significantly worse (as long as $\ell \ll d$), namely $d^{\poly(\log(\ell)/\eps)}$;  
the upper bound follows from~\cite{KOS:08} and a qualitatively matching SQ lower bound was given 
in~\cite{DKPZ21, HsuSSV22}.

Finally, for arbitrary functions of $\ell$ halfspaces, the Gaussian surface area is bounded by $\Gamma = O(\ell)$, 
leading to the following corollary: 

\begin{corollary}[Agnostic Query Learning for Functions of Halfspaces]\label{intro-cor:agnostic-fun-ltf-mq}
There exists an agnostic query learner for arbitrary functions of $\ell$ halfspaces on $\R^d$ with running time $\poly(d) \, O(\ell)^{\poly(\ell/\eps)}$.
\end{corollary}

\noindent Similarly, the best known complexity upper bound with random samples is $d^{\poly(\ell/\eps)}$.

\smallskip

\noindent \textbf{\em Low-degree Polynomial Threshold Functions (PTFs).} 
Another notable application is for the class of low-degree PTFs that depend on a low-dimensional 
subspace. A degree-$\ell$ PTF is any Boolean function $f: \R^d \to \{ \pm 1\}$ 
of the form $h(\bx) = \sgn \left( p(\x)\right)$, where $p: \R^d \to \R$ is a degree at most $\ell$ 
polynomial. Low-degree PTFs have been extensively studied in theoretical machine learning and 
specifically in the context of agnostic learning~\cite{DHK+:10, DSTW:10, DRST14, Kane11}.

Here we consider a natural subclass of low-degree PTFs where the underlying polynomial is a subspace junta.
Specifically, we consider the class of Boolean functions of the form
$f(\bx) = \sgn \left( p(\proj_U \x)\right)$, where 
$U$ is an unknown $k$-dimensional subspace and $p$ is a degree-$\ell$ polynomial in $k$ variables. 
Since the Gaussian surface area of this class of functions is bounded above by $\Gamma = O(\ell)$~\cite{Kane11}, we obtain the following corollary: 

\begin{corollary}[Agnostic Query Learning for Low-Dimensional PTFs]\label{intro-cor:agnostic-ptf-mq}
There exists an agnostic query learner for degree-$\ell$ PTFs on $\R^d$ that depend on an unknown $k$-dimensional subspace with running time 
$\poly(d) \, O(k)^{\poly(\ell/\eps)}$.
\end{corollary}

The above running time bound should be compared with the best known complexity bound of $d^{\poly(\ell/\eps)}$ for agnostic learning with samples~\cite{Kane11}.

\Cref{tab:sample_complexity_comparison} summarizes our contributions for Boolean concept classes in comparison to prior work on agnostic PAC learning (with random samples only).

\begin{table}
\centering
\caption{Learning Boolean Concepts using Queries: 
Running time comparisons of the best known agnostic learners (using random samples)
with our Influence PCA technique (using queries).}
\begin{tabular}{@{}lcc@{}}
\toprule
Concept Class  &  PAC (without queries) &  PAC+Queries 
\\ 
 & $L_1$ Regression \cite{KOS:08} & {\bf Influence PCA (Ours)} \\
\midrule
Single Halfspace& $d^{\poly(1/\eps)}$ & $\poly(d) \, 2^{\poly(1/\eps)}$ \\

Intersections of $k$ Halfspaces& $d^{\poly(\log(k)/\eps)}$ & $\poly(d) \, 2^{\poly(\log (k)/\eps)}$ \\

Functions of $k$ Halfspaces& $d^{\poly(k/\eps)}$ & $\poly(d) \, 2^{\poly(k/\eps)}$ \\

Degree-$\ell$, $k$-Dim. PTFs& $d^{\poly(\ell/\eps)}$ & $\poly(d) \, O(k)^{\poly(\ell/\eps)}$ \\

Low-Dim. Geometric Concepts& $d^{\poly(\Gamma/\eps)}$ & $\poly(d) \,  O(k)^{\poly(\Gamma/\eps)}$ \\
\bottomrule
\end{tabular}
\label{tab:sample_complexity_comparison}
\end{table}

 \section{Technical Overview}
\label{sec:techniques}

We leverage query access to develop a unified dimension-reduction 
framework for agnostically learning both real-valued and Boolean-valued 
multi-index models. As already explained after the statement 
of \Cref{intro-thm:non-proper-real-valued}, 
natural dimension-reduction approaches that work in the realizable (noiseless) setting
inherently cannot be extended to the agnostic setting.

At a high-level, our framework reduces the problem of agnostically 
learning MIMS in $d$ dimensions to agnostically learning the same 
class in $\poly(k/\epsilon)$ dimensions. 
It consists of three main steps:
\begin{itemize}[leftmargin=*]
\item 
First we use queries to the label function
to simulate gradient queries to a ``smoothed'' version $\wt{y}(\x)$ 
of the adversarial label $y(\x)$.  We show that, as long as the concept class
of interest has bounded variation (real-valued MIMs of 
\Cref{def:bounded-variation-concepts}) 
or bounded Gaussian surface area 
(Boolean MIMs of \Cref{def:bounded-surface-area-concepts}),
a hypothesis that has low excess-error with respect to 
the smoothed label $\wt y$ will also have low excess error with
respect to the original label $y(\x)$; 
see \Cref{inf-prop:informal-smoothing-erm}.  

\item 
The second step uses gradient queries to the function $\wt y$ 
in order to  compute an accurate estimate of the influence matrix 
of the ``smoothed'' label, namely 
$\vec M = \E_{\x \sim \normal}[\nabla \wt y(\x) (\nabla \wt y(\x))^\top]$. 
We perform PCA on $\vec M$ and find the top eigenvectors 
(i.e., the eigen-directions whose corresponding eigenvalues 
are larger than some threshold).  
This method is known as outer gradient product~\cite{xia2002adaptive}; 
in the context of learning/testing Boolean concepts, it has been used in 
\cite{DKKTZ21,DeMN21}. (See \Cref{sec:related-work} for a detailed summary 
of related work.)
We show that those ``high-influence'' directions form 
a low-dimensional (i.e., of dimension $\poly(k/\epsilon)$) 
subspace such that there exists a hypothesis that 
(i) depends only on the low-dimensional subspace, 
(ii) has bounded surface area/variation, 
and (iii) is close to our target function. That is, 
we effectively reduce the dimension of our original learning task 
from $d$ down to $\poly(k/\epsilon)$.

\item 
The third step is to solve an agnostic learning task 
of a bounded variation/surface area function
in the low-dimensional subspace spanned by the top eigenvectors of $\vec M$.  
For this step, for learning real-valued MIMs, we rely on a generic 
$L_2$-regression algorithm; for learning Boolean concepts, 
we use the $L_1$-polynomial regression agnostic learner of \cite{KKMS:08, KOS:08}.  Those methods yield non-proper learning algorithms -- to obtain proper-learners, 
we essentially perform a brute-force search over a net 
of the \emph{low-dimensional} parameter space found in the previous step.
\end{itemize}

\subsection{From Zero- to First-Order: Gradient Queries via Oracle Queries}
Intuitively, having access  to queries, for some example $\x$, 
we can ask for the values of $y(\x)$ in a ``small'' neighborhood around $\x$ and therefore
estimate the gradient $\nabla_\x y(\x)$.  
The first issue that we have to overcome is that the observed label $y(\x)$ is not guaranteed to be a differentiable function (even if the underlying target function is). 
To circumvent this issue, 
we employ a strategy similar to the Gaussian convolution technique 
used in zero-order (gradient-free) optimization \cite{nesterov2017random}.
In particular, to estimate the gradient of a function $y(\cdot)$ at $\x$ 
only having access to a value oracle, the method 
samples $\vec z$ from a mean-zero Gaussian with small covariance, 
i.e., $\vec z \sim \normal(\vec 0, \rho \vec I)$ for some small $\rho$, 
and then asks for the value of the function at $\vec x + \rho \vec z$.
Even if the function $y(\cdot)$ itself is non-smooth, 
then, by Stein's identity, we have
$\E_{\vec z \sim \normal}[\vec z ~ y(\vec x + \rho \vec z)] \propto
\nabla \wt y(\x)$, where $\wt y(\x)$ is a smoothed version of $y(\x)$, 
specifically 
$\wt y(\x) = \E_{\vec z \sim \normal}[y(\x + \rho \vec z)]$. 
By drawing $N = \poly(d/\epsilon)$ Gaussian samples 
$\vec z^{(1)}, \ldots, \vec z^{(N)}$, 
we can empirically estimate
the gradient of $\wt y(\cdot)$ at every desired point $\x \in \R^d$. 
Therefore, by performing $N$ queries on the points $\vec z^{(i)}$, 
we obtain an approximation of the gradient $\nabla \wt y(\x)$ 
for any $\x$. 
Even though the above technique yields gradient estimates, 
it comes with a cost: \emph{to obtain the ``smooth'' label $\wt y(\x)$, 
we add noise to the (already corrupted) label $y(\x)$. 
Our plan is to argue that learning using the resulting smoothed labels 
$\wt{y}(\x)$ yields a good classifier for the original instance --- as long as the ``smoothing'' parameter $\rho$ is sufficiently small.}

\paragraph{\OU Smoothing}
One could hope that if we add a small amount of noise to $y(\x)$, the smooth label 
$\wt y(\x)$ will be close to $y(\x)$ (at least in the $L_2$-sense).
Unfortunately, this is not true (even in one dimension), 
as $y(\x)$ may be an arbitrarily complex
function and after smoothing $\wt y(\x)$ may be far from $y(\x)$; 
see \Cref{fig:label-smoothing}. 
To be able to learn from the smoothed instance, we need two properties:
(i) the resulting marginal distribution on the examples 
must be close to the initial $\x$-marginal, 
and (ii) the smoothing operation must not increase 
the excess error of the functions  
in the hypothesis class by a lot. 
In other words, a hypothesis that 
performs well with respect to the smoothed label $\wt{y}(\x)$ should also 
perform well with respect to the original label $y(\x)$.
Applying the Gaussian convolution smoothing $\x + \rho \z$ yields 
a normal distribution that has covariance $(1 +\rho) \vec I $. 
In order to make this distribution be close to a standard normal 
(say, in total variation distance), one would need to apply 
a tiny amount of noise, i.e., $\rho$ should be at most $\poly(1/d)$. 
To avoid changing the $\x$-marginal of the instance, 
instead of simply convolving with a Gaussian kernel, 
we  apply the \OU noise operator $T_\rho$ that rescales $\x$ 
and corresponds to the transformation 
$\wt{\vec x} = \sqrt{1-\rho^2} \vec x + \rho \vec z$. 
We observe that $\wt{\vec x}$ follows a standard normal distribution. 
The resulting ``smoothed'' label $\wt{y}$ is now defined 
as $T_\rho y(\x) = \E_{\vec z \sim \normal}[y(\wt{\x})]$.
Even though the marginal of $\wt{\x}$ matches exactly with the initial marginal,
we have introduced noise to the instance 
and we still need to show that this
does not significantly affect the performance of the hypotheses 
in the function class of interest.

We show that, regardless of how complex the label $y(\x)$ is, 
if the function class of interest is ``well-behaved'' 
--- in the sense that it only contains concepts with 
bounded variation/Gaussian surface area --- 
the \OU noise process will not significantly affect 
the excess error of a hypothesis $h$.

\begin{proposition}[Informal -- \OU Smoothing Preserves the Risk-Minimizer]
\label{inf-prop:informal-smoothing-erm}
Let $y: \R^d \mapsto \R$ and $C$ 
be a class of functions over $\R^d$ 
such that for every $f \in C$ it holds 
$\E_{\x \sim \normal}[\| \nabla f(\x) \|_2^2] \leq L$.
Let $\wt f \in C$ be an $L_2$ risk minimizer with respect to the smoothed label $T_\rho y$ (see \Cref{def:OU-operator}), i.e., $\wt f \in \argmin_{h \in C}\E_{\x \sim \normal}[(h(\x) - T_\rho y(\x))^2]$. Then we have that
\[
\pr_{\x \sim \normal}[(\wt f(\x) - y(\x))^2]
\leq 
\inf_{f \in C}
\pr_{\x \sim \normal}[ (f(\x) - y(\x))^2]
+ O(\rho^2 L) \,.
\]
\end{proposition}
At a high-level, the effect of the noise operator $T_\rho$ on the risk minimizer 
is milder when the function does not change very rapidly. 
To prove \Cref{inf-prop:informal-smoothing-erm}, 
we show that the correlation of any hypothesis $f$ 
with bounded variation  is approximately preserved 
when we replace $y(\x)$ with $T_\rho y(\x)$. 
The correlation of $f$ with respect to $T_\rho y(\x)$ is 
$\E_{\x \sim \normal}[f(\x) T_\rho y(\x)]$.  However, since $T_\rho$ is a symmetric linear operator, we can equivalently apply the smoothing $T_\rho$ 
to $f$ and consider $\E_{\x \sim \normal}[T_\rho f(\x)  y(\x)]$.
Since $f(\x)$ has bounded variation, we can now show via a result on noise sensitivity for real-valued functions, that $T_\rho f(\x)$ is indeed 
close to $f(\x)$ in $L_2^2$. Therefore, the correlation $\E_{\x \sim \normal}[T_\rho f(\x)  y(\x)]$ 
is close to  $\E_{\x \sim \normal}[f(\x)  y(\x)]$. 
The fact that $T_\rho f$ and $f$ are close is intuitively clear: 
the smaller the variation of $f$, $\E_{\x \sim \normal}[\| \nabla f(\x)\|_2^2]$, 
the smaller the effect of slightly perturbing a point $\x$ will
have on the $L_2^2$, as the $L_2^2$ distance between
$f(\x)$ and $f(\sqrt{1-\rho} \x + \rho z)$ is roughly proportional 
to $\rho^2 \|\nabla f(\x)\|_2^2$.
For more details, we refer to \Cref{sec:membership-derivatives} and \Cref{lem:smoothing-real}.

For learning Boolean concepts, we identify 
their Gaussian Surface Area to be the crucial complexity measure that
determines the effect the smoothing operator $T_\rho$ has on the agnostic
learning instance.   Similarly to our result for real-valued functions, 
we reduce preserving the excess error to preserving the correlation of concepts, 
i.e., ensuring that 
$\E_{\x \sim \normal}[f(\x) T_\rho y(\x)] - \E_{\x \sim \normal}[f(\x) y(\x)] $ 
is small for all concepts of interest $f$ --- 
see \Cref{lem:smoothing-boolean} --- 
and then use a result of Ledoux \cite{Ledoux:94} and Pisier \cite{Pisier:86} 
to show that correlations are indeed approximately preserved when the concepts have bounded Gaussian Surface Area; see \Cref{lem:smoothing-boolean}.

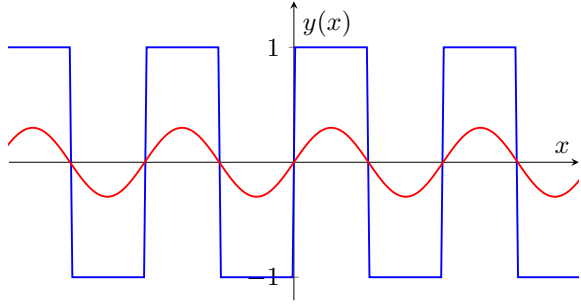
\begin{figure}[h!]
\centering
\begin{tikzpicture}[scale=0.9]
  \begin{axis}[
    xlabel=$x$,
    ylabel=$y(x)$,
    xmin=-2,
    xmax=2,
    ymin=-1.2,
    ymax=1.4,
    axis lines=center,
    width=10cm,
    height=6cm,
    domain=-2*pi:2*pi,
    samples=700,
    xtick=\empty,
    ]
    \addplot+[mark=none,thick] {sign(sin(6 * deg(x)))};
    \addplot+[mark=none,thick] {0.3 * sin(6 * deg(x))};
  \end{axis}
\end{tikzpicture}
\caption{Smoothing the label $y(\x)$. The label $y(\x)$ corresponds to the ``square wave'' (shown in blue).  The smoothed version $\wt{y}(\x)$ is the red curve.  We observe that $y(\x)$ and $\wt{y}(\x)$ are far (in the $L_2$ sense).
}
\label{fig:label-smoothing}
\end{figure}

\subsection{Learning Bounded Variation Functions via Influence PCA}

\paragraph{Real-Valued MIMs} 
Up to this point, we have established that 
(i) we can leverage query access 
in order to efficiently simulate gradient queries 
for the \OU smoothed label $T_\rho y$, 
and (ii) learning from the smoothed label $T_\rho y$ 
is approximately equivalent to learning from the original label 
$y(\x)$.  We will now describe an efficient learner that uses 
the gradient queries to $T_\rho y$. 

Our learner is based on estimating the influence matrix of $T_\rho y$, 
i.e.,  
$\vec M = \E_{\x \sim \normal}[\nabla T_\rho y(\x) (\nabla T_\rho y(\x))^\top]$, 
using  gradient queries.
Our main structural result is a general dimension-reduction tool establishing
the following: given (an approximation of) the influence matrix 
of the smooth function $T_\rho y $, we can perform PCA and 
learn a low-dimensional subspace $V$ so that 
a bounded variation function that depends only on $V$ 
can achieve $\eps$ excess error with respect to $T_\rho y$ in $L_2^2$.
This dimension-reduction step crucially relies on the target concept being
low-dimensional (see \Cref{def:bounded-variation-concepts}).

In fact, our dimension-reduction proof for real-valued concepts 
shows directly that a low-degree polynomial that depends only 
on the low-dimensional space $V$ exists.  
\begin{proposition}[Informal-- Dimension Reduction via Influence PCA: 
Real-Valued Functions]\label{info-prop:main-prop-real-values}
Let $\wt y(\x) = T_\rho y(\x)$ and  let $\vec M = \E_{\x \sim \normal}[ \nabla \wt y (\x) (\nabla \wt y(\x))^\top]$. Moreover, let $V$ be the subspace spanned by 
all the eigenvectors of $\vec M$ whose corresponding eigenvalues are 
at least $ \eps^2/(k\valb)$. The following holds:
\begin{itemize}[leftmargin=*]
\item 
The dimension of $V$ is at most $\poly(\valb,k,1/\rho, 1/\epsilon)$. 
\item There exists a polynomial $q: V \mapsto \R$ of degree $m = O(L/\epsilon^2)$ such that 
\[
\E_{\x\sim \normal}[(q(\proj_V (\x) ) -\wt y (\x))^2]\leq \inf_{f\in \mathfrak R(\valb,\gradb, k) }\E_{\x\sim \normal}[( f(\x) - \wt{y}(\x) )^2]+\eps\;.
\]
\end{itemize}
\end{proposition}

To prove \Cref{info-prop:main-prop-real-values}, 
we explicitly construct a low-dimensional polynomial as follows: 
we first marginalize out the low-influence directions of $\wt y(\cdot)$, 
and then we keep its low-degree Hermite approximation.  

\paragraph{Marginalizing Low-Influence Directions} 
We first construct a low-dimensional (not necessarily polynomial) version
of the noisy label $\wt y$ that preserves the correlation with the target
function $f(\cdot)$.
By the assumption of \Cref{inf-prop:dimension-reduction}, 
all directions in the orthogonal complement $V^\perp$ are low-influence, 
i.e., for $\vec h \in V^\perp$ it holds 
$\E_{\x \sim \normal}[(\vec h \cdot \nabla \wt y(\x))^2] \leq O(\epsilon^2/k)$.
In words, the function $\wt y$ is ``approximately constant'' 
along some low-influence direction $\vec h$. 
Let us first assume that $\wt y$ is exactly constant on all directions 
of $V^\perp$.  Then, in order to preserve
the correlation of $\wt y$ with $f$,  
\emph{we only need to match the expected value of $\wt y$ over $V^\perp$}. 
This motivates the following ``Gaussian Marginalization Operator'' 
$(\Pi_V g) (\x) \eqdef \E_{\vec z \sim \normal}[g(\proj_V \x + \proj_{V^\perp}\vec z)] $ 
(see \Cref{def:gaussian-marginalization-operator} and \Cref{lem:properties-of-the-operators}).  So a natural low-dimensional ``approximation'' 
of $\wt y$ is $\Pi_V \wt y$.
Indeed, if $\wt y$ was constant on $V^\perp$, using the fact that 
$\proj_V \x$ and $\proj_{V^\perp} \x$ are independent standard Gaussians, 
we would obtain that
$$
\E_{\vec z \sim \normal}[\E_{\x \sim \normal}\wt y(\proj_V(\x) + \proj_{V^\perp}(\vec z) ) f (\x)]] - 
\E_{\x\sim \normal}[\wt y (\x) f(\x)]
= 0  \;.$$
Our goal is to show that the Gaussian marginalization $\Pi_V \wt y$ 
achieves similar correlation with $\wt y $ as $f$,
when $\wt y $ is not constant in $V^\perp$ but ``approximately constant'', 
i.e., it has low-influence in directions of $V^\perp$.  
In \Cref{lem:correlation-bound-real} we show that when $V^\perp$ 
contains only low-influence directions, 
the same is approximately true (up to some additive $\epsilon$ error):
\(
\E_{\vec z \sim \normal}[( \wt y(\x) - \Pi_V \wt y(\x) ) f (\x)] 
\leq O(\epsilon) \,.
\)
To do this, we first observe that, since $f$ depends only on the subspace $U$,
it holds that $\Pi_U f = f$ and $\Pi_{V} f$ depends only on the directions
inside the relevant subspace $W = U + V$. 
We can thus restrict our attention on $W$, i.e., 
bound the difference 
$
\E_{\vec z \sim \normal_W}[(\wt y(\vec z) - \Pi_V \wt y (\vec z)) f(\vec z)]$, where $\normal_W$ is a standard normal on the subspace $W$.
We will show that this correlation difference can be bounded by the variance of $\wt y$ in the irrelevant directions. 
Indeed, by the Cauchy-Schwarz inequality, we have
\[
\E_{\vec z \sim \normal_W}[(\wt y(\vec z) - \Pi_V \wt y (\vec z)) f(\vec z)]
\leq 
\left(\E_{\x \sim \normal_W}[f^2(\x)] \right)^{1/2}
\left(\E_{\z \sim \normal_W}[(\wt y(\vec z) - \Pi_{V}  \wt y(\vec z))^2] 
\right)^{1/2}
\,.
\]
We next relate the $L_2^2$ error introduced by the marginalization 
operation $\Pi_V$ on $\wt y$  with the influence matrix $\vec M$.  
We use the Gaussian Poincare inequality, which states that for some 
$g(t) :\R \mapsto \R$ it holds 
$\var_{t \sim \normal}[g(t)] \leq \E_{t \sim \normal}[(g'(t))^2]$.
We obtain that for any subspace $R = \vec r^\perp$ 
(the orthogonal complement to the direction $\vec r$) 
the variance $\E_{\z \sim \normal_W}[(\wt y(\vec z) - \Pi_{R}  \wt y(\vec z))^2]$
is bounded above by 
$\E_{\x \sim \normal_W}[(\nabla \wt y (\x) \cdot \vec r)^2] 
= \vec r^\top \vec M \vec r$.
By repeatedly applying the Gaussian Poincare inequality 
on a basis of the (at most) $k$-dimensional subspace $V^\perp \cap W$, 
we show that 
\[
\E_{\x \sim \normal_W}[(\wt y (\vec z) - \Pi_V  \wt y (\vec z))^2] \leq 
M k \max_{\vec r\in V^\perp,\|\vec r\|_2=1} \vec r^\top \vec M \vec r \leq k ~ O(\epsilon^2/( k M ) = O(\epsilon^2)
\,.
\]
In the above bound, we observe that accepting eigenvectors 
with corresponding eigenvalues at least $\epsilon^2/( M k )$ 
ensures that $\Pi_V \wt y$ achieves at most $O(\epsilon)$ worse 
correlation with $f$ than $\wt y$.

\paragraph{The Low-Degree Polynomial Approximation}
We have established that $\Pi_V \wt y$ is similar 
to $\wt y$ in the sense 
that it has similar (up to $\epsilon^2$) correlation 
with the target function $f(\cdot)$. 
To obtain a polynomial with a similar behavior, we
use the low-degree Hermite expansion of $\Pi_V \wt y$, which we denote 
by $P_m \Pi_V \wt y$, where $P_m g$ maps the function $g$ to its degree 
Hermite expansion. 
We show that in order for $P_m \Pi_V \wt y$ to 
achieve low $L_2^2$ excess error, 
it suffices to pick the degree $m$ so that
$P_m f(\x)$ is close to $f(\x)$ (in $L_2^2$).
We show that 
the following bound for the excess error
defined as
$\mathcal{E}_2(q, f; \wt y)
=
 \E_{\x\sim \normal}[(\wt y(\x) - q(\x))^2] - \E_{\x\sim \normal}[(\wt y(\x)-f(\x))^2]$.  
We refer to \Cref{lem:excess-l2-error-decomposition} 
for the formal statement and proof.
\begin{lemma}[Informal -- Excess $L_2^2$ Error Decomposition]
\label{info-lem:excess-l2-error-decomposition}
It holds 
\begin{align*}
\mathcal{E}_2(P_m \Pi_V \wt y, f; \psi)
\leq 
O(1) \Big(\underbrace{\E_{\x \sim \normal}[(f(\x) - P_m f(\x))^2]}_{\textrm{Polynomial Approximation Error}}
+  \underbrace{\E_{\x \sim \normal}[(\wt y(\x) - \Pi_V \wt y(\x) ) f(\x)]}_{\textrm{Correlation Error}} \Big)
\,.
\end{align*}

\end{lemma}
Since $f(\x)$ has bounded variation 
(see \Cref{def:bounded-variation-concepts}), 
we can show using a result from \cite{KTZ19} 
(see \Cref{lem:low-degree-approximation-smooth-functions}) 
that with degree $m = O(L/\epsilon^2)$,
it holds that 
$\E_{\x \sim \normal}[(f(\x) - P_m f(\x))^2] = \epsilon$.  Moreover,
in the previous paragraph, we have already  established that 
the correlation error is also $O(\epsilon)$.

\paragraph{Polynomial Regression in $V$}
So far, we have identified the subspace $V$ and we know 
that there exists a polynomial that depends on $V$ and achieves low $L_2^2$ error with the smoothed label $\wt y = T_\rho y$. 
Since we have established that the smoothing operation $T_\rho$ 
does not affect the excess error of a bounded-surface area concept by a lot 
(see \Cref{inf-prop:informal-smoothing-erm}),
we know that the same concept will achieve low excess-error with respect to the 
original label $y$. 
Having established this, for our final step we may directly 
perform polynomial regression in the low-dimensional subspace $V$
to learn a polynomial with low-excess error. Since the dimension
of $V$ is roughly $\poly(M k /\epsilon)$ and the degree of the polynomial
is $\poly(L/\epsilon)$, the total sample and computational complexity of 
this task is roughly $k^{\poly(L/ \epsilon)}$.

\paragraph{Boolean MIMs}

At a high level, the proof and algorithm for Boolean MIMs is similar to that
for real-valued MIMs.  We show the following dimension reduction lemma that essentially reduces the initial problem to learning a bounded surface area concept in a $\poly(k/\epsilon)$-dimensional subspace $V$.

\begin{proposition}[Informal -- Dimension-Reduction via Influence PCA: Boolean Concepts]
\label{inf-prop:dimension-reduction}
    Let $V$ be the subspace spanned by
    all the eigenvectors of
    $\vec M = \E_{\x \sim \normal}[\nabla T_\rho y(\x) (\nabla T_\rho y(\x))^\top]$ whose corresponding eigenvalues are
    at least $\Omega(\eps^2/k)$. The following holds:
    \begin{itemize}[leftmargin=*]
    \item The dimension of $V$ is at most $\poly(k/(\epsilon \rho))$.
    \item
    There exists $g: \R^d \to \{ \pm 1 \}$ with $\Gamma(g) \leq \Gamma$
    and $g(\x)=g(\proj_V \x)$ for all $\x \in \R^d$ such that
    \[
    \E_{\x \sim \normal}[|g(\x)- T_\rho y(\x)|]
    \leq \inf_{f\in \mathfrak B(\Gamma, k)} \E_{\x \sim \normal}[|f( \x) - T_\rho y (\x)|] + \eps \,.
    \]
    \end{itemize}
\end{proposition}

So far, we have identified the subspace $V$ and we know 
that there exists a bounded surface area Boolean concept 
that depends on $V$ and achieves low $L_1$ error with the smoothed label $T_\rho y$. 
Since we have established that the smoothing operation $T_\rho$ 
does not affect the excess error of a bounded-surface area concept by a lot 
(see \Cref{inf-prop:informal-smoothing-erm} and \Cref{lem:boolean-excess}), 
we know that the same concept will achieve low excess-error with respect to the 
original label $y$. Having established this, for our final step we may use the 
$L_1$-agnostic learner of \cite{KOS:08} on the $k$-dimensional subspace $V$ 
to learn a PTF of degree  $\poly(\Gamma/\epsilon)$ 
with $(\dim(V))^{\poly(\Gamma/\epsilon)} = k^{\poly(\Gamma/\epsilon)}$ 
samples and time.

\subsection{Hardness of Proper Agnostic Query Learning for ReLUs and Halfspaces} 

Here we sketch our hardness reduction, establishing that the exponential dependence
in $1/\eps$ is inherent for proper agnostic learners, even with query access to the function (see \Cref{proper lower bound theorem} and 
\Cref{proper lower bound theorem relu}).
In particular, we show that assuming there are no polynomial-time algorithms 
for the Small-Set Expansion (SSE) problem~\cite{RaghavendraS10}, 
then there are no polynomial time \emph{proper} agnostic learning algorithms 
for ReLUs and homogeneous halfspaces with respect to the Gaussian distribution.

The basic idea of our argument is to reduce to the problem of (approximately) 
optimizing a homogeneous degree-$4$ polynomial over the unit sphere 
(for the case of halfspaces we reduce to optimizing a degree-$5$ polynomial).
As there are already known reductions from SSE to the problem 
of finding approximate maxima of degree-$4$ polynomials 
(and for halfspaces we can do a simple reduction from degree-$4$ to degree-$5$)  this will suffice.

For this, we note that if $f(\x)$ is a polynomial and 
$g(\x) = \mathrm{ReLU}(\vec v\cdot \x)$ for $\vec v$ a unit vector,  then $\E[f(\x)g(\x)]$ is a low-degree polynomial in $\vec v$. 
In fact, by specifying $f$, we can make this into any homogeneous degree-$5$ 
polynomial we desire. This gives us SSE hardness of approximating $\E[f(\x)g(\x)]$.

If $f$ were a Boolean function we would be done. However, as this is not 
the case, we need two additional steps. Firstly, we scale down $f$ and 
truncate it so that its values stay within $[-1,1]$ (note that this 
introduces only a small error if the average size of $f$ is small). 
Second, we replace $f$ by a random Boolean function $\tilde f$ so 
that $\E[\tilde f(\x)] = f(\x)$. Doing this, it is not hard to see that 
with high probability over the randomness of defining $\tilde f$ that 
$\E[\tilde f(\x)g(\x)]$ is arbitrarily close to $\E[f(\x)g(\x)]$ for all 
functions $g$.

Now even if the algorithm was given an explicit description of our 
function $\tilde f$, finding a ReLU function $g$ that approximately maximizes 
$\E[\tilde f(\x) g(\x)]$ is essentially equivalent to approximately 
optimizing a homogeneous degree-$5$ polynomial of the sphere, which is 
SSE-hard.

\section{Related Work}
\label{sec:related-work}

Here we discuss prior and related work that was not already discussed in the introduction.

\paragraph{Comparison to Prior Work} 
We start by providing an explicit comparison with prior work. 

Our algorithmic template involves two steps to agnostically 
learn multi-index models under the Gaussian distribution. 
First, we use queries to ``smooth'' the label function without 
adding a lot of noise to the instance. We then use PCA on the 
expected gradient outer-product of the ``smoothed'' concept 
$\E_{\x \sim \D_\x}[\nabla f(\x) \nabla f(\x)^T]$ 
to find a low-dimensional space containing an (nearly) optimal 
hypothesis.

Using PCA on the expected gradient outer-product  
is a well-known dimension reduction technique that 
has been applied in many supervised learning settings, see, 
e.g.,~\cite{xia2002adaptive,mukherjee2006estimation,mukherjee2006learning,wu2010learning}.
We emphasize that prior results of this type focus on 
(i) the noiseless (realizable) setting, and 
(ii) the case of differentiable target functions. 
In comparison, we perform agnostic learning with non-differentiable functions by crucially exploiting query access. 
Using sample access only, estimating the gradient of 
$f(\x)$ requires exponentially many examples in the dimension, see, e.g.,~\cite{mukherjee2006learning}.

\cite{GKK:08} developed an efficient agnostic query learner for decision trees under the uniform distribution on the Boolean hypercube. The approach of \cite{GKK:08} crucially relies on
the fact that the target hypothesis can be represented as a
sparse polynomial. 
The class of functions we consider (\Cref{def:bounded-surface-area-concepts}) --- and in particular even a single halfspace or ReLU --- does not have this property, and therefore methods relying on sparsity~\cite{KushilevitzMansour:93, GKK:08} are not applicable.

In the context of property testing,
\cite{DeMN21} used a similar approach based on PCA on the expected outer gradient product to test whether
the observed label is close to a smooth low-dimensional
junta (similarly to \Cref{def:bounded-surface-area-concepts}). 
An important difference with the current work is that 
in many interesting applications the link function may assumed 
to be known, e.g., agnostically learning a ReLU or a halfspace, and the goal is to {\em learn} a good hypothesis --- 
a task that information-theoretically requires $\Omega(d)$ samples.
In contrast, \cite{DeMN21} focuses on the semi-parametric task of only testing the unknown link function 
(and not identifying the underlying low-dimensional subspace) 
while avoiding a $\poly(d)$ dependence in the sample complexity.

Finally, related to our setting is the more recent work of \citep{DKKTZ21}, where a combination of polynomial regression 
and PCA on the average outer product of the gradient 
was employed for proper, agnostic learning of a single halfspace with runtime and sample complexity $d^{\poly(1/\epsilon)}$.  
In this work, we crucially exploit the query access 
to bypass the polynomial regression step 
and significantly improve the runtime to 
$\poly(d)2^{\poly(1/\epsilon)}$ (for the special case 
of a single halfspace).

\paragraph{Agnostically Learning Boolean Functions with Queries}
It is known (see, e.g.,~\cite{Feldman08})
that the availability of queries does {\em not} help computationally in the distribution-free agnostic setting. Specifically, Feldman~\cite{Feldman08} showed 
that every concept class that is agnostically 
learnable with queries is also agnostically learnable 
from random samples only (while preserving computational efficiency within a polynomial factor). 
This simple yet powerful fact has motivated the study of 
agnostic query learning 
{\em with respect to specific natural distributions}, 
such as the uniform distribution on the hypercube 
or the Gaussian distribution. 

In the context of learning Boolean functions, the study of distribution-specific
agnostic learning with queries has a rich history.
One of the earliest results in this vein is
the classical algorithm of Goldreich and Levin~\cite{GoldreichLevin:89} that 
uses queries to efficiently agnostically learn parity functions under the 
uniform distribution. (Recall that the problem of learning parities with noise 
is conjectured to be computationally hard with random samples only.) Kushilevitz and Mansour~\cite{KushilevitzMansour:93}, building on the ideas
of~\cite{GoldreichLevin:89}, developed an efficient (non-agnostic) query learner for 
decision trees under the uniform distribution. 
As already mentioned,~\cite{GKK:08} subsequently 
gave a polynomial-time agnostic query learner for 
decision trees under the uniform distribution.

\section{Roadmap, Notation, and Preliminaries}
\subsection{Roadmap}
In \Cref{sec:gradient-queries}, we show that we can use queries
to simulate gradient access to the \OU smoothing $T_\rho y$. 
In \Cref{sec:smoothing-real,sec:smoothing-boolean}, 
we show that the noise operator we use 
does not affect the agnostic learning task for real-valued functions and Boolean concepts.
In \Cref{sec:real-valued}, we show our result for learning real-valued functions 
and prove \Cref{intro-thm:non-proper-real-valued}. In \Cref{ssec:ReLU}, 
we show how \Cref{intro-thm:non-proper-real-valued} implies agnostic learning 
for linear combinations of ReLU activations and deep networks.
In \Cref{sec:geometric-concepts}, we give our agnostic learner for 
Boolean concepts with bounded surface area 
and establish \Cref{intro-thm:non-proper-geometric} and the associated applications. 
In \Cref{sec:hardness}, we show that under the SSE hypothesis, 
no polynomial-time proper query learner for agnostically learning ReLUs or LTFs exists.
In \Cref{sec:ltfs} and \Cref{sec:relu}, we give our result for proper agnostic learning of LTFs and ReLUs.

\subsection{Notation and Preliminaries}\label{sec:prelims}
\paragraph{Basic Notation}
For $n \in \Z_+$, let $[n] \eqdef \{1, \ldots, n\}$.  We use small boldface characters for vectors
and capital bold characters for matrices.  For $\bx \in \R^d$ and $i \in [d]$, $\bx_i$ denotes the
$i$-th coordinate of $\bx$, and $\|\bx\|_2 \eqdef (\littlesum_{i=1}^d \bx_i^2)^{1/2}$ denotes the
$\ell_2$-norm of $\bx$.  We will use $\bx \cdot \by $ for the inner product of $\bx, \by \in \R^d$
and $ \theta(\bx, \by)$ for the angle between $\bx, \by$.  We slightly abuse notation and denote
$\vec e_i$ the $i$-th standard basis vector in $\R^d$.  We will use $\1_A$ to denote the
characteristic function of the set $A$, i.e., $\1_A(\x)= 1$ if $\x\in A$ and $\1_A(\x)= 0$ if
$\x\notin A$.

\paragraph{Asymptotic Notation}
We use the standard $O(\cdot), \Theta(\cdot), \Omega(\cdot)$ asymptotic notation. We also use
$\wt{O}(\cdot)$ to omit poly-logarithmic factors.

\paragraph{Probability Notation}
We use $\E_{x\sim \D}[x]$ for the expectation of the random variable $x$ according to the
distribution $\D$ and $\pr[\mathcal{E}]$ for the probability of event $\mathcal{E}$. For simplicity
of notation, we may omit the distribution when it is clear from the context.  For $(\x,y)$
distributed according to $\D$, we denote $\D_\x$ to be the distribution of $\x$ and $\D_y$ to be the
distribution of $y$. For unit vector $\vec v\in \R^d$, we denote $\D_{\vec v}$ the distribution of
$\x$ on the direction $\vec v$, i.e., the distribution of $\x_{\vec v}$.

\paragraph{Gaussian Space}
 Let $\normal( \boldsymbol\mu, \vec \Sigma)$ denote the $d$-dimensional Gaussian distribution with mean $\boldsymbol\mu\in  \R^d$ and covariance $\vec \Sigma\in \R^{d\times d}$, we denote $\phi_d(\cdot)$ the pdf of the $d$-dimensional Gaussian and we use the $\phi(\cdot)$ for the pdf of the standard normal. In this work we usually consider the standard normal, i.e., $\mu = \vec 0$ and $\vec \Sigma = \vec I$, and therefore, we denote it simply $\normal$. 
We define the standard $L^p$ norms with respect to the Gaussian measure, i.e., $\|g\|_{L^p} = ( \E_{\x \sim \normal} [ |g(\x)|^p)^{1/p}$.
We denote by $L^2(\normal)$ the vector space of all functions $f:\R^d
\to \R$ such that $\E_{\vec x \sim \normal_0}[f^2(x)] < \infty$.  The usual
inner product for this space is
$\E_{\vec x \sim \normal_0}[f(\vec x) g(\vec x)]$.
While, usually one considers the probabilists's or physicists' Hermite polynomials,
in this work we define the \emph{normalized} Hermite polynomial of degree $i$ to be
\(
H_0(x) = 1, H_1(x) = x, H_2(x) = \frac{x^2 - 1}{\sqrt{2}},\ldots,
H_i(x) = \frac{He_i(x)}{\sqrt{i!}}, \ldots
\)
where by $He_i(x)$ we denote the probabilists' Hermite polynomial of degree $i$.
These normalized Hermite polynomials form a complete orthonormal basis for the
single dimensional version of the inner product space defined above. To get an
orthonormal basis for $L^2(\normal)$, we use a multi-index $V\in \N^d$
to define the $d$-variate normalized Hermite polynomial as
$H_V(\vec x) = \prod_{i=1}^d H_{v_i}(x_i)$.  
The total degree of $H_V$ is
$|V| = \sum{v_i \in V} v_i$.
Given a function $f \in L^2$ we compute its Hermite coefficients as
\(
\hat{f}(V) = \E_{\vec x\sim \normal} [f(\vec x) H_V(\vec x)]
\)
and express it uniquely as
\(
\sum_{V \in \N^d} \hat{f}(V) H_V(\vec x).
\)
We denote by $\Pm{k}f(\x)$ the degree $k$ partial sum of the Hermite expansion of $f$,
$\Pm{k} f (\vec x) = \sum_{|V| \leq k} \hat{f}(V) H_V(\vec x)$.
Then, since the basis of Hermite polynomials is complete, we have
\(
\lim_{k \to \infty} \E_{x \sim \normal}[\lp(f(\vec x) - \Pm{k}f(\vec x) \rp)^2] = 0.
\)
Parseval's identity states that
\(
  \E_{\x \sim \normal}[ \lp(f(\x) - \Pm{k}f(\x) \rp)^2 ]
  = \sum_{|V| = k}^{\infty} \hat{f}(V)^2.
\)

\section{From Zero- to First-Order: Derivative Queries via Oracle Queries}
\label{sec:membership-derivatives}

In this section, we show that we can efficiently simulate gradient
access to a smoothed version of the label $y$ using queries.  In \Cref{sec:gradient-queries} we show how to use the
\OU operator to get acecss to gradient queries of $y$. 
In \Cref{sec:smoothing-boolean} and \Cref{sec:smoothing-real} we show that the noise that we introduce in order to simulate the gradient queries does not affect the agnostic learning task for Boolean and real valued concepts as long as the Gaussian surface area (for Boolean concepts) and the expected gradient norm (for real-valued functions) are bounded.

\subsection{Gradient Queries via Oracle Queries}
\label{sec:gradient-queries}
We first formally define the \OU smoothing operator.

\begin{definition}[\OU\ Operator]
\label{def:OU-operator}
Let $\rho \in (0,1)$.
    We denote as $T_\rho$ the linear operator that maps a function $g \in L^2(\normal)$ to the function
    $T_\rho g$ defined as:
    \[
    (T_\rho g) (\vec x) \eqdef\E_{\vec z\sim \normal}\left[g(\sqrt{1-\rho^2}\x+\rho\vec z)\right]\;.
    \]
    To simplify notation, we often write 
    $T_\rho g (\vec x) $ instead of  $(T_\rho g) (\vec x) $.
\end{definition}

The \OU operator is well studied (see, e.g., \cite{Bog:98, KOS:08} and  references therein) and has several structural properties that enable the analysis of our algorithm. 
Its crucial property is that regardless of how complex the initial function $g$ is, $T_\rho g$ is always everywhere differentiable
and also the norm of the gradient of $T_\rho g$ only depends on
the maximum value of the function $g$.  In the next fact we collect the
properties that we use.  

\begin{fact}[see, e.g.,~\cite{Bog:98}]\label{fct:ou-smooth}
    Let $g:\R^d\mapsto\R$. For the function $T_\rho g(\x)$ the following properties hold
    \begin{enumerate}
        \item $T_\rho g(\x)$ is differentiable at every point $\x$.
        \item $T_\rho g(\x)$ is $1/\rho$-Lipschitz, i.e., $\|\nabla T_\rho g(\x)\|_2\leq \|g\|_\infty/\rho$.
        \item For any $p\geq 1$, $T_\rho$ is a contraction with respect the $\|\cdot\|_p$, i.e., it holds $\|T_\rho g\|_{L^p}\leq \|g\|_{L^p}$.
    \end{enumerate}
\end{fact}

Using it allows the gradient of the smoothed function $T_\rho g(\x)$ to be computed  directly given value access to the underlying function $g$.  We now present the main result of this section showing that given query access to the label $y(\cdot)$ we can efficiently simulate gradient queries to
the smoothed label $T_\rho y(\cdot)$ with roughly $\wt{O}(d/\epsilon)$ queries.

\begin{lemma}[Gradient Queries from Oracle Queries]\label{lem:estimation}
Fix $\epsilon, \delta, \rho > 0$.
Let $y(\x): \R^d \mapsto \R$ be a function in $L_2^2(\normal)$ with
$|y(\x)| \leq M$.  There exists an algorithm
(see \Cref{alg:gradient-queries}) that given a point $\x \in \R^d$ 
makes $N = \wt{\Omega}(dM/\epsilon)\log(1/\delta)$ queries to $y(\x)$ and, in polynomial time, returns a vector $\wt{\vec \xi}$ such that, with probability at least $1-\delta$, it holds $\| \wt{\vec \xi} - \nabla T_\rho y(\x)\|_2 \leq \epsilon$.
\end{lemma}
\begin{proof}
  To show the lemma, we first need to show that for any point $\x\in \R^d$, we can use enough queries to estimate $\mderiv y(\x)$ accurately, meaning that we need to estimate the random variable $\vec Z=\frac{\sqrt{1-\rho^2}}{\rho}\E_{\vec z\sim \normal{(\vec 0,\vec I)}}\left[y(\sqrt{1-\rho^2}\x+\rho\vec z)\vec z\right]$ accurately. Note that by definition the random variable $\vec Z$ is $1/\rho^2$ sub-gaussian, therefore from a simple application of the Hoefding inequality, we get that with $O(dM/(\rho\eps)^2\log(1/\delta_1))$ queries, we can find a $\widetilde{\vec Z}$ such that
    $\|\widetilde{\vec Z}-\E[\vec Z]\|_2\leq \eps$ with probability at least $1-\delta_1$.
\end{proof}

\begin{lemma}[Gradient of Smoothed Label]
Let $\rho \in (0,1)$.
    We denote as $D_\rho$ the linear operator that maps a function $g \in L^2(\normal)$ to the function
    $D_\rho g$ defined as:
    $(D_\rho g) (\vec x) \eqdef \nabla (T_{\rho}g)(\vec x).$
    It holds that 
    \[
    (D_\rho g) (\vec x) 
    =\frac{\sqrt{1-\rho^2}}{\rho}\E_{\vec z\sim \normal}\left[g(\sqrt{1-\rho^2}\x+\rho\vec z)\vec z\right]\;.
    \]
    To simplify notation, we often write 
    $D_\rho g (\vec x) $ instead of  $(D_\rho g) (\vec x) $.
\end{lemma}
\begin{proof}
We first observe that for any fixed $\x$ the random variable 
$\sqrt{1-\rho^2} \vec x + \rho \vec z$ is distributed
according to $\normal(\sqrt{1-\rho^2} \vec x, \rho^2 \vec I)$.  Therefore, we have 
\[
T_\rho g(\x) = 
\E_{\z \sim \normal}
[g(\sqrt{1-\rho^2}\x + \rho \vec z)]
= 
\E_{\vec u \sim \normal(\sqrt{1-\rho^2} \x, \rho^2 \vec I)} [g(\vec u)]
\]
We can now directly compute the gradient of the smoothed function $T_\rho g$:
\begin{align*}
\nabla_{\x} (T_\rho g) (\x) 
&= \nabla_{\x}
\E_{\vec u \sim \normal(\sqrt{1-\rho^2} \x, \rho^2 \vec I)} [g(\vec u)]
= 
\frac{\sqrt{1-\rho^2} }{\rho^2}
\E_{\vec u \sim \normal(\sqrt{1-\rho^2} \x, \rho^2 \vec I)} \left[
g(\vec u) (\vec u- \sqrt{1-\rho^2} \vec x) \right]
\\
&=
\frac{\sqrt{1-\rho^2}}{\rho}
\E_{\vec z \sim \normal} \left[g(\sqrt{1-\rho^2}\x + \rho \vec z) \vec z  \right]
\,.
\end{align*}
\end{proof}

\begin{Ualgorithm}
	\centering
	\fbox{\parbox{6in}{
			{\bf Input:}   $\eps>0$, $\delta>0$, $\rho > 0$, location $\x \in \R^d$. \\
			{\bf Requries:}  Sample and query access to distribution of labeled examples $\D$\\
			{\bf Output:}  An estimation $\wt{\vec \xi} =  \nabla T_\rho y(\x)$
   such that $\|\wt {\vec \xi} - \nabla T_\rho y(\x)\|_2 \leq \epsilon$.
                    \begin{enumerate}
                        \item Sample $N = \wt{O}(d/\epsilon) \log(1/\delta)$ points $\vec z^{(1)},\ldots, \vec z^{({N})} \sim \normal$.
                        \item Perform $N$ Queries at the locations
                        $\vec q^{(j)} = \sqrt{1- \rho^2} \vec x + \rho \vec z^{(j)}$
                        and obtain $y^{(j)}$.
                    \item  Return  the empirical estimate 
                        \(
                        \wt{\xi}
                        = \frac{\sqrt{1-\rho^2}}{ N \rho}\sum_{j=1}^{N}
                        y^{(j)} \vec z^{(j)} \,.\)
			\end{enumerate}
	}}
 \medskip
	\caption{Simulating Gradient Queries with Queries}
	\label{alg:gradient-queries}
\end{Ualgorithm}

\subsection{Smoothing the Labels for Learning Real-valued Functions}
\label{sec:smoothing-real}

In this section we show that adding noise to the label $y(\x)$ in order 
to make it smooth and compute its gradients does not ``change'' the agnostic learning task significantly. 
Assume that there exists a learning algorithm that can learn a hypothesis
$h(\cdot)$ that achieves $\epsilon$-excess error compared to a class of concepts $C$, given access to the smooth labels $T_\rho y(\x)$.  In other words,
assume that we are given a learner that finds a  hypothesis $h(\cdot)$
that satisfies 
\[
\E_{\x \sim \normal}[(h(\x) - T_\rho y(\x))^2]
\leq 
\inf_{f \in C} \E_{\x \sim \normal}[(f(\x) - T_\rho y(\x))^2]
+ \epsilon \,.
\]
Then, can we say that $h(\cdot)$ will perform well compared to the same class $C$ under the original (non-smooth) label $y(\cdot)$?
We show that this is true when (i) the hypothesis $h(\cdot)$ produced by the learner is not very complicated in the sense that it has bounded variation and (ii) the hypothesis class $C$ that we are comparing $h(\cdot)$ against has also bounded variation. 

In particular, we show that a hypothesis $h(\cdot)$ 
achieves $\epsilon$-excess error compared to some concept class $C$
in the \emph{smoothed} instance,  achieves $(\epsilon + O(\sqrt{\rho})$-excess error with respect to the original instance.  In other words, as long as the variation and $L_2^2$ norms of the target concept class and the hypothesis produced by the 
learner are bounded, smoothing the noisy label $y(\x)$ does not introduce significantly more noise to the instance.  To simplify notation, we first define the excess error, i.e., the error
of a classifier minus the error of the best-in-class classifier of some class $C$.
\begin{definition}[Excess Error]
\label{def:excess-error}
Given hypotheses $h, f :\R^d \mapsto \R$ we define the $L_1$-excess error of $h(\cdot)$ compared to $f(\cdot)$ with respect to the label $y(\cdot)$ 
to be $\mathcal{E}_1(h, f;y) = \E_{\x \sim \normal}[|h(\x) - y(\x)|]
- \E_{\x \sim \normal}[|f(\x) - y(\x)|]$.   Moreover, for a class of
concepts $C$ we define the excess error of $h(\cdot)$ compared to $C$ with respect to $y(\cdot)$ as $\sup_{f \in C} \mathcal{E}_1(h, f;y)$.
Similarly, we define the $L_2^2$-excess error as 
$\mathcal{E}_2(h, f;y) = \E_{\x \sim \normal}[(h(\x) - y(\x))^2]
- 
\E_{\x \sim \normal}[(f(\x) - y(\x))^2]
$
and 
$\mathcal{E}_2(h, C ;y) = \sup_{f \in C} \mathcal{E}_2(h, f;y)$.
\end{definition}

We now show that that the \OU noise operator also preserves the $L_2^2$-excess error of a classifier $h:\R^d \mapsto \R$ as long as the target class and the classifier $h$ have bounded expected gradient.

\begin{proposition}[Smoothing the Noisy Labels]\label{lem:smoothing-real}
Fix $f \in \mathcal{R}(\valb,\gradb, k)$.
Let $y: \R^d \mapsto \R$ be a function in $L^2(\normal)$ with
$\E_{\x \sim \normal}[y^2(\x)] \leq \valb$.
Moreover, let $p(\x): \R^d \mapsto \R$ be an almost everywhere differential function in $L_2(\normal)$ with $\Exn[\|\nabla p(\x)\|_2^2]\leq \gradb$. It holds that
\[
\mathcal{E}_2(p, C ; y)
\leq 
\mathcal{E}_2(p, C;T_\rho y) + O(\sqrt{\rho \valb \gradb}) \;.
\]
\end{proposition}
    \begin{proof}[Proof of \Cref{lem:smoothing-real}] 

We first prove the following lemma that connects the excess error 
of a real-valued function $h(\cdot)$ with respect to the smoothed label $T_\rho y(\cdot)$ to its excess error with respect to the original label $y(\cdot)$.  
If the operator $T_\rho$ 
preserves the correlation of all concepts $f\in C$, i.e.,
$|\E_{\x \sim \normal}[f(\x) y(\x)] - 
\E_{\x \sim \normal}[ f(\x) T_\rho y(\x) ]| \leq \eps$  for all $f \in C$
and it also preserves the correlation of the hypothesis $h(\cdot)$, i.e.,
$
|
\E_{\x \sim \normal}[h(\x) y(\x)] - 
\E_{\x \sim \normal}[ h(\x) T_\rho y(\x) ] | \leq \eps$, then
the excess error of $h(\cdot)$ with respect to $y(\cdot)$ is at most 
$2 \eps$ worse than its excess error with respect to the smoothed label
$T_\rho y(\cdot)$. In the following lemma, we show that we can connect the $L_2$-excess error with the correlation of concepts.

\begin{lemma}[From Excess Error to Correlation Preservation]
\label{lem:excess-real}
Let $h: \R^d \mapsto \R$ be a real-valued hypotheses 
and $C$ be a class of real-valued hypotheses.
It holds 
\begin{align*}
\mathcal{E}_2(h, C; T_\rho y) - \mathcal{E}_2(h, C; y) \leq  
2\sup_{f \in C} 
\Big| &\E_{\x \sim \normal}[ f(\x) T_\rho y(\x) ] - \E_{\x \sim \normal}[ f(\x) y(\x) ]  \Big| + 
\\
&2 \Big| \E_{\x \sim \normal}[ h(\x) T_\rho y(\x) ] - \E_{\x \sim \normal}[ h(\x) y(\x) ]  \Big|
\,. 
\end{align*}
\end{lemma}
\begin{proof}
We first note that 
$\mathcal{E}_2(h, C;T_\rho y) - \mathcal{E}_2(h, C; y) =
\sup_{f \in C} \mathcal{E}_2(h, f; T_\rho y)-
\sup_{f \in C} \mathcal{E}_2(h, f;  y)
\leq  \sup_{f \in C} 
\big|\mathcal{E}_2(h, f; T_\rho y) - \mathcal{E}_2(h, f;  y) \big|$.
For some fixed concept $f \in C$, we have 
\begin{align*}
\mathcal{E}_2(h, f; T_\rho y) = 
\E_{\x \sim \normal}[h^2(\x)] 
- \E_{\x \sim \normal}[f^2(\x)]
+ 2  \E_{\x \sim \normal}[( f(\x)- h(\x) ) (T_\rho y)] \,.
\end{align*}
Therefore, we have 
\begin{align*}
\mathcal{E}_2(h, f; T_\rho y) 
&-
\mathcal{E}_2(h, f; y) 
\\
&= 
2  \left(\E_{\x \sim \normal}[f(\x) (T_\rho y(\x) - y(\x) )] 
+ 
\E_{\x \sim \normal}[h(\x) (T_\rho y(\x) - y(\x) )] 
\right)
\,.
\end{align*}
By taking the supremum over the $f$, we complete the proof.
\end{proof}
Note that $\E_{\x \sim \normal}[f(\x) (T_\rho y(\x) - y(\x) )]=\E_{\x \sim \normal}[y(\x) (T_\rho f(\x) - f(\x) )]$. Therefore, using \CS inequality we have that
\begin{align*}
    \E_{\x \sim \normal}[y(\x) (T_\rho f(\x) - f(\x) )]&\leq 
\left(\E_{\x \sim \normal}[y^2(\x)]\E_{\x \sim \normal}[ (T_\rho f(\x) - f(\x) )^2]\right)^{1/2}
\\&\leq \sqrt{\valb}\left(\E_{\x \sim \normal}[ (T_\rho f(\x) - f(\x) )^2]\right)^{1/2}\;,
\end{align*}
where we used that $\E_{\x \sim \normal}[y^2(\x)]\leq \sqrt{\E_{\x \sim \normal}[y^4(\x)]}\leq \valb$.  To bound the remaining term, we prove the following claim.
\begin{claim}\label{clm:difference} Let $f \in L^2(\normal)$ be a continuous and  (almost everywhere) differentiable function.  Then, $\E_{\x \sim \normal}[( T_\rho f(\x) - f(\x) )^2]\leq 2 \rho^2 \E_{\x \sim \normal}[\|\nabla f(\x)\|_2^2 $.
 \end{claim}
 \begin{proof}
 
We will use the following result from \cite{KTZ19}.
\begin{fact}[Correlated Differences, (Lemma 7 in \cite{KTZ19})]
\label{fct:correlated-differences}
Let $f \in L^2(\normal)$ be an (almost everywhere) differentiable
function.  
Denote by \[D_\tau = \normal
\bigg( 
\vec 0, 
\begin{pmatrix}
\vec{I} & (1 - \tau) \vec{I} \\ 
(1-\tau) \vec{I} & \vec{I}
\end{pmatrix}
\bigg) \,.
\]
It holds 
\(
\E_{(\x,\vec z) \sim D_\tau}[(f(\x) - f(\vec z))^2]
\leq 2 \tau ~ \E_{\x \sim \normal}[\|\nabla f(\x)\|_2^2]
\,.
\)
\end{fact}
Therefore, using Jensen's inequality, we have that
\begin{align*}
\E_{\x \sim \normal}[ (T_\rho f(\x) - f(\x) )^2 ]
=
\E_{\x \sim \normal}[ (\E_{\vec z \sim \normal}[
f(\sqrt{1-\rho^2} \x + \rho \vec z)] - f(\x) )^2 ]
\leq 
\E_{(\x, \vec z') \sim D_\tau}[ (f(\vec z') - f(\x))^2 ] \,, 
\end{align*}
for $ \tau = 1- \sqrt{1-\rho^2}$.
Therefore, using \Cref{fct:correlated-differences}, we obtain
\[
\E_{\x \sim \normal}[(T_\rho f(\x) - f(\x))^2]
\leq 2 (1- \sqrt{1-\rho^2}) 
\E_{\x \sim \normal}[\|\nabla f(\x)\|_2^2 
\leq 2 \rho^2 \E_{\x \sim \normal}[\|\nabla f(\x)\|_2^2 \,,
\]
where we used the fact that $\sqrt{1-\rho^2} \geq 1-\rho^2$ 
which holds for all $\rho \in [0,1]$ and implies that
$1- \sqrt{1- \rho^2} \leq \rho^2$.
 \end{proof}
Therefore, from \Cref{clm:difference}, we have that
\begin{align*}
\mathcal{E}_2(p, C ; y)
\leq 
\mathcal{E}_2(p, C;T_\rho y) + O(\sqrt{\rho \valb})
\Big(\sqrt{\Exn[\|\nabla f(\x)\|_2^2]}+\sqrt{\Exn[\|\nabla p(\x)\|_2^2]}\Big)\;.
\end{align*}
Using that $\Exn[\|\nabla f(\x)\|_2^2],\Exn[\|\nabla p(\x)\|_2^2]\leq \gradb$, we complete the proof of \Cref{lem:smoothing-real}.
    \end{proof}

\subsection{Smoothing Labels for Learning Boolean Concepts} 
\label{sec:smoothing-boolean}

The following proposition shows that the $L_1$-excess error 
of a hypothesis $h$ with respect to the original label $y$
is close to its $L_1$-excess error with respect to the smoothed 
label $T_\rho y$ as long as (i) the class $C$ contains concepts 
with bounded surface area and (ii) the classifier $h$ also has bounded surface area.
\begin{proposition}[Smoothing the Noisy Labels Preservs $L_1$-Excess Error]\label{lem:smoothing-boolean}
Fix $y: \R^d \mapsto \{\pm 1\}$ and let $C$ be a class of Boolean concepts. 
It holds 
\[
\mathcal{E}_1(h, C ; y)
\leq 
\mathcal{E}_1(h, C;T_\rho y)
+ O(\rho) ~ (\Gamma(C) + \Gamma(h))
\,,
\] 
where $\mathcal{E}(\cdot, \cdot; \cdot)$ is the excess error defined in 
\Cref{def:excess-error}
\end{proposition}

\begin{proof}
    
We first prove the following lemma showing that connects the excess error 
of a classifier $h(\cdot)$ with respect to the smoothed label $T_\rho y(\cdot)$ to its excess error with respect to the original label $y(\cdot)$.  This is analogous to the real-valued case (\Cref{lem:excess-real}).
In the following lemma we show that we can connect the $L_1$-excess error with the correlation of concepts (which basically relies on
the identity $|t-s| = 1 - ts $ when $t \in [-1,1]$ and $s \in \{\pm 1\}$.
\begin{lemma}[From Excess Error to Correlation Preservation: Boolean Concepts]\label{lem:boolean-excess}
Let $h: \R^d \mapsto \{\pm 1\}$ and $C$ be a class of Boolean hypotheses.  It holds 
\begin{align*} 
\mathcal{E}_1(h, C; T_\rho y) - \mathcal{E}_1(h, C; y) \leq  
\sup_{f \in C}
\Big|
\E_{\x \sim \normal}[ f(\x) T_\rho y(\x) ]
 &- \E_{\x \sim \normal}[ f(\x) y(\x) ]  \Big| +
\\
&\Big|
\E_{\x \sim \normal}[ h(\x) T_\rho y(\x) ]
- \E_{\x \sim \normal}[ h(\x) y(\x) ]  \Big|
\,.
\end{align*}
\end{lemma}
\begin{proof}
    We first note that 
$\mathcal{E}_1(h, C;T_\rho y) - \mathcal{E}_1(h, C; y) =
\sup_{f \in C} \mathcal{E}_1(h, f; T_\rho y)-
\sup_{f \in C} \mathcal{E}_1(h, f;  y)
\leq  \sup_{f \in C} 
\big|\mathcal{E}_1(h, f; T_\rho y) - \mathcal{E}_1(h, f;  y) \big|$. 
Using the fact that $\E_{\x \sim \normal}[|f_1(\x) - f_2(\x)|]=1-\E_{\x \sim \normal}[f_1(\x)f_2(\x)]$, for any functions $f_1:\R^d\mapsto [-1,1]$ and $f_2:\R^d\mapsto \{\pm 1\}$, we have that
\begin{align*}
    \mathcal{E}_1(h, f; T_\rho y)=\E_{\x \sim \normal}[|T_\rho y(\x) - h(\x)|]-\E_{\x \sim \normal}[|T_\rho y(\x) - f(\x)|]&=\E_{\x \sim \normal}[T_\rho y(\x) f(\x)]-\E_{\x \sim \normal}[T_\rho y(\x) h(\x)]\;.
\end{align*}
Therefore, for some concept $f\in C$, we have that
\begin{align*}
\big|\mathcal{E}_1(h, f; T_\rho y) - \mathcal{E}_1(h, f;  y) \big|=\big|\E_{\x \sim \normal}[(T_\rho y(\x)-y(\x)) f(\x)] \big|+\big|\E_{\x \sim \normal}[(T_\rho y(\x)-y(\x)) h(\x)] \big|\;.
\end{align*}
Taking the supremum over the $C$ completes the proof.
\end{proof}
First, note that since $|y(\x)|\leq 1$, it also holds that $|T_\rho y(\x)|\leq 1$.
Using \Cref{lem:boolean-excess}, we have that \Cref{lem:smoothing-boolean} is equivalent to showing that 
for a Boolean function $f: \R^d \mapsto \{\pm 1\}$ it holds 
\(
|\E_{\x \sim \normal}[(T_\rho y(\x)-y(\x)) f(\x)]|\leq O(\rho) ~ \Gamma(f) \;.
\) We do this in the following lemma.
\begin{lemma}[$T_\rho$ Preserves Correlation]\label{lem:smoothing-thelabels-surface}
Let $y: \R^d \mapsto \{\pm 1\}$ and let $f: \R^d \mapsto \{\pm 1\}$ be a (Borel) Boolean function. It holds that 
\[
\Big|\E_{\x \sim \normal}[ f(\x)  T_\rho y(\x)] -
\E_{\x \sim \normal}[f(\x) y(\x)]\Big|
\leq O(\rho) ~ \Gamma(f)  \,.
\]
\end{lemma}
\begin{proof}
Using the fact that the \OU noise operator $T_\rho$ is a symmetric linear operator on $L^2(\normal)$, we have 
\begin{align*}
\E_{\x \sim \normal}[f(\x) T_\rho y(\x)] 
=
\E_{\x \sim \normal}[y(\x) T_\rho f(\x)] 
= 
\E_{\x \sim \normal}[y(\x) f(\x)]
+ 
\E_{\x \sim \normal}[y(\x) (T_\rho f(\x) - f(\x))] \,.
\end{align*}
Therefore, 
\[
\Big|\E_{\x \sim \normal}[ f(\x)  T_\rho y(\x)] -
\E_{\x \sim \normal}[f(\x) y(\x)]\Big|
= 
\Big|
\E_{\x \sim \normal}[y(\x) (T_\rho f(\x) - f(\x))]
\Big| \leq 
\E_{\x \sim \normal}[|T_\rho f(\x) - f(\x)|] \,,
\]
where, for the inequality we used the fact that the label $y(\x) \in \{\pm 1\}$.
We next bound the term  $\E_{\x \sim \normal}[|T_\rho f(\x) - f(\x)|]$. 
We will use the following result from Ledoux and Pisier as stated in \cite{KOS:08}.
\begin{fact}[Ledoux-Pisier \cite{ledoux1994}] \label{lem:ledoux-pisier}
Let $f: \R^d \mapsto \{\pm 1\}$ be a Boolean function. It holds 
\(
\E_{\x \sim \normal}[ f(\x) T_\rho f(\x)] \geq 1 - 2 \sqrt{\pi} ~ \Gamma(f) ~ \rho \,.
\)
\end{fact}

In what follows, we denote by $K$ the set labeled as positive by the LTF $f(\x)$.
Using the fact that $\E_{\x \sim \normal}[|T_\rho f(\x) - f(\x)|]=1-\E_{\x \sim \normal}[T_\rho f(\x)f(\x)]$, which holds because  
$|T_\rho f(\x)|\leq 1$ and $f(\x) \in \{\pm 1\}$, we have 
\begin{align*} 
\E_{\x \sim \normal}[|T_\rho f(\x) - f(\x)|] 
&= 
1 - \E_{\x \sim \normal}[ f(\x) T_{\rho}f(\x) ]
\leq  O(\rho \Gamma(f)) \,,
\end{align*}
where the inequality  follows from \Cref{lem:ledoux-pisier}.

\end{proof}
Applying \Cref{lem:smoothing-thelabels-surface} on $f$ and $g$ gives the result.
\end{proof}

 \section{Agnostically Learning Real-valued Multi-index Models} \label{sec:real-valued}

In this section we present our algorithmic result \Cref{intro-thm:non-proper-real-valued} for learning real-valued function classes in the $L_2^2$ norm.  For convenience, we first restate the  class of bounded variation concepts that we consider.  

\begin{definition}[Bounded Variation, Low-Dimensional Concepts]
Fix $L, M>0$ and $k \in \Z_+$. 
We define the class $\mathfrak R(\valb, \gradb, k)$ 
of continuous, (almost everywhere) differentiable real-valued functions with the following properties:
\begin{enumerate}
\item 
For every $f \in \mathfrak{R}(\valb, \gradb, k)$,  it holds 
$(\E_{\x \sim \normal^d}[f^4(\x)])^{1/2} \leq \valb$ 
and 
$\E_{\x \sim \normal^d}[\|\nabla f(\x)\|_2^2] \leq\gradb$.
\item 
There exists a subspace $U$ of $\R^d$ of dimension 
at most $k$
such that $f$ depends only on $U$, i.e., 
for every $\x \in \R^d$, $f(\x) = f(\proj_U \x)$.
\end{enumerate}
\end{definition}

We now state the main result of this section (the formal version of 
\Cref{intro-thm:non-proper-real-valued}).
\begin{theorem}[Improper Learner for Real-valued Functions]\label{thm:non-proper-real-valued}
Fix $k \in \mathbb N$ and $\valb,\gradb \in \R^+$.
Let $D$ be a distribution on $\R^d\times \R^+$ such that the 
$\x$-marginal of $D$ is standard $d$-dimensional normal. 
There exists an algorithm that makes $N_q = \poly(d/\eps)$ queries, draws $N_s = \poly(d) + \poly((k \valb/\eps)^{L^2/\eps^4}, 1/\eps, \log(1/\delta))$ samples from $D$, runs in time $\poly(N_s,N_q,d)$ and outputs a polynomial $p:\R^d\mapsto\R$ so that with probability at least 
$1-\delta$ it holds
\[
\E_{(\x,y)\sim D}[(p(\x)- y)^2]\leq \inf_{f\in \mathfrak R(\valb,\gradb, k) }\E_{(\x,y)\sim D}[(f(\x)- y)^2]+\eps\;.
\]
\end{theorem}

Before we proceed to the proof we define the Hermite expansion operator
that maps a function $f$ to its degree $m$ Hermite polynomial.
\begin{definition}[Hermite Expansion Operator]
Given a function $f \in L^2(\normal)$, we denote by $P_{m}(f)(\x)$, the linear operator that maps $f$ to the Hermite polynomial of degree $m$ of $f$, i.e.,
\[
(\Pm{m} f)(\x) = \sum_{|I| \leq m} \widehat{f}(I) H_I(\x),
\]
where $H_I$ is the multivariate Hermite polynomial of degree $I \in \mathbb N^{d}$ and $\widehat{f}(I) = \E_{\x \sim \normal}[f(\x) H_I(\x)]$ is the corresponding Hermite coefficient of $f(\x)$.
\end{definition}

The following lemma bounds the error of the polynomial approximation of degree $m$ for ``smooth'' functions. Its proof is implicit in \cite{KTZ19}; 
we provide a short proof for completeness.
\begin{lemma}[Polynomial Approximation of Smooth Functions]
\label{lem:low-degree-approximation-smooth-functions}
Let $f(\x): \R^d \mapsto \R$ be an (almost everywhere) differentiable function and $m \in \mathbb N$. It holds 
\[
\E_{\x\sim \normal}[(f(\x)-\Pm{m} f (\x))^2]\leq O\Big(\frac{1}{m}\Big) \E_{\x \sim \normal}[\| \nabla f(\x)\|_2^2] \,.
\]
\end{lemma}
\begin{proof}
We denote as $\Pm{> m}f$ the Hermite expansion of $f$, which contains the terms with degrees higher than $m$. We have that
\begin{align*}
    \E_{\x\sim \normal}[(f(\x)-\Pm{m} f (\x))^2]=\E_{\x\sim \normal}[(\Pm{>m} f (\x))^2]= \sum_{I: |I| > m}  (\wh{f}(I))^2\leq  \frac{1}{m}\sum_{I: |I| > m}  |I|(\wh{f}(I))^2\;,
\end{align*}
where in the last inequality, we used that $1\leq |I|/m$. 
Furthermore, (see, e.g., the proof of Lemma 6 in \cite{KTZ19}) 
we have that for a continuous and (almost everywhere) differentiable 
function $f$, it holds that 
\[
\E_{\x \sim \normal}[\| \nabla f(\x)\|_2^2]=\sum_{I\in \N^d}  |I|(\wh{f}(I))^2\;.\] 
Combining the above, the result follows.
\end{proof}

As we discussed in \Cref{sec:techniques} to show that an approximately optimal, low-dimensional concept exists we will use the Gaussian Marginalization Operator defined below.
\begin{definition}[Gaussian Marginalization Operator]
\label{def:gaussian-marginalization-operator}
Let $U$ be a subspace of $\R^d$.  Denote by 
$D_{U^\perp}$ the standard normal distribution on
the subspace $U^\perp$ (we assume that a vector 
$\vec z \sim D_{U^\perp}$ is a $d$-dimensional vector
that lies in $U^\perp$).
Given a function $f\in L^2(\normal)$, we denote
by $\Pi_U f$ the linear operator defined by
\[ 
(\Pi_U f)(\x) = \E_{\vec z \sim D_{U^\perp}}[f(\proj_{U}(\x) + \vec z )] \,. \]
\end{definition}

\paragraph{Motivation about the Gaussian Marginalization Operator, $\Pi_V$}
By the assumption of \Cref{inf-prop:dimension-reduction}, 
all directions in the orthogonal complement $V^\perp$ are low-influence, 
i.e., for $\vec h \in V^\perp$ it holds 
$\E_{\x \sim \normal}[(\vec h \cdot \nabla \wt y(\x))^2] \leq O(\epsilon^2/k)$.
In words, the function $\wt y$ is ``approximately constant'' 
along some low-influence direction $\vec h$. 
Let us first assume that $\wt y$ is exactly constant on all directions 
of $V^\perp$.  Then, in order to preserve
the correlation of $\wt y$ with $f$,  
\emph{we only need to match the expected value of $f$ over $V^\perp$}. 
This motivates the following ``Gaussian Marginalization Operator''
of \Cref{def:gaussian-marginalization-operator}.
Indeed, if $\wt y$ was constant on $V^\perp$, using the fact that 
$\proj_V \x$ and $\proj_{V^\perp} \x$ are independent standard Gaussians, 
we would obtain that
$$
\E_{\vec z \sim \normal}[\E_{\x \sim \normal}f(\proj_V(\x) + \proj_{V^\perp}(\vec z) ) \wt y (\x)]] - 
\E_{\x\sim \normal}[f(\x) \wt y (\x) ]
= 0  \;.$$
We observe that since $\Pi_V f$ is a convex combination 
of different translations of $f$ and $\mathfrak{B}(\Gamma, k)$ 
is closed under translations, we obtain that the Gaussian surface area of $f$  is also bounded above by $\Gamma$.

In the next lemma, we collect some useful properties of the Gaussian Marginalization Operator.
\begin{lemma}\label{lem:properties-of-the-operators}
    Let $g\in L^2(\normal)$  and $V\subseteq \R^d$. We have the following properties for the operator $\Pi_V$.
    \begin{itemize}
        \item $\Pi_V$ are contractions, i.e., $\E_{\x\sim \normal}[(\Pi_V g(\x))^2]\leq \E_{\x\sim \normal}[ g^2(\x)]$.
\item Let $U,V\subseteq \R^d$, it holds that $\Pi_V\Pi_{U+V} g=\Pi_{V+U}\Pi_{V+U^\perp} g=
        \Pi_V g$.
    \end{itemize}
\end{lemma}
\begin{proof}
    To show that $\Pi_V$ is a contraction, note that 
    $\E_{\x\sim \normal}[(\Pi_V g(\x))^2]\leq \E_{\x\sim \normal}\E_{\vec z\sim \normal_{V^\perp}}[ g^2(\x_V+\vec z)]=\E_{\x\sim \normal}[ g^2(\x)]$, where we used Jensen's inequality.
    For the second part, let $H=(U+V)^\perp$ and note that
    \begin{align*}
        \Pi_V\Pi_{U+V} g&=\Pi_V\E_{\vec z\sim \normal_{H}}[g(\x_{U+V} +\vec z)]=\Pi_V\E_{\vec z\sim \normal_{H}}[g(\x_{V} +\x_{U/V}+\vec z)]
        =\E_{\vec z\sim \normal_{H}}[\Pi_V g(\x_{V} +\x_{U/V}+\vec z)]\\&=\E_{\vec z\sim \normal_{H}}[\E_{\vec z'\sim \normal_{U/V}}[ g(\x_{V} +\vec z'+\vec z)] =\Pi_{U+V}\Pi_{U^\perp +V}g\;,
    \end{align*}
    where we used Fubini's theorem.
\end{proof}

\begin{lemma}\label{lem:properties-of-the-operators-2}
    Let $g\in L^2(\normal)$, $m\in \N$ and $V\subseteq \R^d$. We have the following properties for the operators $\Pm{m}$ and $\Pi_V$.
    \begin{itemize}
        \item $\Pm{m}$ is a contraction, i.e., $\E_{\x\sim \normal}[(\Pm{m} g(\x))^2]\leq \E_{\x\sim \normal}[ g^2(\x)]$.
        \item $\Pm{m}$ and $\Pi_V$ commute, i.e., $\Pm{m} \Pi_V g= \Pi_V \Pm{m} g$.
\end{itemize}
\end{lemma}
\begin{proof}
       First, we show that $\Pm{m}$ is a contraction. Using the $g\in L^2(\normal)$, we have that $g$ admits a Hermite expansion. We denote as $\Pm{> m}g$ the Hermite expansion of $g$, which contains the terms with degrees higher than $m$. We have that
    \[
    \E_{\x\sim \normal}[ g^2(\x)]=\E_{\x\sim \normal}[ (\Pm{m} g(\x) + P_{>m}g(\x))^2]=\E_{\x\sim \normal}[ (\Pm{m} g(\x))^2 + (P_{>m}g(\x))^2]\;,
    \]
    where in the last equality, we used that the Hermite basis is orthogonal, hence $\E_{\x\sim \normal}[\Pm{m} g(\x)P_{>m}g(\x)]=0$. Therefore,  
    $\E_{\x\sim \normal}[ g^2(\x)]\geq \E_{\x\sim \normal}[ (\Pm{m} g(\x))^2]$.

Next, we show that $\Pi_U$ and $\Pm{m}$ commute.
\begin{claim}[$\Pm{m}$ and $\Pi_U$ commute]
\label{lem:projection-hermite-commutativity}
Let $g\in L^2(\normal)$, $m \in \mathbb N$,
and $V$ be a subspace of $\R^d$. 
It holds that $\Pm{m} \Pi_V g = \Pi_V \Pm{m} g$.
\end{claim}
\begin{proof}
Because $\Pm{m}$ and $\Pi_V$ are linear operators, it suffices to show the above each term of the Hermite basis, i.e.,
\[
\E_{\x\sim \normal}[\Pi_V g(\x)H_{I}(\x)]H_{I}(\x)=\E_{\vec z\sim \normal_{V^\perp}}[\E_{\x\sim \normal}[g(\x)H_{I}(\x)]H_{I}(\x_V+\vec z)]\;.
\]
Note that if $H_{I}(\x)$ does not depend on $V^\perp$, then $\E_{\vec z\sim \normal_{V^\perp}}[H_{I}(\x_V+\vec z)]=H_{I}(\x)$. Therefore, we have 
    \begin{align*}
        \E_{\vec z\sim \normal_{V^\perp}}[\E_{\x\sim \normal}[g(\x)H_{I}(\x)]H_{I}(\x_V+\vec z)]&=
\E_{\x\sim \normal}[g(\x)H_{I}(\x)]H_{I}(\x)
=\E_{\x\sim \normal_V}[\E_{\vec z\sim \normal_{V^\perp}}[g(\x_V+\vec z)H_{I}(\x+\vec z)]]H_{I}(\x)
\\&=\E_{\x\sim \normal_V}[\E_{\vec z\sim \normal_{V^\perp}}[g(\x_V+\vec z)]H_{I}(\x)]H_{I}(\x)\\&=\E_{\x\sim \normal}[\Pi_V g(\x)H_{I}(\x)]H_{I}(\x)\;.
    \end{align*}
    In the case where $H_{I}(\x)$ depends on $V^\perp$, we have that $\E_{\vec z\sim V^\perp}[H_{I}(\x_V+\vec z)]=0$. Therefore, it suffices to prove that $\E_{\x\sim \normal}[\Pi_V g(\x)H_{I}(\x)]H_{I}(\x)=0$. Note that 
    \begin{align*}
        \E_{\x\sim \normal}[\Pi_V g(\x)H_{I}(\x)]H_{I}(\x)&=\E_{\x\sim \normal_V}[\E_{\vec z\sim \normal_{V^\perp}}[\Pi_V g(\x+\vec z)]H_{I}(\x_V+\vec z)]H_{I}(\x)
        \\&=\E_{\x\sim \normal_V}[\Pi_V g(\x)\E_{\vec z\sim \normal_{V^\perp}}[H_{I}(\x_V+\vec z)]]H_{I}(\x)=0\;.
    \end{align*}
\end{proof}This completes the proof of \Cref{lem:properties-of-the-operators-2}.
\end{proof}

\begin{Ualgorithm}
	\centering
	\fbox{\parbox{6in}{
			{\bf Input:}   $\eps>0$, $\delta>0$ and sample and query access to distribution $\D$\\
			{\bf Output:}  An estimation of $\vec M=\E_{\x\sim \D_\x}[ D_\rho y(\x) D_\rho y(\x)^\top]$.
			\begin{enumerate}
				\item $\rho\gets C\eps^2$, $\eta\gets C\eps^2$, for $C>0$ sufficiently small constant. 
				\item Let $S_N$ be the set that contains $N$ samples $\x^{(1)},\ldots,\x^{(N)}$ from the distribution $\D$.
    \item For each $\x\in S_N$, use \Cref{alg:gradient-queries} to get a gradient estimate $\widehat{(D_\rho y)}(\x)$ of $(D_\rho y) (\vec x)$.
                    \item  $\textbf{return } \widehat{\vec M}=\frac{1}{N}\sum_{i=1}^N\widehat{(D_\rho y)}(\x^{(i)})\widehat{(D_\rho y)}(\x^{(i)})^\top$. 
			\end{enumerate}
	}}
 \medskip
	\caption{Estimation of the influence matrix $\vec M$ with Queries } \label{alg:membership-est}
\end{Ualgorithm}
Having access to the gradient, enables us to calculate the influence matrix of the function which captures the sensitivity of the function in different directions.  We formally define the influence 
matrix of a function $g$.
\begin{definition}[Influence Matrices]
Given a differentiable $g \in L^2(\normal)$, we define the influence matrix as
\[
\boldsymbol{\mathrm{Inf}}_g \eqdef \Exn[\nabla g(\x) \nabla g(\x)^\top].
\]
Fix $\rho \in (0, 1)$.  Given $g \in L^2(\normal)$ (not necessarily differentiable), 
we define its $\rho$-smoothed influence matrix as 
\[
\boldsymbol{\mathrm{Inf}}^{\rho}_g \eqdef \Exn[D_\rho g(\x) (D_\rho g(\x)) ^\top] \,.
\]
\end{definition}

\subsection{Influence PCA for Learning in $L_2^2$}

In this section we show that for learning real-valued concepts of bounded variation
in $L_2^2$ we can effectively reduce the dimension of the problem via PCA in
the influence of the smoothed label $T_\rho y$.  
We show that the low-degree polynomial approximation of the 
smoothed label $T_\rho y$ can be projected down to the subspace $V$ via the
Gaussian Marginalization Operator. In other words, we construct an explicit
polynomial approximation of the label $T_\rho$ that depends only on the low-dimensional subspace $V$.  We now state our dimension-reduction result.

    \begin{proposition}\label{prop:main-prop-real-values}
         Fix $\eps,\valb,\gradb,Q > 0$ and let $\psi : \R^d \mapsto \R$ with
    $|\psi(\x)|\leq Q$ and $\|\nabla \psi(\x)\|_2\leq \Psi$. 
    Let $\eta$ be sufficiently small multiple of $\eps^2/(k\valb)$ and $m$ be sufficiently large multiple of $(Q^2\gradb)/\eps^2$. Let $\wh{\vec M} $ so that $\|\boldsymbol{\mathrm{Inf}}^{}_\psi-\wh{\vec M}\|_2\leq \eta/2$ and let $V$ be the subspace spanned by 
    all the eigenvectors of $\wh{\vec M}$ whose corresponding eigenvalues are 
    at least $\eta$.  Then, it holds
    \begin{enumerate}
        \item     \[
    \E_{\x\sim \normal}[(\Pm{m} \Pi_V \psi (\x)-\psi(\x))^2]\leq \inf_{f\in \mathfrak R(\valb,\gradb, k) }\E_{\x\sim \normal}[((\psi(\x)-f(\x))^2]+\eps\;.
    \]
    \item The dimension of $V$ is at most $O(\Psi^2/\eta)$.
    \end{enumerate}

    \end{proposition}

    \begin{proof}[Proof of \Cref{prop:main-prop-real-values}] Fix $f\in \mathfrak R(\valb,\gradb, k)$. By assumption, there exists a subspace $U$ of dimension at most $k$, so that $f$ depends only on $U$, i.e., $f(\x)=f(\proj_U \x)$. Therefore, $\Pi_{U+ V}f(\x)=f(\x)$.
\begin{lemma}[Excess $L_2^2$ Error Decomposition]\label{lem:excess-l2-error-decomposition}
We have
\begin{align*}
\mathcal{E}_2(P_m \Pi_V \psi, f; \psi)
\leq 
Q \underbrace{(\E_{\x \sim \normal}[(f(\x) - P_m f(\x))^2])^{1/2}}_{\text{Polynomial Approximation Error}}
+ 2 \underbrace{\E_{\x \sim \normal}[(\psi(\x) - \Pi_V \psi(\x) ) f(\x)]}_{\text{Correlation Error}}
\,.
\end{align*}

\end{lemma}

\begin{proof}
We have that
\begin{align*}
  \E_{\x\sim \normal}[&(\psi(\x)-\Pm{m}\Pi_V \psi(\x))^2]- \E_{\x\sim \normal}[(\psi(\x)-f(\x))^2]\\&= \E_{\x\sim \normal}[(\Pm{m}\Pi_V \psi(\x))^2-f^2(\x)]+2\E_{\x\sim \normal}[ \psi (\x)(f(\x)-\Pm{m}\Pi_{V} \psi(\x))]
  \\&=\underbrace{\E_{\x\sim \normal}[(\Pm{m}\Pi_V \psi(\x))^2-f^2(\x)]+2\E_{\x\sim \normal}[ \Pi_{V} \psi (\x)(f(\x)-\Pm{m}\Pi_V \psi(\x))]}_{I}+2\E_{\x\sim \normal}[ (\psi (\x)-\Pi_{V} \psi (\x))f(\x)]\;,
\end{align*}
where we used that $\E_{\x\sim \normal}[ (\psi (\x)-\Pi_{V} \psi (\x))\Pm{m}\Pi_{V} \psi(\x)]=0$.
Furthermore, note that $\Exn[\Pi_{V} \psi (\x)\Pm{m}\Pi_V \psi(\x)]=\Exn[(\Pm{m}\Pi_V \psi(\x))^2]$, therefore, we have that
\begin{align*}
    I&=
    \E_{\x\sim \normal}[-(\Pm{m}\Pi_V \psi(\x))^2-f^2(\x)]+2\E_{\x\sim \normal}[ \Pi_{V} \psi (\x)f(\x)]
    \\&\leq \E_{\x\sim \normal}[-(\Pm{m}\Pi_V \psi(\x))^2-(\Pm{m}f(\x))^2]+2\E_{\x\sim \normal}[ \Pi_{V} \psi (\x)f(\x)]
    \\&= -\E_{\x\sim \normal}[(\Pm{m}\Pi_V \psi(\x)-\Pm{m}f(\x))^2]+ 2\E_{\x\sim \normal}[ \Pi_{V} \psi (\x)(f(\x)-\Pm{m}f(\x))]
    \\&\leq 2\E_{\x\sim \normal}[ \Pi_{V} \psi (\x)(f(\x)-\Pm{m}f(\x))]\;.
\end{align*}
Using that $\Exn[(\Pi_{V} \psi (\x))^2]\leq \Exn[( \psi (\x))^2]\leq Q^2$ and \CS inequality we get that 
$
\E_{\x\sim \normal}[ \Pi_{V} \psi (\x)(f(\x)-\Pm{m} f(\x))]\leq Q \E_{\x \sim \normal}[(f(\x) - P_m f(\x))^2]^{1/2}$. This completes the proof of 
\Cref{lem:excess-l2-error-decomposition}.
\end{proof}

\begin{lemma}[Correlation Error Bound] \label{lem:correlation-bound-real}
It holds 

\begin{equation} 
    \E_{\x\sim \normal}[
    (\psi(\x)- \Pi_V \psi(\x)) f(\x)] \leq  O(\eps)\;.
\end{equation}
\end{lemma}
\begin{proof}

Note that $f(\x)$ depends only on the subspace $U$, therefore, $\Pi_{U+V}f(\x)=f(\x)$. Therefore, we have that
\begin{align*}
     \E_{\x\sim \normal}[
    (\psi(\x)- \Pi_V \psi(\x)) f(\x)] &= \E_{\x\sim \normal}[
    (\Pi_{V+U}\psi(\x)- \Pi_{V+U}\Pi_V \psi(\x)) f(\x)] 
    \\&=\E_{\x\sim \normal}[
    (\Pi_{V+U}\psi(\x)-\Pi_V  \Pi_{V+U}\psi(\x)) f(\x)]
    \\&\leq \left(\E_{\x\sim \normal}[
    (\Pi_{V+U}\psi(\x)-\Pi_V  \Pi_{V+U}\psi(\x))^2]\E_{\x\sim \normal}[
   f^2(\x)]\right)^{1/2}\;,
\end{align*}
where in the last equality we used \Cref{lem:properties-of-the-operators-2} and in the last inequality we used the \CS inequality. Note that $\Exn[f^2(\x)]\leq M$. To bound the other term we show that $\E_{\x\sim \normal}[(\Pi_{V+ U} \psi (\x)-\Pi_{V} \Pi_{U+V} \psi (\x))^2]$ is small. For that, we prove a generalization of \Cref{lem:gaussian-smoothing}.
\begin{lemma}[Generalized Gaussian Marginalization  Error]\label{lem:gaussian-smoothing-geometric}
Let $g: \R^d \mapsto \R$ be a function in $L^2(\normal)$ such that $\nabla g \in L^2(\normal)$ and let $V, U$ be subspaces of $\R^d$.
It holds 
\[
\E_{\x \sim \normal}[(\Pi_{V+U} g(\x) - \Pi_{V} \Pi_{V+U} g(\x))^2] \leq 
\dim( V^\perp\cap U) \max_{\vec v\in V^\perp\cap U,\|\vec v\|_2=1}\E_{\x \sim \normal}[(\nabla g(\x) \cdot \vec v)^2] \,.
\]
\end{lemma}
\begin{proof}
Assume that $\dim(V^\perp\cap U)=k\leq d$. Using the rotation invariance of the Gaussian distribution, without loss of generality, we may 
assume that $\vec e_1,\ldots, \vec e_k$ is a basis of $V^\perp\cap U$. Note that it holds $ \Pi_{V} \Pi_{U+V}  g(\x)=\Pi_{V+U^\perp} \Pi_{V+U} g(\x)=\Pi_{V+U}\Pi_{V+U^\perp} g(\x)$.
We have 
\begin{align*}
   \E_{\x \sim \normal}[(\Pi_{U+V} g(\x) - \Pi_{V} \Pi_{U+V} g(\x))^2] &=
\E_{\x \sim \normal}[(\Pi_{U+V} g(\x) - \Pi_{V+U}\Pi_{V+U^\perp}g(\x))^2]
\\&\leq \E_{\x \sim \normal}[( g(\x) - \Pi_{V+U^\perp}g(\x))^2]
\\&=
\E_{\x_{k+1},\ldots, \x_d \sim \normal }[ \var_{\x_1,\ldots \x_k \sim \normal} [g(\x_1, \ldots, \x_d)] ]
\\&\leq \frac 12 \E_{\x_{k+1},\ldots, \x_d \sim \normal} \left[\sum_{i=1}^k \E_{\x_{1},\ldots, \x_k \sim \normal} \left[\var_{\x_i \sim \normal} [g(\x_1, \ldots,\x_i,\ldots \x_d)] \right] \right]\,, 
\end{align*}
where in the inequality, we used Efron-Stein's inequality.
Using \Cref{lem:poincare}, for each $i\in [k]$ we have $\var_{\x_i \sim \normal} [g(\x_1, \ldots,\x_i,\ldots \x_d)]\leq \E_{\x_i\sim \normal}[(\nabla g(\x_1,\ldots,\x_d)\cdot \vec e_i)^2]$, and therefore we have
\begin{align*}
    \E_{\x \sim \normal}[(g(\x) - r(\x))^2] &\leq \sum_{i=1}^k
\E_{\x_1,\ldots, \x_d \sim \normal} [  (\nabla g(\x) \cdot \vec e_i)^2]\leq 
k\max_{\vec v\in H^\perp,\|\vec v\|_2=1}\E_{\x \sim \normal}[(\nabla g(\x) \cdot \vec v)^2]\;.
\end{align*}
This completes the proof of \Cref{lem:gaussian-smoothing-geometric}.
\end{proof}

From \Cref{lem:gaussian-smoothing-geometric}, we have that 
    \[
    \E_{\x\sim \normal}[( \Pi_{V+ U} \psi (\x)- \Pi_{V} \Pi_{U+V} \psi(\x))^2 ]\leq \dim(U \cap V^\perp) \max_{\vec v\in U \cap V^\perp,\|\vec v\|_2=1} \E_{\x\sim \normal}[((\nabla\psi (\x))\cdot \vec v)^2]\;.
    \]
    Furthermore note that $\max_{\vec v\in U \cap V^\perp,\|\vec v\|_2=1} \E_{\x\sim \normal}[((\nabla\psi (\x))\cdot \vec v)^2]\leq \eta/2+\max_{\vec v\in U \cap V^\perp,\|\vec v\|_2=1} \vec v^\top \wh{\vec M}\vec v\leq 2\eta$ because the subspace $U \cap V^\perp$ contains vectors with influence at most $\eta$. Note that $\dim(U \cap V^\perp)\leq \dim(U)\leq k$ and noting $\eta =O( \eps^2/(Mk))$ completes the proof of \Cref{lem:correlation-bound-real}.
\end{proof}

Combining \Cref{lem:correlation-bound-real,lem:excess-l2-error-decomposition} and using that  $\Exn[(f(\x)-\Pm{m} f(\x))^2]\leq \gradb/m$ from \Cref{lem:low-degree-approximation-smooth-functions}, we get that 
\[
    \E_{\x\sim \normal}[(\Pm{m} \Pi_V \psi (\x)-\psi(\x))^2]\leq \inf_{f\in \mathfrak R(\valb,\gradb, k) }\E_{\x\sim \normal}[((\psi(\x)-f(\x))^2]+\eps\;.
    \]

    To show that the subspace $V$ has small dimension, we show the following lemma
\begin{lemma}\label{lem:dim-of-subspace}
    Fix $\eta>0,\rho\in(0,1)$. 
    Let $\psi$ be a function from $\R^d$ to $\R$ such that $\|\nabla \psi(\x)\|_2\leq \Psi$ and let $V$ be the subspace spanned by all the eigenvectors of $ \boldsymbol{\mathrm{Inf}}_g$ with eigenvalue at least $\eta$. Then the dimension of the subspace $V$ is $\dim(V)=O(\Psi^2/ \eta)$.
\end{lemma}

\begin{proof}
    Let $m=\dim(V)$. $V$ is spanned by the eigenvectors of $ \boldsymbol{\mathrm{Inf}}_g=\Exn[\nabla \psi(\x) (\nabla \psi(\x)) ^\top]$ with eigenvalue at least $\eta$, hence,
    \[
    m\eta \leq \tr(\boldsymbol{\mathrm{Inf}}^{}_g)=\Exn[\tr(\nabla \psi(\x) (\nabla \psi(\x)) ^\top)]=\Exn[\|\nabla \psi(\x) \|_2^2]\;.
    \]
    From the assumption, we have that $\Exn[\|\nabla \psi(\x) \|_2^2]=O(\Psi^2)$.
Therefore, we have that  $m\leq O(\Psi^2/\eta)$.
\end{proof}

   An application of the lemma above (\Cref{lem:dim-of-subspace}) gives, which gives that the subspace it at most $O(\Psi^2/\eta)$.
This completes the proof of \Cref{prop:main-prop-real-values}
    \end{proof}

\subsection{Proof of \Cref{thm:non-proper-real-valued}}
We will use the following fact about the $L_2$ polynomial regression.
    \begin{fact}[see, e.g., Theorem D.7 \cite{DKKTZ21}]\label{fct:l2-regr} Let $D$ be a distribution on $\R^d\times \R$ such that the $\x$-marginal of $D$ is standard $d$-dimensional normal and the labels $y$ are bounded by $M$. The $L_2$-regression algorithm draws $N=\poly((dm)^{m^2},1/\eps,M,\log(1/\delta))$ samples from $D$, runs in time $\poly(N,d)$, and outputs a polynomial $p:\R^d\mapsto\R$ such that with probability at least $1-\delta$ it holds
\[
\E_{(\x,y)\sim D}[(p(\x)-y)^2]\leq \min_{p\in \mathcal P_m}\E_{(\x,y)\sim D}[(p(\x)- y)^2]+\eps\;,
\]
where $\mathcal P_m$ is the class of polynomials with degree at most $m$.
\end{fact}
We first show that we can truncate the labels with $|y(\x)|\geq M'=\valb^ {1/2}/\eps^{1/2}$ without increasing the error by a lot. From Markov's inequality, we have that
\[
\pr[|f(\x)|\geq M']\leq \E_{\x\sim \normal}[f^2(\x)]/(M')^2\leq\sqrt{\E_{\x\sim \normal}[f^4(\x)]}/(M')^2 \leq \eps\;.
\]
Let $\mathrm{trunc}(y(\x))=\sign(y(\x))\min(|y(\x)|,M')$. We have that
\begin{align*}
    \E_{\x\sim \normal}[(f(\x)-\mathrm{trunc}(y(\x)))^2]&=     \E_{\x\sim \normal}[(f(\x)-\mathrm{trunc}(y(\x)))^2(\1\{|f(\x)| \leq M'\}+\1\{|f(\x)| > M'\})]
    \\&\leq \E_{\x\sim \normal}[(f(\x)-y(\x))^2]+\E_{\x\sim \normal}[(f(\x)-\mathrm{trunc}(y(\x)))^2\1\{|f(\x)| > M'\}]
    \\&\leq \E_{\x\sim \normal}[(f(\x)-y(\x))^2]+2(\sqrt{\E_{\x\sim \normal}[f^4(\x)]}+(M')^2)\sqrt{\pr[|f(\x)|\geq M']}
    \\&\leq \E_{\x\sim \normal}[(f(\x)-y(\x))^2]+\eps\;.
\end{align*}
For the rest of the proof, we assume that $y(\x)$ is truncated at $M'$. Let $\psi(\x)=T_\rho y$ for $\rho=\poly(\eps/(\valb \gradb))$. Note that $\|\nabla \psi(\x)\|_2\leq M'$. From \Cref{lem:estimation}, with
$N=\poly(d/\eps)\log(1/\delta)$ queries, we get that
with probability $1-\delta/2$ a matrix $\vec M$, so that $\|\vec M-\boldsymbol{\mathrm{Inf}}^{}_\psi\|_F\leq \eps$.  Applying \Cref{prop:main-prop-real-values} to
the matrix $\vec M$, we get that in the subspace $V$ spanned by the eigenvectors
of the matrix $\vec M$ with
eigenvalues larger than $\eta=\poly(\eps/\valb  k))$ with dimension at most $O(\poly(M',1/\eta,1/\eps))$, there exists a polynomial $p:V\mapsto \R$ of degree $m=\poly(M_2/\eps)$ with $\Exn[p^2(\x)]\leq \E[\psi^2(\x)]\leq (M')^2$, so that 
$$\E_{\x\sim \normal}[(p(\x)-\psi(\x))^2]\leq \E_{\x\sim \normal}[(f(\x)-\psi(\x))^2]+\eps/2\;.$$
From \Cref{lem:smoothing-real}, we get that for the same polynomial and using that $\Exn[\|\nabla p(\x)\|_2\leq m \Exn[p^2(\x)]$, it also holds that 
$$\E_{\x\sim \normal}[(p(\x)-y(\x))^2]\leq \E_{\x\sim \normal}[(f(\x)-y(\x))^2]+\eps/2\;.$$
Let $\vec P:\R^d\mapsto V$ be the projection matrix to the subspace $V$. Let $(\vec P\x,y)\sim D'$, where $(\x,y)\sim D$. We use the $L_2$-regression algorithm on $D'$ and from \Cref{fct:l2-regr}, using $\poly((k\valb/\eps)^{\gradb^2/\eps^4},1/\eps,\log(1/\delta))$ samples from $D'$, we get a polynomial $p':V\mapsto \R$ so that with probability at least $1-\delta$, it holds
    \[
\E_{\x\sim \normal}[(p'(\vec P\x)-y(\x))^2]\leq \E_{\x\sim \normal}[(p(\x)-y(\x))^2]+\eps/2\leq \E_{\x\sim \normal}[(f(\x)-y(\x))^2]+\eps\;.
\]

This completes the proof of \Cref{thm:non-proper-real-valued}.

\subsection{Applications of \Cref{thm:non-proper-real-valued}}
\label{ssec:ReLU}

In this section, we apply \Cref{thm:non-proper-real-valued} for several real-valued activations. We start by applying our theorem for the class of ReLU activations.

\begin{theorem}[Improper Learner for ReLUs Activations]\label{thm:non-proper-real-RelLU}
Fix $M\in \R_+$. Let $\mathcal C$ be the concept class containing all the ReLU activations with normal vectors bounded in $\ell_2$ norm by $M$.
Let $D$ be a distribution on $\R^d\times \R$ such that the 
$\x$-marginal of $D$ is the standard $d$-dimensional normal. 
There exists an algorithm that makes $N_q = \poly(d M/\eps)$ queries, draws $N_s = \poly(d/\epsilon) + 2^{\poly(M/\eps)} \log(1/\delta)$ samples from $D$, runs in time $\poly(N_s,N_q,d)$ and outputs a polynomial $p:\R^d\mapsto\R$ so that with probability at least 
$1-\delta$ it holds
\[
\E_{(\x,y)\sim D}[(p(\x)- y)^2]\leq \inf_{f\in \mathcal C}\E_{(\x,y)\sim D}[(f(\x)- y)^2]+\eps\;.
\]
\end{theorem}
\begin{proof}
    To prove the above theorem it suffices to show that $\mathcal C\subseteq \mathfrak{R}(\sqrt{3}M^2,M^2,1)$. Note that $\Exn[(\mathrm{ReLU}(\vec w\cdot \x))^4]\leq \Exn[(\vec w\cdot \x)^4]\leq 3 M^4$. Furthermore, we bound the derivative of the activation. 
 We have that
    \begin{align*}
        \Exn[\|\nabla_\x \mathrm{ReLU}(\vec w\cdot \x)\|_2^2]=\Exn[\|\1\{\vec w\cdot \x\geq 0\} \vec w\|_2^2]\leq M^2\;.
    \end{align*}
    Therefore, it follows that $\mathcal C\subseteq \mathfrak{R}(\sqrt{3}M^2,M^2,1)$. An application of \Cref{thm:non-proper-real-valued} gives the result.
\end{proof}

We next consider learning Single-index models (SIMs) 
with an  unknown Lipschitz link
function $g:\R \mapsto \R$, i.e., $f(\x) = g(\vec w\cdot \x)$.
\begin{definition}
    We define the class of $L$-Lipschitz SIMs on $\R^d$ denoted $\mathrm{ SIM}(L,M)$ as follows. For each $f\in \mathrm{ SIM}(L,M)$, $f(\x) = g(\vec w\cdot \x)$, for $L$-Lipschitz $g:\R\mapsto\R$ and $\|\vec w\|_2\leq M$.
\end{definition}

\begin{theorem}[Improper Learner for SIMs]\label{thm:non-proper-real-sims}
Fix $L,M\in \R_+$.
Let $D$ be a distribution on $\R^d\times \R$ such that the 
$\x$-marginal of $D$ is the standard $d$-dimensional normal. 
There exists an algorithm that makes $N_q = \poly(d L/\eps)$ queries, draws $N_s = \poly(d/\epsilon) + 2^{\poly(L M/\eps)} \log(1/\delta)$ samples from $D$, runs in time $\poly(N_s,N_q,d)$ and outputs a polynomial $p:\R^d\mapsto\R$ so that with probability at least 
$1-\delta$ it holds
\[
\E_{(\x,y)\sim D}[(p(\x)- y)^2]\leq \inf_{f\in \mathrm{SIM}(L,M)}\E_{(\x,y)\sim D}[(f(\x)- y)^2]+\eps\;.
\]
\end{theorem}
\begin{proof}
    Note that for any $f\in \mathrm{SIM}(L)$ by definition if holds that $\|\nabla f(\x)\|_2\leq L$ and also that $\E[f^4(\x)]\leq L^4\E[(\vec w\cdot \x)^4]\lesssim M^4L^4$. Therefore, we have that $f\in \mathrm{ SIM}(L,M)\subseteq \mathfrak{R}(M^2L^2,L,1)$. An application of \Cref{thm:non-proper-real-valued} gives the result.
\end{proof}

We define the class of linear combinations of ReLU networks.
\begin{definition}[ReLU Networks]
We define the class $\mathfrak Re(M, k)$ of ReLU networks as follows. For each $f\in \mathfrak Re(M, k)$, $f(\x)=\vec W_2\mathrm{ReLU}(\vec W_1\x)$, for matrices $\vec W_1\in\R^{k\times d}, \vec W_2\in\{\pm 1\}^{k\times 1}$, with $\|\vec W_1\|_{op}\leq M$.
\end{definition}

We  give our result for learning linear combinations of ReLUs, i.e.,
real-valued functions of the form $f(\x) = \sum_{i=1}^k a_i \mathrm{ReLU}(\vec w^{(i)} \cdot \x ) $, where $a_i \in \R$.

\begin{theorem}[Improper Learner for Linear Combinations of ReLUs]\label{thm:non-proper-real-linear-combinations-of-RelLU}
Fix $k \in \mathbb N$ and $M\in \R_+$.
Let $D$ be a distribution on $\R^d\times \R$ such that the 
$\x$-marginal of $D$ is the standard $d$-dimensional normal. 
There exists an algorithm that makes $N_q = \poly(d M/\eps)$ queries, draws $N_s = \poly(d/\epsilon) + (k M/\eps)^{\poly(k M/\eps)} \log(1/\delta)$ samples from $D$, runs in time $\poly(N_s,N_q,d)$ and outputs a polynomial $p:\R^d\mapsto\R$ so that with probability at least 
$1-\delta$ it holds
\[
\E_{(\x,y)\sim D}[(p(\x)- y)^2]\leq \inf_{f\in \mathfrak Re(M, k)}\E_{(\x,y)\sim D}[(f(\x)- y)^2]+\eps\;.
\]
\end{theorem}
\begin{proof}
     We show that $\mathfrak Re(M, k)\subseteq \mathfrak R(\valb',\gradb',k)$ for appropriate parameters $\valb',\gradb'$. We show the following 
     \begin{lemma}\label{lem:sum-of-relu}
    Let $f(\x)=\sum_{i=1}^ka_i\mathrm{ReLU}(\vec w^{(i)}\cdot \x)$ where $a_i\in\{\pm 1\}\in \R$ and  $\vec w^{(i)}\in \R^d$ with $\|\vec w^{(i)}\|_2\leq M$ for all $i\in[k]$. Then, we have that $f\in \mathfrak R(kM^2,kM^2,k)$.
\end{lemma}
\begin{proof}
We have that $\Exn[f^4(\x)]\leq \Exn[(\sum_{i=1}^k\mathrm{ReLU}(\vec w^{(i)}\cdot \x))^4]$. From the \CS inequality we have that $(\sum_{i=1}^k z_i)^2\leq k\sum_{i=1}^k z_i^2$. Therefore, applying this inequality twice, we get that
$\Exn[f^4(\x)]\leq k^3\sum_{i=1}^k\Exn[(\mathrm{ReLU}(\vec w^{(i)}\cdot \x))^4]\leq O(k^3M^4)$.
 We then bound the derivative of $f$. We have that
    \begin{align*}
        \Exn[\|\nabla_\x f(\x)\|_2^2]=k\sum_{i=1}^k\Exn[\|\1\{\vec w^{(i)}\cdot \x\geq 0\} \vec w^{(i)}\|_2^2]\leq O(kM^2)\;.
    \end{align*}
\end{proof}
Then the proof follows from \Cref{lem:sum-of-relu} along with \Cref{thm:non-proper-real-valued}.
\end{proof}

We now give an improved result for learning sums of ReLUs, i.e.,
real-valued functions of the form $f(\x) = \sum_{i=1}^k  \mathrm{ReLU}(\vec w^{(i)} \cdot \x ) $.
We first define the class of sum of ReLUs.
\begin{definition}[Sums of ReLU Networks]
We define the class $\mathfrak Re_+(M, k)$ of ReLU networks as follows. For each $f\in \mathfrak Re_+(M, k)$, $f(\x)=\mathrm{ReLU}(\vec W\x)$, for matrices $\vec W\in\R^{k\times d}$, with $\E[f^2(\x)]\leq M$.
\end{definition}
\begin{theorem}[Improper Learner for Sums of ReLUs]\label{thm:non-proper-real-sum-of-RelLU}
Fix $k \in \mathbb N$ and $M\in \R_+$.
Let $D$ be a distribution on $\R^d\times \R_+$ such that the 
$\x$-marginal of $D$ is the standard $d$-dimensional normal. 
There exists an algorithm that makes $N_q = \poly(d M/\eps)$ queries, draws $N_s = \poly(d/\epsilon) + (k M/\eps)^{\poly(M/\eps)} \log(1/\delta)$ samples from $D$, runs in time $\poly(N_s,N_q,d)$ and outputs a polynomial $p:\R^d\mapsto\R$ so that with probability at least 
$1-\delta$ it holds
\[
\E_{(\x,y)\sim D}[(p(\x)- y)^2]\leq \inf_{f\in \mathfrak Re_+(M, k)}\E_{(\x,y)\sim D}[(f(\x)- y)^2]+\eps\;.
\]
\end{theorem}
\begin{proof}
Note that from \Cref{lem:sum-of-relu} we have that for $f(\x)=\sum_{i=1}^k\mathrm{ReLU}(\vec w^{(i)}\cdot \x)$,  where $\vec w^{(i)}\in \R^d$ with $\Exn[f^2(\x)]\leq \Exn[(\sum_{i=1}^k\mathrm{ReLU}(\vec w^{(i)}\cdot \x))^4]\leq M^2$. We show that  $f\in \mathfrak R(kM^2,M^2,k)$.  Similar to \Cref{thm:non-proper-real-linear-combinations-of-RelLU}, we have that $\Exn[f^4(\x)]\leq O(kM^2)$.
The proof differs from \Cref{thm:non-proper-real-linear-combinations-of-RelLU} on the fact that we can bound the gradient of $f$ by the $L_2^2$ norm of $f$ yielding a bound independent of $k$ in the exponent. We show that 
$ \Exn[\|\nabla_\x f(\x)\|_2^2]\leq O(M)$.
We have that
\begin{align*}
        \Exn[\|\nabla_\x f(\x)\|_2^2]&=\Exn[\|\sum_{i=1}^k\1\{\vec w^{(i)}\cdot \x\geq 0\} \vec w^{(i)}\|_2^2]
        \\&= \Exn[\sum_{i,j=1}^k\1\{\vec w^{(i)}\cdot \x\geq 0\} \1\{\vec w^{(j)}\cdot \x\geq 0\}\vec w^{(i)}\cdot \vec w^{(j)}] 
        \\&\leq 2\Exn[\sum_{i,j=1}^k\mathrm{ReLU}(\vec w^{(i)}\cdot\x)\mathrm{ReLU}(\vec w^{(j)}\cdot\x) ]
        \\&\leq 2\Exn[f^2(\x)]\leq O(M)\;.
    \end{align*}

        Therefore, we have that $\mathcal C\subseteq \mathfrak{R}(M,2M,k)$. An application of \Cref{thm:non-proper-real-valued} gives the result.
\end{proof}

Next we show our result for a general ReLU network. We first define the clas of Deep ReLU networks.
\begin{definition}[Deep ReLU Networks]
\label{def:deep-nets}
We define the class $\mathfrak D(M, L, k, S)$ of depth-$(L+1)$ ReLU networks as follows. For each $f\in \mathfrak D(M, L, k)$, $f(\x)=\vec W_L\mathrm{ReLU}(\vec W_{L-1}\cdots \mathrm{ReLU}(\vec W_1\x))$, for matrices $\vec W_1\in\R^{k\times d},\ldots, \vec W_L\in\R^{k_L\times 1}$, with $\|\vec W_i\|_{op}\leq M$ and $k_i\leq S$.
\end{definition}

\noindent 
We show the following theorem.

\begin{theorem}[Agnostic Learner for Deep ReLU Networks]\label{thm:non-proper-real-deep-of-RelLU}
Fix $k,S,L \in \mathbb N$ and $M\in \R_+$.
Let $D$ be a distribution on $\R^d\times \R^+$ such that the 
$\x$-marginal of $D$ is the standard $d$-dimensional normal. 
There exists an algorithm that makes $N_q = \poly(d M/\eps)$ queries, draws $N_s = \poly(d/\eps) + 2^{\poly(k S M/\eps)} \log(1/\delta)$ samples from $D$, runs in time $\poly(N_s,N_q,d)$ and outputs a polynomial $p:\R^d\mapsto\R$ so that with probability at least 
$1-\delta$ it holds
\[
\E_{(\x,y)\sim D}[(p(\x)- y)^2]\leq \inf_{f\in \mathfrak D(M, L, k, S)}\E_{(\x,y)\sim D}[(f(\x)- y)^2]+\eps\;.
\]
\end{theorem}
\begin{proof}
    We show that $\mathfrak D(M, L, k, S)\subseteq \mathfrak R(\valb',\gradb',k)$ for appropriate parameters $\valb',\gradb'$. We first calculate for each $f\in  \mathfrak D(M, L, k, S)$, the $\Exn[\|\nabla f(\x)\|_2^2]$. Denote as $D_i(\x)=\vec W_i\mathrm{ReLU}(\vec W_{i-1}\cdots \mathrm{ReLU}(\vec W_1\x))$ the sub-network of $f(\x)$. From the product rule, we have that
    \[
    \nabla f(\x)=\vec W_{L}\mathrm{diag}(\1\{D_{L-1}\geq 0\})\vec W_{L-1}\cdots \mathrm{diag}(\1\{\vec W_1\x\geq 0\})\vec W_1\;.
    \]
    Therefore, we have that $\|\nabla f(\x)\|_2\leq \prod_{i=1}^L\|\vec W_{i}\|_{op}\sqrt{k_i}\leq (MS)^L$. Using the Poincare inequality, we can show that $\Exn[f^2(\x)]\leq k\Exn[\|\nabla f(\x)\|_2^2]\leq k(MS)^L$. Therefore, $\mathfrak D(M, L, k, S)\subseteq \mathfrak R((kMS)^{O(L)},(kMS)^{O(L)},k)$. Then the proof follows from  \Cref{thm:non-proper-real-valued}.
\end{proof}

 \section{Agnostically Learning Boolean Multi-index Models}\label{sec:geometric-concepts}

In this section, we present our results for Boolean multi-index models 
of bounded surface area. For convenience, we restate the
class of concepts that we consider.
\begin{definition}[Bounded Surface Area, Low-Dimensional Boolean Concepts]
We define the class $\mathfrak B(\Gamma, k)$ of Boolean concepts with the following
properties:
\begin{enumerate}
\item For every $f \in \mathfrak B(\Gamma, k)$, it holds $\Gamma(f) \leq \Gamma$.
\item For every $f \in \mathfrak B(\Gamma, k)$, there exists a subspace $U$ of $\R^d$ of dimension at most $k$ such that $f$ depends only on $U$, i.e., 
for every $\x \in \R^d$, $f(\x) = f(\proj_U \x)$.
\item $\mathfrak B(\Gamma, k)$ is closed under translations, i.e.,
if $f(\x) \in \mathfrak B(\Gamma, k)$ then $f(\x + \vec t) \in \mathfrak B(\Gamma,k)$ for all $\vec t \in \R^d$.
\end{enumerate}
\end{definition}

We remark that $\mathfrak{B}(\Gamma, k)$ is a general, \emph{non-parametric class}.  For example $\mathfrak{B}(\Omega(k), k)$ contains LTFs,  intersections of $k$ LTFs, and Polynomial Threhsold Functions (PTFs)  of degree at most $k$ (that depend on a $k$-dimensional subspace). Our learner is able to learn a hypothesis of low excess error when compared against all concepts
of $\mathfrak{B}(\Gamma, k)$ with roughly $\poly(d/\epsilon) + k^{\poly(\Gamma/\epsilon)}$ runtime.

\begin{theorem}\label{thm:non-proper-geometric}
Fix $k \in \mathbb N$ and $M \in \R^+$.
Let $D$ be a distribution on $\R^d\times \{\pm 1\}$ such that the 
$\x$-marginal of $D$ is standard $d$-dimensional normal. There exists an algorithm that makes $N_q=\poly(d/\eps)$ queries and draws $N_s=\poly(d/\eps)+\poly((k \Gamma/\eps)^{\Gamma^2/\eps^4}, 1/\eps, \log(1/\delta))$ samples from $D$ and runs in time $\poly(N_s,N_q,d)$ and outputs a polynomial $p:\R^d\mapsto\R$ so that with probability at least 
$1-\delta$ it holds
\[
\pr_{(\x,y)\sim D}[\sign(p(\x))\neq y]\leq \inf_{f\in \mathfrak B(\Gamma,k) }\pr_{(\x,y)\sim D}[f(\x)\neq y]+\eps\;.
\]
\end{theorem}

\subsection{Influence PCA for Learning in $L_1$-norm}

In this section we show our main dimension-reduction tool for 
the concepts of bounded surface  of \Cref{def:bounded-surface-area-concepts}.   Our dimension-reduction tool establishes that: 
given (an approximation of) the influence matrix  of the smooth function $T_\rho y $, we can perform PCA and 
learn a low-dimensional subspace $V$ so that 
a bounded surface area concept that depends only on $V$ 
can achieve $\eps$ excess error with respect to $T_\rho y$ in $L_1$-norm.
We now state our result.
\begin{proposition}[Dimension Reduction]\label{prop:dimension-reduction-geometric-concepts}
    Fix $\eps > 0,k\in \N$ and let $\psi : \R^d \mapsto [-1,1]$
    be a differentiable function with $\|\nabla \psi(\x)\|_2 \leq \Psi$ for all $\x \in \R^d$.
    Let $\eta$ be sufficiently small multiply of $\eps^2/k$ let $\wh{\vec M} \in \R^{d \times d}$ be such that 
    $\| \wh {\vec M} -  \boldsymbol{\mathrm{Inf}}_{\psi} \|_2 \leq \eta/2$. Let $V$ be the subspace spanned by 
    all the eigenvectors of $\wh{\vec M}$ whose corresponding eigenvalues are 
    at least $\eta$.  The following hold true:
    \begin{enumerate}
    \item 
    There exists $g$ with $\Gamma(g)\leq \Gamma$ so that $g(\x)=g(\proj_V \x)$ such that 
    \[
    \E_{\x \sim \normal}[|g(\x)- \psi(\x)|] 
    \leq \inf_{f\in \mathfrak B(\Gamma, k)} \E_{\x \sim \normal}[|f( \x) - \psi(\x)|]
    + \eps \,.
    \]
    \item The dimension of $V$ is at most $O(\Psi^2/ \eta)$.
    \end{enumerate}
    \end{proposition}

Before we proceed to the proof of \Cref{prop:dimension-reduction-geometric-concepts} we give some intuition behind the choice of
the Gaussian Marginalization Operator defined above.
We first give the following simple lemma showing that in order to 
show that a concept class $C_2$ (think of this as the class of
concepts that depend on the subspace $V$ of \Cref{prop:dimension-reduction-geometric-concepts}) has not much worse approximation
error (to some label $y$) than some other class $C_1$ (think of this 
as the original concept class $\mathfrak{B}(\Gamma,k)$) as long as
for every concept $f$ of $C_1$, we can construct a distribution over 
concepts of $C_2$ that (on expectation) achieves at most $\eps$ worse correlation with the label $y$ than the original concept $f$.
Its proof relies on the simple fact that for $t \in [-1,1]$ and 
$s \in \{\pm 1\}$, it holds that $|t-s| = 1 - t s$.
\begin{lemma}[Correlating Convex Combinations]
\label{lem:correlating-convex-combinations}
Fix a function $y:\R^d \mapsto [-1,1]$ and $\epsilon > 0$.
Let $C_1, C_2$ be classes of Boolean concepts on $\R^d$.
Assume that for every $f \in C_1$ there exists a distribution
$Q$ over hypotheses of the class $C_2$ such that 
\[\E_{\x \sim \normal}[f(\x) y(\x)] - 
\E_{\x \sim \normal}\Big[\E_{g \sim Q}[g(\x) ~ y(\x)] \Big] \leq \epsilon
\,.
\]
Then 
\(
\inf_{g \in C_2}\E_{\x \sim \normal}[|g(\x) - y(\x)|]
-
\inf_{f \in C_1}\E_{\x \sim \normal}[|f(\x)  - y(\x)|]
\leq \epsilon \).
\end{lemma}
\begin{proof}
Fix $f\in C_1$. By assumption, we have that there exists a distribution $Q_f$ over $C_2$ so that 
\[\E_{\x \sim \normal}[f(\x) y(\x)] - 
\E_{\x \sim \normal}\Big[\E_{g \sim Q_f}[g(\x) ~ y(\x)] \Big] \leq \epsilon
\,.
\]
Note that $g\in\{\pm 1\}$ and $|y(\x)|\leq 1$, therefore the expectation is bounded and hence from Fubini's theorem, we have that
\[\E_{g \sim Q_f}\big[\E_{\x \sim \normal}[f(\x) y(\x)] - 
\E_{\x \sim \normal}[g(\x) ~ y(\x)] \Big] \leq \epsilon
\,.
\]
That means that there exists a $r_f\in C_2$ so that
\[
\E_{\x \sim \normal}[f(\x) y(\x)] - 
\E_{\x \sim \normal}[r_f(\x) ~ y(\x)] \leq \epsilon
\,.
\]
Because $f(\x),r_f(\x)$ are Boolean functions we have that
$\E_{\x \sim \normal}[f(\x) y(\x)]=1-\E_{\x \sim \normal}[|f(\x) y(\x)]|$ and $\E_{\x \sim \normal}[r_f(\x) y(\x)]=1-\E_{\x \sim \normal}[|r_f(\x) y(\x)]|$. Therefore, we have
\[
\E_{\x \sim \normal}[|r_f(\x) - y(\x)|]
-
\E_{\x \sim \normal}[|f(\x)  - y(\x)|]
\leq \epsilon
\,.
\]
Furthermore, because $r_f\in C_2$, we have that $\E_{\x \sim \normal}[|r_f(\x) - y(\x)|]\geq \inf_{g \in C_2}\E_{\x \sim \normal}[|g(\x) - y(\x)|]$. The proof is completed by taking the supremium over all the $f\in C_1$. This completes the proof of \Cref{lem:correlating-convex-combinations}.
\end{proof}

\subsubsection{Proof of \Cref{prop:dimension-reduction-geometric-concepts}}
We define that set of hypotheses $\mathfrak B_V(\Gamma, k)=\{f\in \mathfrak B(\Gamma, k): f(\proj_V (\x))=f(\x)\}$. We are going to show that
\begin{equation}\label{eq:contradiction-gen}
\inf_{g \in \mathfrak B_V(\Gamma, k)}\E_{\x \sim \normal}[|g(\x) - \psi(\x)|]
\leq
\inf_{f \in \mathfrak B(\Gamma, k)}\E_{\x \sim \normal}[|f(\x)  - \psi(\x)|]
+ \epsilon\;. 
\end{equation}
To prove \Cref{eq:contradiction-gen}, by \Cref{lem:correlating-convex-combinations} it suffices to construct, for each $f\in \mathfrak B(\Gamma, k)$, 
a distribution $Q$ over the set $\mathfrak B_V(\Gamma, k)$ and show that 
$\E_{\x \sim \normal}[f(\x) \psi(\x)] - 
\E_{\x \sim \normal}\Big[\E_{g \sim Q}[g(\x) ~ \psi(\x)] \Big] \leq \epsilon
$.
To this end, we first show that 
\begin{equation}\label{eq:contradiction-gen32}
\E_{\x\sim \normal}[(f(\x)-\Pi_V f(\x))\psi(\x)]\leq \eps\;.
\end{equation}
Note that $\Pi_V f(\x)$ is a distribution over $\mathfrak B_V(\Gamma, k)$. To see that note that $\Pi_V f(\x)=\E_{\vec z\sim \normal_{V^\perp}}[f(\x_V+\vec z)]$ and note that for each $\vec z\in \R^d$, we have that $f(\x_V+\vec z)\in\mathfrak B_V(\Gamma, k)$. 
To prove \Cref{eq:contradiction-gen}, we prove the following lemma: 
\begin{lemma}[Correlation via Gaussian Marginalization]\label{lem:correlation-to-CS}
Let $y: \R^d \mapsto \R$ be some function in $L^2(\normal)$.
Fix some concept $f \in \mathfrak{B}(\Gamma, k)$ and denote by
$U$ the subspace of $\R^d$ that $f$ depends on (i.e., 
$f(\x) = f(\proj_U(\x))$.  Moreover, let $V$ be some other
subspace of $\R^d$.
Then it holds 
\[
\E_{\x \sim \normal}[(f(\x) - \Pi_V f (\x)) y(\x)]
\leq 2 
\sqrt{\E_{\x \sim \normal}[(\Pi_{U+ V} y(\x) - \Pi_{V}\Pi_{U+ V}  y(\x))^2] }
\,.
\]
\end{lemma}
\begin{proof}
Note that by our assumption $f(\x)=f(\x_{U})=\Pi_{U}f(\x)$, therefore, $\Pi_{V}\Pi_{U}f(\x)=\Pi_{V}f(\x)$. Observe that $\E_{\x\sim \normal}[(f(\x)-\Pi_V f(\x))\Pi_{V}\Pi_{U+ V}y(\x)]=\E_{\x\sim \normal}[(\Pi_{V}\Pi_{U}f(\x)-\Pi_V f(\x))\Pi_{V}\Pi_{U+ V}y(\x)]=0$, which gives that
\begin{align*}
    \E_{\x\sim \normal}[&(f(\x)-\Pi_Vf(\x))y(\x)]=\E_{\x\sim \normal}[(f(\x)-\Pi_V f(\x))(\Pi_{U+ V}y(\x)-\Pi_{V}\Pi_{U+ V} y(\x))]\;.
\end{align*}
Using that $f$ is a Boolean function, we have that $|\Pi_V f(\x)|\leq 1$, hence, $|f(\x)-\Pi_V f(\x)|\leq 2$, which gives
\begin{align*}
    \E_{\x\sim \normal}[&(f(\x)-\Pi_Vf(\x))y(\x)]\leq 2\E_{\x\sim \normal}[|\Pi_{U+ V} y(\x)-\Pi_{V}\Pi_{U+ V} y(\x)|]\leq 2\sqrt{\E_{\x \sim \normal}[(\Pi_{U+ V}y(\x) - \Pi_{V}\Pi_{U+ V}y(\x))^2] }\;,
\end{align*}
where we used \CS inequality. This completes the proof of \Cref{lem:correlation-to-CS}. 
\end{proof}

It remains to bound the term  $\E_{\x\sim \normal}[(\Pi_{U+ V}\psi(\x)-\Pi_{V}\Pi_{U+ V} \psi(\x))^2]$.
Observe that $\dim(U\cap V^\perp)\leq \dim(U)\leq k$, by applying \Cref{lem:gaussian-smoothing-geometric} we get that
\begin{align*}
    \E_{\x\sim \normal}[(\Pi_{U+ V}\psi(\x)-\Pi_{V}\Pi_{U+V} \psi(\x))^2]&\leq 
    \dim(U\cap V^\perp) \max_{\vec v\in {U\cap V^\perp},\|\vec v\|_2=1}\E_{\x\sim \normal}[(\nabla \psi(\x)\cdot \vec v)^2]
     \\&=k\max_{\vec v\in U\cap V^\perp,\|\vec v\|_2=1}\vec v^\top\boldsymbol{\mathrm{Inf}}_\psi\vec v\;. 
\end{align*}
Furthermore, using that $\|\wh{\vec M}-\boldsymbol{\mathrm{Inf}}_\psi\|_2\leq \eta/2$, we have that
\[
\max_{\vec v\in U\cap V^\perp,\|\vec v\|_2=1}\vec v^\top\boldsymbol{\mathrm{Inf}}_\psi\vec v\leq \eta/2 +\max_{\vec v\in U\cap V^\perp,\|\vec v\|_2=1} \vec v^\top\wh{\vec M}\vec v\leq 2\eta\;, 
\]
where in the last inequality we used that $\vec v$ lies in the $V^{\perp}$ and for any $\vec v\in V^{\perp}$, it holds $\vec v^\top\wh{\vec M}\vec v\leq \eta$. By choosing $\eta=\eps^2/(32k)$, we have shown that
\[
\E_{\x\sim \normal}[(f(\x)-\Pi_V f(\x))\psi(\x)]\leq \eps\;.
\]
The proof then follows from \Cref{lem:correlating-convex-combinations}. It remains to bound the dimension of the subspace $V$. To this end, we use \Cref{lem:dim-of-subspace} which gives that $\dim(V)=O(\Psi^2/\eta^2)$. This completes the proof of \Cref{prop:dimension-reduction-geometric-concepts}.

\subsection{Proof of \Cref{thm:non-proper-geometric}}

For learning geometric concepts, we use the standard $L_1$-regression algorithm from \cite{KKMS:08}.
\begin{fact}[Theorem 9 \cite{KOS:08}]\label{fct:l1-regr} Let $\mathcal C$ be a class of Boolean functions in $\R^d$. Let $D$ be a distribution on $\R^d\times \{\pm 1\}$ such that the $\x$-marginal of $D$ is the standard $d$-dimensional normal. The $L_1$-regression algorithm draws $N=\poly(d^{\Gamma(\mathcal{C})^2/\eps^4},1/\eps,\log(1/\delta))$ samples from $D$ and runs in time $\poly(N,d)$ and outputs a polynomial $p:\R^d\mapsto\R$ so that with probability at least $1-\delta$ it holds
\[
\pr_{(\x,y)\sim D}[\sign(p(\x))\neq y]\leq \min_{f\in \mathcal C}\pr_{(\x,y)\sim D}[f(\x)\neq y]+\eps\;.
\]
\end{fact}
Let $\psi(\x)=T_\rho y(\x)$ with $\rho$ be less than a sufficiently small constant multiple of $\eps/\Gamma(f)$. Using $\poly(d,1/\eps,\log(1/\delta)$ queries, we calculate an estimation $\vec M$ of the influence matrix $\boldsymbol{\mathrm{Inf}}_\psi$  (\Cref{lem:estimation}). Let $V$ be the subspace spanned with the eigenvectors of $\vec M$ with eigenvalue at least $\eta$, where $\eta$ is a sufficiently small constant multiply of $\eps^2/k$.
Using \Cref{prop:dimension-reduction-geometric-concepts}, we have that $V$ has dimension at most $O(\Gamma^2k/\eps^4)$ and furthermore there exists $g$ with $\Gamma(g)\leq \Gamma$ so that $g(\x)=g(\proj_V \x)$ such that 
    \[
    \E_{\x \sim \normal}[|g(\x) -\psi(\x)|] 
    \leq  \E_{\x \sim \normal}[|f( \x) - \psi(\x)|]
    + \eps \,.
    \]
    From \Cref{lem:smoothing-boolean}, we have that it also holds
        \[
    \E_{\x \sim \normal}[|g(\x) -y(\x)|] 
    \leq  \inf_{f\in \mathfrak B(\Gamma,k)}\E_{\x \sim \normal}[|f( \x) - y(\x)|]
    + O(\eps) \,.
    \]
    Equivalently, we have that $   \pr_{\x \sim \normal}[g(\x) \neq y(\x)] 
    \leq  \inf_{f\in \mathfrak B(\Gamma,k)}\pr_{\x \sim \normal}[f( \x) \neq y(\x)]
    + O(\eps)$.
    Let $\vec P:\R^d\mapsto V$ be the projection matrix to the subspace $V$. Let $(\vec P\x,y)\sim D'$, where $(\x,y)\sim D$. We the $L_1$-regression algorithm on $D'$ and from \Cref{fct:l1-regr}, using $\poly((k\Gamma/\eps)^{\Gamma^2/\eps^4},1/\eps,\log(1/\delta))$ samples from $D'$, we get a polynomial $p:V\mapsto \R$ so that with probability at least $1-\delta$, it holds
    \[
\pr_{(\x,y)\sim D}[\sign(p(\vec P\x))\neq y]\leq \inf_{f\in \mathfrak B(\Gamma,k)}\pr_{(\x,y)\sim D}[f(\x)\neq y]+\eps\;.
\]
This completes the proof of \Cref{thm:non-proper-geometric}.

\subsection{Corollaries for Intersections of Halfspaces and PTFs}

Using \Cref{thm:non-proper-geometric}, we can show the following corollary for intersections of $k$ halfspaces:
\begin{corollary}\label{corr:non-proper-interction}
    Let $\mathcal C$ be the class of intersections $k$ halfspaces in $\R^d$. Let $D$ be a distribution on $\R^d\times \{\pm 1\}$ such that the $\x$-marginal of $D$ is the standard $d$-dimensional normal. There exists an algorithm that makes $N_q=\poly(d/\eps)$ queries and draws $N_s=\poly(d/\eps)+\poly((k/\eps)^{\log(k)/\eps^4},1/\eps,\log(1/\delta))$ samples from $D$ and runs in time $\poly(N_s,N_q,d)$ and outputs a polynomial $p:\R^d\mapsto\R$ so that with probability at least $1-\delta$ it holds
\[
\pr_{(\x,y)\sim D}[\sign(p(\x))\neq y]\leq \min_{f\in \mathcal C}\pr_{(\x,y)\sim D}[f(\x)\neq y]+\eps\;.
\]
\end{corollary}
\begin{proof}[Proof of \Cref{corr:non-proper-interction}]
For the proof, we need the following fact about the Gaussian surface area of the intersection of $k$ halfspaces.
    \begin{fact}[Theorem 20 of \cite{KOS:08}]\label{fct:intersection-surface} The surface area $\Gamma(f)$ of the intersection of $k$ halfspaces is at most $O(\sqrt{\log k})$.
    \end{fact}
    The proof follows from \Cref{thm:non-proper-geometric} and \Cref{fct:intersection-surface}.
\end{proof}

We show that we can use \Cref{thm:non-proper-geometric} to learn low-degree polynomial threshold functions (PTFs) that depend only on a small dimensional subspace.

\begin{corollary}\label{corr:non-proper-ptfs}
    Let $\mathcal C$ be the class of degree-$\ell$ PTFs in $\R^d$ that depend on an unknown $k$-dimensional subspace. Let $D$ be a distribution on $\R^d\times \{\pm 1\}$ such that the $\x$-marginal of $D$ is the standard $d$-dimensional normal. There exists an algorithm that makes $N_q=\poly(d/\eps)$ queries,  draws $N_s=\poly(d/\eps)+\poly((k/\eps)^{\ell/\eps^4},1/\eps,\log(1/\delta))$ samples from $D$, runs in time $\poly(N_s,N_q,d)$ and outputs a polynomial $p:\R^d\mapsto\R$ so that with probability at least $1-\delta$ it holds
\[
\pr_{(\x,y)\sim D}[\sign(p(\x))\neq y]\leq \min_{f\in \mathcal C}\pr_{(\x,y)\sim D}[f(\x)\neq y]+\eps\;.
\]
\end{corollary}
\begin{proof}[Proof of \Cref{corr:non-proper-ptfs}]
For the proof, we need the following fact about the Gaussian surface area of degree-$\ell$ PTFs.
    \begin{fact}[Gaussian Surface Area of PTFs, \cite{Kane11}]\label{fct:ptfs-surface} The surface area $\Gamma(f)$ of $\ell$-degree polynomial threshold functions is at most $O(\ell)$.
    \end{fact}
    The proof follows from \Cref{thm:non-proper-geometric} and \Cref{fct:ptfs-surface}.
\end{proof}

Finally, we show that we can use \Cref{thm:non-proper-geometric} to learn arbitrary functions of $\ell$ halfspaces.

\begin{corollary}\label{corr:non-proper-functions}
    Let $\mathcal C$ be the class of functions of $\ell$ halfspaces  in $\R^d$. Let $D$ be a distribution on $\R^d\times \{\pm 1\}$ such that the $\x$-marginal of $D$ is the standard $d$-dimensional normal. There exists an algorithm that makes $N_q=\poly(d/\eps)$ queries, draws $N_s=\poly(d/\eps)+\poly((\ell/\eps)^{\ell/\eps^4},1/\eps,\log(1/\delta))$ samples from $D$, runs in time $\poly(N_s,N_q,d)$ and outputs a polynomial $p:\R^d\mapsto\R$ so that with probability at least $1-\delta$ it holds
\[
\pr_{(\x,y)\sim D}[\sign(p(\x))\neq y]\leq \min_{f\in \mathcal C}\pr_{(\x,y)\sim D}[f(\x)\neq y]+\eps\;.
\]
\end{corollary}
\begin{proof}[Proof of \Cref{corr:non-proper-functions}]
We note that the Gaussian surface area of functions of $\ell$ halfspaces is bounded above by $\ell$. From \cite{KOS:08} (see, e.g., Fact 17), we have that the surface area of a Boolean function $f$ that depends on $\ell$ halfspaces, is bounded above by the sum of the surface area of the individual halfspaces; therefore, we have that $\Gamma(f)\leq O(\ell)$.  The proof follows from \Cref{thm:non-proper-geometric}.
\end{proof}

\section{Hardness of Agnostic Proper Learning of Halfspaces and ReLUs with Queries} \label{sec:hardness}

One might ask if the exponential dependence on $1/\eps$ in our upper bound (\Cref{intro-cor:agnostic-relu-mq,intro-cor:agnostic-ltf-mq})
is necessary or just an artifact of our algorithmic approach. In this section, 
we provide some evidence that it is inherent. Unfortunately, there are 
very few circumstances where one can prove computational lower bounds 
against improper learners with query access to the function. 
So our bounds will apply only to proper learners. 
The basic idea of our argument is that if $f(\x) = \sgn(\vec v\cdot \x)$ is a linear threshold function or $f(\x)=\mathrm{ReLU}(\vec v\cdot \x)$
with $\vec v$ a unit vector and $p(\x)$ a polynomial, 
then $\E[f(\x)p(\x)]$ will be a polynomial in $\vec v$. As approximately optimizing low-degree 
polynomials over the unit sphere is conjectured to be computationally hard, 
this will prove hardness for proper learning of linear threshold functions.
In particular, our hardness reduction starts from the small-set expansion problem~\cite{RaghavendraS10}.
We then rely on results of~\cite{BarakBHKSZ12} 
to reduce this problem to one about polynomial optimization. 
In particular we have:
\begin{theorem}\label{basic poly approx hardness theorem}
If there is a polynomial-time algorithm that given $\vec a_1,\vec a_2,\ldots,\vec a_n \in \R^d$ 
outputs a constant factor approximation to $\max_{\|\x\|_2=1} \frac{1}{n}\sum_{i=1}^n (\vec a_i\cdot \x)^4$, then there is a polynomial time algorithm for the small-set expansion problem.
\end{theorem}

We note here that $\max_{\|\x\|_2=1} \frac{1}{n}\sum_{i=1}^n (\vec a_i\cdot \x)^4$ 
is a homogeneous degree-$4$ polynomial. It will be important for our purposes 
that the polynomial in question have odd degree. 
Fortunately, we can reduce to this case.
\begin{corollary} \label{basic poly approx hardness cor}
If there is a polynomial-time algorithm that given a homogeneous degree-$5$ polynomial 
$p$ on $\R^d$ outputs a constant factor approximation to $\max_{\|\x\|_2=1} p(\x)$, 
then there is a polynomial-time algorithm for the small-set expansion problem.
\end{corollary}
\begin{proof}
We give a reduction to this problem from the problem 
in \Cref{basic poly approx hardness theorem}. 
In particular, given $\vec a_1,\ldots, \vec a_n\in\R^d$, 
we let $q(\x) =  \frac{1}{n}\sum_{i=1}^n (\vec a_i\cdot \x)^4$.
We then define the homogeneous degree-$5$ polynomial 
$p$ on $\R^{d+1}$ as $p(\x,y) = q(\x)y$ 
(where $x$ here represents the first $d$ coordinates of the input and $y$ represents the last one). 
We note that if $\|(\x,y)\|_2 = 1$, then $\|\x\|_2 = a$ and $y=b$ for some $a^2+b^2=1$. 
Letting $\x' = \x/a$ and using the homogeneity of $q$, 
we have that $p(\x,y) = a^4 b q(\x')$. For fixed $\x'$, the maximum of this 
over $a,b$ is obtained when $a = \sqrt{4/5}$ and $b=\sqrt{1/5}$. 
Thus, the maximum value of $p(\x,y)$ over the unit sphere 
equals the maximum value of $q(\x')$ over the unit sphere times $16/5^{2.5}$. 
Thus, finding a constant-factor approximation to the maximum value of one 
is equivalent to finding such an approximation of the other. This completes our proof.
\end{proof}

We are now ready to state our main theorem.

\begin{theorem}[Hardness of Proper Learning for LTFs]\label{proper lower bound theorem}
Suppose that there is an algorithm that given query access 
to a Boolean function $f$ on $\R^d$ runs in $\poly(d)$ time and approximates the minimum misclassification error between $f$ and a homogeneous LTF 
(with respect to the standard Gaussian distribution) to additive error $\eps$ 
for some $\eps < d^{-10}$. Then there is a polynomial-time algorithm 
for the small set expansion problem.
\end{theorem}
Before we prove \Cref{proper lower bound theorem}, 
we note that any proper agnostic learner can be used to approximate this error 
merely by approximating the error between $f$ and the learned function. 
Thus, this result will imply a lower bound for learning.

\begin{proof}
We assume throughout that $d$ is sufficiently large, 
as otherwise there is nothing to prove. We proceed by a reduction 
from the problem in \Cref{basic poly approx hardness cor}. 
In particular, let $p$ be a homogeneous degree-$5$ polynomial on $\R^d$. 
Let $\vec T$ be the unique symmetric tensor so that $p(\x) = \vec T(\x,\x,\x,\x,\x).$ 
By scaling $\vec T$, we may assume that $\|\vec T\|_2 = 1$. 
Let $q(\x) = (\vec T\cdot H(\x))$, where $H(\x)$ is the tensor whose entries 
are the degree-$5$ Hermite polynomials in $\x$.

Morally, we would like to take $f(\x) = q(\x)$. 
Unfortunately, this does not work for two reasons.

First, $f(\x)$ needs to be Boolean, while $q(\x)$ distinctly is not. 
We can fix this by taking $f$ to be a random function, 
where the expected value of $f(\x)$ equals $q(\x)$.

Unfortunately, this cannot work because the expected value of $f(\x)$ 
must still be in $[-1,1]$, while $q$ is unbounded. 
To solve this, we first scale $q$ down substantially and then truncate its extreme values.
To do this, we define:
$$
t(x) = 
\begin{cases} 
1 & \textrm{if } x > 1 \\ 
-1 & \textrm{if } x < -1 \\ 
x & \textrm{otherwise.} 
\end{cases}
$$
We then divide $\R^d$ into tiny boxes of side length $\delta$ 
for some very small $\delta.$ For each box $B$, we pick an $\x\in B$ 
and then (independently for each box) let $f$ be $1$ on $B$ 
with probability $(t(q(\x)/d)+1)/2$ and $-1$ on $B$ otherwise. 
We note that the expected value of $f$ on $B$ is $t(q(\x)/d)$, 
where $\x$ is the representative element. As the difference between 
$q$ at the representative element $\x$ of $B$ and at any other point 
in $B$ will be small if $\delta$ is 
(and if the box is not too far from the origin), 
it is not hard to see that the expectation over the randomness 
in defining $f$ of $|\Exn[f(\x)\sgn(\vec v\cdot \x)] - \Exn[t(q(\x)/d) \sgn(\vec v\cdot \x)]|$ 
goes to $0$ with $\delta$. As the variance of $ \Exn[f(\x)\sgn(\vec v\cdot \x)]$ 
also goes to $0$ with $\delta$, if we take $\delta$ sufficiently small, 
then with high probability over the randomness in $f$, we have that 
$|\Exn[f(\x)\sgn(\vec v\cdot \x)] - \Exn[t(q(\x)/d) \sgn(\vec v\cdot \x)]| < \eps/2$ 
for all unit vectors $\vec v$. 
Therefore, finding an $\eps$ additive approximation to the minimum 
misclassification error between $f$ and an LTF is equivalent 
to finding a $2\eps$-additive approximation to the maximum value 
of $\Exn[f(\x)\sgn(\vec v\cdot \x)]$, which in turn is sufficient to find 
an $\eps$-additive approximation of $ \Exn[t(q(\x)/d) \sgn(\vec v\cdot \x)]$. 
We will show that this is computationally hard.

To start with, we note that $\Exn[q(\x)^2] = \|\vec T\|_2 = 1$. 
Therefore, by standard concentration bounds, 
we have that $\pr_{\x\sim \normal}[|q(\x)| > d] = \exp(-\Omega(d^{2/5})) < \eps^3$. 
Therefore, by the Cauchy-Scwartz inequality, we have that
$$
\Exn[|q(\x)/d - t(q(\x))/d|] \leq \sqrt{\pr_{\x\sim \normal}(|q(\x)| > d) \Exn[q(G\x)^2]} \leq \eps/2 \;.
$$
Thus, if one can approximate the maximum value of $ \Exn[t(q(\x)/d) \sgn(\vec v\cdot \x)]$ 
to additive error $\eps$, one can approximate the maximum value of 
$\Exn[(q(\x)/d) \sgn(\vec v\cdot \x)]$ to additive error $\eps/2$. 
However, we can compute this expectation by comparing the Hermite expansions 
for $q(\x)/d$ and $\sgn(\vec v\cdot \x)$. In particular, the former only 
has non-vanishing terms in degree $5$, where they are given by the tensor $\vec T/d$. 
The latter has its degree-$5$ Hermite tensor given by $c_5 \vec v^{\otimes 5}$, 
where $c_5 = \E_{z\sim \normal}[h_5(z)\sgn(z)] = (3/2)\sqrt{1/(15\pi)}.$ 
Therefore, we have that $$
\Exn[(q(\x)/d) \sgn(\vec v\cdot \x) ] = (\vec T/d)\cdot (c_5 \vec v^{\otimes 5}) 
= (c_5/d) \vec T(\vec v,\vec v,\vec v,\vec v,\vec v) = (c_5/d) p(\vec v) \;.
$$
Thus, finding an $\eps/2$-additive approximation to the maximum value 
of $\Exn[(q(\x)/d) \sgn(\vec v\cdot \x)]$ for unit vectors $\vec v$ 
is equivalent to finding an $O(d^{-9})$-additive approximation 
to the maximum value of $p(\vec v)$ over unit vectors $\vec v$. 
We claim that doing this would give a constant-factor multiplicative 
approximation to the maximum value of $p(\vec v)$, 
finishing our reduction to the problem of \Cref{basic poly approx hardness cor}. 
To do this, we need to show that the maximum value of $p(\vec v)$ is much larger than $d^{-9}$.

To show this, we note that because $\|\vec T\|_2 = 1$, the sum of the squares 
of the entries of $\vec T$ is $1$. Since $\vec T$ has only $d^5$ entries, 
this means that it must have some entry with norm at least $d^{-5}$. 
Therefore, there must be unit vectors $\vec v_1,\vec v_2,\ldots,\vec v_5$ so that 
$\vec T(\vec v_1,\vec v_2,\vec v_3,\vec v_4,\vec v_5) \geq d^{-5}$. However, this value is proportional to $\sum_{\eps_1,\ldots,\eps_5 \in \{\pm 1\}} \eps_1\eps_2\cdots \eps_5 
p(\eps_1\vec v_1 + \eps_2\vec v_2 + \ldots + \eps_5\vec v_5).$ 
As each term here is proportional to $p$ of some unit vector 
(using the fact that $p$ is homogeneous), this implies that there is some unit vector $\vec v$ 
with $|p(\vec v)| \gg d^{-5}$. Replacing $\vec v$ by its negation if necessary, 
we have that the maximum value of $p(\vec v)$ over unit vectors $\vec v$ is $\Omega(d^{-5})$.
This completes our proof.
\end{proof}

\begin{theorem}[Hardness of Proper Learning for ReLUs]\label{proper lower bound theorem relu}
Suppose that there is an algorithm that given query access 
to a real-valued function $f$ on $\R^d$ runs in $\poly(d)$ time 
and approximates the minimum $L_2^2$ error between $f$ and a homogeneous ReLU 
(with respect to the standard Gaussian distribution) to additive error $\eps$ 
for some $\eps < d^{-4}$. Then there is a polynomial-time algorithm 
for the small set expansion problem.
\end{theorem}
\begin{proof}
    Let $p$ be a homogeneous degree-$4$ polynomial on $\R^d$. 
Let $\vec T$ be the unique symmetric tensor so that $p(\x) = \vec T(\x,\x,\x,\x).$ 
By scaling $\vec T$, we may assume that $\|\vec T\|_2 = 1$. 
Let $f(\x) = (\vec T\cdot H(\x))$, where $H(\x)$ is the tensor whose entries 
are the degree-$4$ Hermite polynomials in $\x$. 

 We can compute this expectation by comparing the Hermite expansions 
for $f(\x)$ and $\mathrm{ReLU}(\vec v\cdot \x)$. In particular, the former only 
has non-vanishing terms in degree $4$, where they are given by the tensor $\vec T$. 
The latter has its degree-$4$ Hermite tensor given by $c_4 \vec v^{\otimes 4}$, 
where $c_4 = \E_{z\sim \normal}[h_4(z)\mathrm{ReLU}(z)] = -(2\pi(24+\sqrt{2}))^{-1}$. 
Therefore, we have that $$
\Exn[f(\x) \mathrm{ReLU}(\vec v\cdot \x) ] = \vec T\cdot (c_4 \vec v^{\otimes 4}) 
= c_4 \vec T(\vec v,\vec v,\vec v,\vec v) = c_4 p(\vec v) \;.
$$

Thus, finding an $\eps$-additive approximation to the maximum value 
of $\Exn[ f(\x) \mathrm{ReLU}(\vec v\cdot \x)]$ for unit vectors $\vec v$ 
is equivalent to finding an $O(d^{-5})$-additive approximation 
to the maximum value of $p(\vec v)$ over unit vectors $\vec v$. 
We claim that doing this would give a constant-factor multiplicative 
approximation to the maximum value of $p(\vec v)$, 
finishing our reduction to the problem of \Cref{basic poly approx hardness cor}. 
To do this, we need to show that the maximum value of $p(\vec v)$ is much larger than $d^{-5}$.

To show this, we note that because $\|\vec T\|_2 = 1$, the sum of the squares 
of the entries of $\vec T$ is $1$. Since $\vec T$ has only $d^4$ entries, 
this means that it must have some entry with norm at least $d^{-4}$. 
Therefore, there must be unit vectors $\vec v_1,\vec v_2,\vec v_3,\vec v_4$ so that 
$\vec T(\vec v_1,\vec v_2,\vec v_3,\vec v_4) \geq d^{-4}$. However, this value is proportional to $\sum_{\eps_1,\ldots,\eps_4 \in \{\pm 1\}} \eps_1\eps_2\cdots \eps_4
p(\eps_1\vec v_1 + \eps_2\vec v_2 + \eps_3\vec v_3  + \eps_4\vec v_4).$ 
As each term here is proportional to $p$ of some unit vector 
(using the fact that $p$ is homogeneous), this implies that there is some unit vector $\vec v$ 
with $|p(\vec v)| \gg d^{-4}$. Replacing $\vec v$ by its negation if necessary, 
we have that the maximum value of $p(\vec v)$ over unit vectors $\vec v$ is $\Omega(d^{-4})$.
This completes our proof.
\end{proof}
 
 \bibliographystyle{alpha}
\bibliography{mydb}
\appendix

\section*{Appendix}

\section{Proper Agnostic Query Learner for LTFs} \label{sec:ltfs}

In this section we present our algorithmic result for agnostic proper learning of linear threshold functions with membership queries.
Our goal is to show \Cref{thm:agnostic-ltf-mq} which we state below.

\begin{theorem}[Proper Agnostic Query Learner for LTFs]\label{thm:agnostic-ltf-mq}
    Let $\mathcal{C}$ be the class of LTFs on $\R^d$ and denote by $y(\x) \in \{\pm 1\}$ the label (chosen by an adversary) 
    of $\x \in \R^d$. 
    There exists an algorithm that makes $N_s = \poly(d/\eps)$ sample queries, $N_q = \poly(d/\eps)$ queries,
    and, with runtime $\poly(d/\eps) + 2^{\poly(1/\eps)}$, computes an LTF $f(\x) : \R^d \mapsto \{\pm 1\}$ such that,
    with probability at least $1-\delta$, it holds 
\(
\pr_{\x \sim \normal}[h(\x) \neq y(\x)] \leq \inf_{ c \in \mathcal{C}}\pr_{\x \sim \normal}[c(\x) \neq y(\x)] + \epsilon \,.
\)
\end{theorem}

Our algorithm is presented in \Cref{alg:agnostic-proper}. It uses membership queries to estimate a matrix $M$ corresponding to the influence matrix of the appropriately smoothed label function $y(\x)$. It then restricts attention to a small subspace $V$ given by the large eigenvectors of $\vec M$ and exhaustively searches for a near-optimal halfspace with a normal vector in $V$.

\begin{Ualgorithm}
	\centering
	\fbox{\parbox{6in}{
			{\bf Input:}   $\eps>0$, $\delta>0$ and sample and query access to distribution $\D$\\
			{\bf Output:}  A hypothesis $h\in{\cal C}$ such as $\err_{0-1}^\D(h)\leq \min_{f\in {\cal C}}\err_{0-1}^\D(f)+\eps$ with probability $1-\delta$.
			\begin{enumerate}
				\item $\rho\gets C\eps^2$, $\eta\gets C\eps^2$, for $C>0$ sufficiently small constant. 
				\item Estimate $\vec M=\E_{\x\sim \D_\x}[ D_\rho y(\x) D_\rho y(\x)^\top]$ using $\poly(d/\eps)$ queries using \Cref{alg:membership-est}.
				\item Let $V$ be the subspace spanned by the eigenvectors of $\vec M$ whose eigenvalues are at least $\eta$.
\item Let ${\cal H_V}$ be the set of LTFs with normal vectors in $V$. Compute the ERM hypothesis $h \in {\cal H_V}$ using $m = \Theta(\frac{\dim( V )}{\eps^2}\log(1/\delta))$ i.i.d.\ samples
from $\D$ in time $O(m^{\dim( V )})$.\label{alg:emprical-outputs}
\item $\textbf{return } h$.
			\end{enumerate}
			
	}}
 \medskip
	\caption{Agnostic Proper Learning of LTFs with Membership Queries}
	\label{alg:agnostic-proper}
\end{Ualgorithm}

\subsection{Reducing the Dimension via Influence PCA}
\label{ssec:dimension-reduction}
In this section we show our main dimension-reduction tool.
We prove that the subpspace corresponding to eigvenvectors with large (i.e., larger than $\poly(\eps)$) eigenvalues contains 
the normal-vector of an approximately optimal halfspace.
\begin{proposition}[Dimension Reduction via Influence PCA: LTFs]\label{prop:dimension-reduction}
    Fix $\eps > 0$ and let $\psi : \R^d \mapsto [-1,1]$
    be a differentiable function with $\|\nabla \psi(\x)\|_2 \leq \Psi$ for all $\x \in \R^d$.
    Let $\eta$ be a sufficiently small multiple of $\eps^2$ and let $\wh{\vec M} \in \R^{d \times d}$ be such that 
    $\| \wh {\vec M} -  \boldsymbol{\mathrm{Inf}}_{\psi} \|_2 \leq \eta/2$. 
    Moreover, let $V$ be the subspace spanned by 
    all the eigenvectors of $\wh {\vec M}$ whose corresponding eigenvalues are at least $\eta$.  The following hold true:
    \begin{enumerate}
    \item 
    There exists $\vec v \in V$ and $t\in\R$ such that 
    \[
    \E_{\x \sim \normal}[|\sgn(\vec v \cdot \x+t) - \psi(\x)|] 
    \leq \inf_{\vec w \in \R^d,t\in \R} \E_{\x \sim \normal}[
    |\sgn(\vec w \cdot \x+t) - \psi(\x)|] + \eps \,.
    \]
    \item The dimension of $V$ is at most $O(\Psi^2/\eta)$.
\end{enumerate}
    \end{proposition}

\begin{proof}
Suppose for the sake of contradiction that there exists a halfspace $f\in \cal C$ such that for every halfspace
$ g \in {\cal C}_{V}$, it holds
\begin{equation}\label{eq:contradiction}
\E_{\x\sim \normal}[(f(\x)- g(\x))\psi(\x)]\geq \eps\;.
\end{equation}

We can write $f(\x)= \sign(\vec w \cdot \x + t) = \sign(\vec w_{V} \cdot \x + \vec w_{V^\perp} \cdot \x +t)$.
Note that $\vec w_{V^\perp} \neq \vec 0$, since otherwise we
would have $f \in {\cal C}_V$.
For simplicity, we denote $\vec \subvector = \vec w_{V^\perp}/\|\vec w_{V^\perp}\|_2$.  Notice that
the direction $\vec \subvector$ has low influence, since $\vec \subvector \in V^\perp$.
Recall that by $\normal_{\vec \subvector}$ we denote the projection of $\normal$ onto the (one-dimensional) subspace
spanned by $\vec \subvector$.
We define $f_{V}(\x)=\E_{\vec z \sim \normal_{\vec \subvector }}[f(\vec z + \x_{V})]$ to be
a convex combination of halfspaces in ${\cal C}_V$.  In particular, $f_V(\x)$ is a smoothed version
of the halfspace $\sgn(\vec w_V \cdot \x + t)$ whose normal vector belongs in $V$.
Moreover, we define $\psi_{V}(\x)=\E_{\vec z \sim \normal_{\vec \subvector }}[\psi(\vec z + \x_{V})]$. By adding and subtracting $\psi_{V}(\x)$, we get the following
\begin{align*}
    \E_{\x\sim \normal}[(f(\x)-f_V(\x))\psi(\x)]&=\E_{\x\sim \normal}[(f(\x)-f_V(\x))(\psi(\x)-\psi_V(\x))]+\E_{\x\sim \normal}[(f(\x)-f_V(\x))\psi_V(\x)]
    \\&=\E_{\x\sim \normal}[(f(\x)-f_V(\x))(\psi(\x)-\psi_V(\x))]\;,
\end{align*}
 where the last equality holds because the term $\psi_V(\x)$ does not depend on the direction $\subvector$, therefore, 
 \begin{align*}
      \E_{\x\sim \normal}[(f(\x)-f_V(\x))\psi_V(\x)]&= \E_{\vec x \sim \normal_{\vec \subvector^\perp  }}\left[\E_{\vec z \sim \normal_{\vec \subvector}}[(f(\x +\vec z)-f_V(\x))\psi_V(\x)]\right]\\&=\E_{\vec \x \sim \normal_{\vec \subvector^\perp }}\left[(f_V(\x)-f_V(\x))\psi_V(\x)\right]=0\;.
 \end{align*}
It remains to bound the term  $\E_{\x\sim \normal}[(f(\x)-f_V(\x))(\psi(\x)-\psi_V(\x))]$. From \CS\ inequality, we have
\[
\E_{\x\sim \normal}[(f(\x)-f_V(\x))(\psi(\x)-\psi_V(\x))]\leq 
\sqrt{2 \E_{\x\sim \normal}[(\psi(\x)-\psi_V(\x))^2]}\;.
\]
First, we show the following to bound this term

\begin{lemma}[Gaussian Marginalization Error]\label{lem:gaussian-smoothing}
Let $g: \R^d \mapsto \R$ be a function in $L^2(\normal)$ such that $\nabla g \in L^2(\normal)$.
Fix some direction $\vec v$ with $\|\vec v \|_2 = 1$ and define 
$r(\x) = \E_{z \sim \normal} [g(\proj_{\vec v^\perp}(\x) + z \vec v)]$.
Then it holds 
\[
\E_{\x \sim \normal}[(g(\x) - r(\x))^2] \leq 
\E_{\x \sim \normal}[(\nabla g(\x) \cdot \vec v)^2] \,.
\]
\end{lemma}
\begin{proof}
Using the rotation invariance of the Gaussian distribution, without loss of generality, we may 
assume that $\vec v = \vec e_1$.
We have 
\[
\E_{\x \sim \normal}[(g(\x) - r(\x))^2] =
\E_{\x \sim \normal}[(g(\x) - \E_{z \sim \normal}[r(z, \x_2,\ldots, \x_d)])^2] 
=
\E_{\x_2,\ldots, \x_d \sim \normal^{d-1}} [ \var_{\x_1 \sim \normal} [g(\x_1, \ldots, \x_d)] ]\,.
\]
To bound the variance in the above expression, we can use Poincare's inequality.
\begin{fact}[Gaussian Poincare Inequality]
\label{lem:poincare}
Let $g: \R \mapsto \R$ be a function in $L^2(\normal)$ such that $g' \in L^2(\normal)$.
Then it holds 
\(
\var_{z \sim \normal}[g(z)] \leq \E_{z \sim \normal}[|g'(z)|^2] \,.
\)
\end{fact}
Using \Cref{lem:poincare}, we have
$\E_{\x \sim \normal}[(g(\x) - r(\x))^2] \leq 
\E_{\x_2,\ldots, \x_d \sim \normal^{d-1}} [ \E_{\x_1 \sim \normal} [(\nabla g(\x) \cdot \vec e_1)^2]
= 
\E_{\x\sim \normal^{d}} [(\nabla g(\x) \cdot \vec e_1)^2] $.
\end{proof}
Therefore, from \Cref{lem:gaussian-smoothing}, we get that
\[
\E_{\x\sim \normal}[(f(\x)-f_V(\x))(\psi(\x)-\psi_V(\x))]\leq 2\sqrt{\E_{\x\sim \normal}[(\nabla \psi(\x)\cdot \vec \subvector)^2]}=2\sqrt{\vec \subvector^\top\boldsymbol{\mathrm{Inf}}_\psi\vec \subvector}\;.
\]
Furthermore, using that $\|\wh{\vec M}-\boldsymbol{\mathrm{Inf}}_\psi\|_2\leq \eta/2$, we have that
\[
\E_{\x\sim \normal}[(f(\x)-f_V(\x))(\psi(\x)-\psi_V(\x))]\leq 2\sqrt{\eta/2 +\vec \subvector^\top\wh{\vec M}\vec \subvector}\leq 4\sqrt{\eta}\;.
\]
where in the last inequality we used that $\subvector$ lies in the $V^{\perp}$ and for any $\vec v\in V^{\perp}$, it holds $\E_{\x\sim \normal}[(\vec v^\top\wh{\vec M} \vec v)^2]\leq \eta$. By choosing $\eta=\eps^2/32$, we have
\[
\E_{\x\sim \normal}[(f(\x)-g(\x))\psi(\x)]\leq 2\sqrt{\eta}\leq \eps/2\;,
\]
which from \Cref{eq:contradiction}, we get a contradiction. 
An application of \Cref{lem:dim-of-subspace} completes the proof of \Cref{prop:dimension-reduction}.
\end{proof}

We require the following standard fact showing the existence of a small $\eps$-cover of halfspaces for of the set $V$.

\begin{fact}[see, e.g., Corollary 4.2.13 of \cite{Ver18}]\label{fct:cover}
For any $\eps>0$, there exists an explicit $\eps$-cover $\mathcal H$ of halfspaces over the $\R^d$ over the $L_1$ norm with respect the Gaussian distribution of size $\poly(1/\eps)^d$, i.e., for any $f(\x)=\sign(\vec w\cdot\x+t)$ there exists $f'\in \mathcal H$ so that $\Exn[|f(\x)-f'(\x)|]\leq \eps$.
\end{fact}

We are now ready to prove the main theorem of this section.

\begin{proof}[Proof of \Cref{thm:agnostic-ltf-mq}]
We first show that there is a set $\cal H$ of size $(1/\eps)^{O(1/\eps^6)}$
which contains tuples $(\vec u, t)$ with $\vec u\in \R^d$ and $t\in \R$,
such that
\[
  \pr_{(\x,y)\sim \D}[\sign(\vec u\cdot \x +t)\neq y]
  \leq \inf_{f\in \cal{C}}\pr_{(\x,y)\sim \D}[f(\x)\neq y] +\eps\;.
\]
First note we can assume that $1/\eps^6\leq d$, since otherwise
one can directly do a brute-force search over an $\eps$-cover
of the $d$-dimensional unit ball: we do not need to perform
our dimension-reduction process.  The runtime to perform this brute-force search
will be $(1/\eps)^{O(d)} \log(1/\delta)$ which, by the assumption that $1/\eps^6 > d$, is smaller
than $(1/\eps)^{O(1/\eps^6)} \log(1/\delta)$.

Let $f\in \cal C$ be such that the $\E_{(\x,y)\sim \D}[|f(\x)-y(\x)|]$ is minimized. 
Let $\psi(\x)=T_\rho y(\x)$ for $\rho=\Theta(\eps^2)$. Note that from \Cref{fct:ou-smooth}, we have that $\|\nabla \psi(\x)\|_2\leq 1/\rho$. From \Cref{lem:estimation} for 
$N=\poly(d/\eps)\log(1/\delta)$ queries, we get 
with probability at least $1-\delta/2$, a matrix $\vec M$ so that $\|\vec M-\boldsymbol{\mathrm{Inf}}_{\psi}\|_F\leq \eps^2$.  Applying \Cref{prop:dimension-reduction} to
the matrix $\vec M$, we get that the subspace $V$ spanned by the eigenvectors
of the matrix $\vec M$ with
eigenvalues larger than $\eta=\Theta(1/\eps^2)$ contains a vector $\vec v\in V$, such that
\begin{equation*}
\min_{t\in \R} \E_{(\x,y)\sim \D}[|\psi(\x)-\sign(\vec v\cdot \x+t)|]\leq \E_{(\x,y)\sim \D}[|\psi(\x)-f(\x)|]+\eps\;.
\end{equation*}
Moreover, by \Cref{prop:dimension-reduction}, the dimension of $V$ is $O(1/\eps^6)$. 
Applying \Cref{fct:cover}, we get that there exists an $\eps$-cover $\mathcal H$ of halfspaces in $V$ of size $(1/\eps)^{O(1/\eps^6)}$, so that for any $f(\x)=\sign(\vec w\cdot\x+t)$ with $\vec w\in V$ there exists $f'\in \mathcal H$ so that $\Exn[|f(\x)-f'(\x)|]\leq \eps$. Hence, there exists $f'\in \mathcal H$ so that
\[
\min_{f'\in \mathcal H} \E_{(\x,y)\sim \D}[|\psi(\x)-f'(\x)|]\leq \E_{(\x,y)\sim \D}[|\psi(\x)-f(\x)|]+2\eps\;.
\]
Furthermore, from \Cref{lem:smoothing-boolean} and the fact that for halfspaces, $\Gamma(f)=O(1)$ (see, e.g., \cite{KOS:08}), we have that it also holds
\[
\min_{f'\in \mathcal H} \E_{(\x,y)\sim \D}[|y(\x)-f'(\x)|]\leq \E_{(\x,y)\sim \D}[|y(\x)-f(\x)|]+O(\eps)\;.
\]

To complete the proof, we show that Step~\ref{alg:emprical-outputs} of \Cref{alg:agnostic-proper} 
outputs the correct hypothesis. From ERM for halfspaces,
it follows that $O(\frac{1}{\eps^2}\log({\cal H}/\delta))$
samples are sufficient to guarantee that the excess error of the chosen hypothesis is at most $\eps$
with probability at least $1-\delta/2$.
To bound the runtime of the algorithm, we note that the exhaustive search 
over an $\eps$-cover takes time $(1/\eps)^{O(1/\eps^6)} \log(1/\delta)$.
Thus, the total runtime of our algorithm in the case
where $1/\eps^6 \leq d$ is
\[
  \Big( \poly(d/\eps) + (1/\eps)^{O(1/\eps^6)} \Big) \log(1/\delta) \,.
\]
This completes the proof of \Cref{thm:agnostic-ltf-mq}.
\end{proof}

\section{Proper Agnostic Query Learner for ReLUs} \label{sec:relu}

In this section we present our algorithmic result for agnostic proper learning of ReLU functions with queries.
Our goal is to show \Cref{thm:agnostic-relu-mq} which we state below.

\begin{theorem}[Proper Agnostic Query Learner for ReLUs]\label{thm:agnostic-relu-mq}
    Let $\mathcal{C}$ be the class of ReLUs on $\R^d$ with normal vectors bounded by $M>0$ and denote by $y(\x) \in \R$ the label (chosen by an adversary) 
    of $\x \in \R^d$. 
    There exists an algorithm that makes $N_s = \poly(d/\eps)\log(1/\delta)$ sample queries, $N_q = \poly(d/\eps)\log(1/\delta)$ queries,
    and, with runtime $\poly(d) 2^{\poly(1/\eps)}$, computes a ReLU activation $f(\x)$ such that,
    with probability at least $1-\delta$, it holds 
\(
\E_{\x \sim \normal}[(h(\x) - y(\x))^2] \leq \inf_{ c \in \mathcal{C}}\E_{\x \sim \normal}[(c(\x) -y(\x))^2] + \epsilon \,.
\)
\end{theorem}

Our algorithm is presented in \Cref{alg:agnostic-proper-relu}. It uses membership queries to estimate a matrix $M$ corresponding to the influence matrix of the appropriately smoothed label function $y(\x)$. It then restricts attention to a small subspace $V$ given by the large eigenvectors of $\vec M$ and exhaustively searches for a near-optimal halfspace with a normal vector in $V$.

\begin{Ualgorithm}
	\centering
	\fbox{\parbox{6in}{
			{\bf Input:}   $\eps>0$, $\delta>0$ and sample and query access to distribution $\D$\\
			{\bf Output:}  A hypothesis $h\in{\cal C}$ such as $\E_{\x \sim \normal}[(h(\x) - y(\x))^2] \leq \inf_{ c \in \mathcal{C}}\E_{\x \sim \normal}[(c(\x) -y(\x))^2] + \epsilon $ with probability $1-\delta$.
			\begin{enumerate}
				\item $\rho\gets C\eps^2$, $\eta\gets C\eps^2$, for $C>0$ sufficiently small constant. 
				\item Estimate $\vec M=\E_{\x\sim \D_\x}[ D_\rho y(\x) D_\rho y(\x)^\top]$ using $\poly(d/\eps)$ queries using \Cref{alg:membership-est}.
				\item Let $V$ be the subspace spanned by the eigenvectors of $\vec M$ whose eigenvalues are at least $\eta$.
\item Let ${\cal H_V}$ be the set of ReLUs with normal vectors in $V$. Compute the ERM hypothesis $h \in {\cal H_V}$ using $m = \Theta(M\frac{\dim( V )}{\eps^2}\log(1/\delta))$ i.i.d.\ samples
from $\D$ in time $O(m^{\dim( V )})$.\label{alg:emprical-outputs-relu}
\item $\textbf{return } h$.
			\end{enumerate}
			
	}}
 \medskip
	\caption{Agnostic Proper Learning of ReLU Activations with Queries}
	\label{alg:agnostic-proper-relu}
\end{Ualgorithm}

\subsection{Reducing the Dimension via Influence PCA}
\label{ssec:dimension-reduction-relu}
\begin{proposition}[Dimension Reduction via Influence PCA: ReLU]\label{prop:dimension-reduction0relu}
    Fix $\eps > 0,M>0$ and let $\psi : \R^d \mapsto \R$
    be a differentiable function with $\|\nabla \psi(\x)\|_2 \leq \Psi$ for all $\x \in \R^d$.
    Let $\eta$ be a sufficiently small multiple of $\eps^2/M$ and let $\wh{\vec M} \in \R^{d \times d}$ be such that 
    $\| \wh {\vec M} -  \boldsymbol{\mathrm{Inf}}_{\psi} \|_2 \leq \eta/2$. 
    Moreover, let $V$ be the subspace spanned by 
    all the eigenvectors of $\wh {\vec M}$ whose corresponding eigenvalues are at least $\eta$.  The following hold true:
    \begin{enumerate}
    \item 
    There exists $\vec v \in V$ with $\|\vec v\|_2\leq M$ and $t\in\R$ such that 
    \[
    \E_{\x \sim \normal}[(\mathrm{ReLU}(\vec v \cdot \x+t) - \psi(\x))^2] 
    \leq \inf_{\vec w \in \R^d,\|\vec w\|_2\leq M,t\in \R} \E_{\x \sim \normal}[
    (\mathrm{ReLU}(\vec w \cdot \x+t) - \psi(\x))^2] + \eps \,.
    \]
    \item The dimension of $V$ is at most $O(\Psi^2/\eta)$.
\end{enumerate}
    \end{proposition}

\begin{proof}
Suppose for the sake of contradiction that there exists a $\vec w\in  \R^d,t\in \R$ such that for every $\vec v\in V,t'\in \R$, it holds
\begin{equation*}
\E_{\x\sim \normal}[(\mathrm{ReLU}(\vec v\cdot\x+t')- \psi(\x))^2]\geq \E_{\x\sim \normal}[(\mathrm{ReLU}(\vec w\cdot\x+t)- \psi(\x))^2]+\eps\;,
\end{equation*}
the above is equivalent to
\begin{equation}\label{eq:contradiction-relu}
2\E_{\x\sim \normal}[(\mathrm{ReLU}(\vec w\cdot\x+t)-\mathrm{ReLU}(\vec v\cdot\x+t'))\psi(\x))]\geq \eps+\E_{\x\sim \normal}[(\mathrm{ReLU}(\vec w\cdot\x+t))^2]-\E_{\x\sim \normal}[(\mathrm{ReLU}(\vec v\cdot\x+t'))^2]\;,
\end{equation}

Let $f(\x)=\mathrm{ReLU}(\vec w\cdot\x+t)$ and  $f_{V}(\x)=\Pi_Vf(\x)$.
Note that $\vec w_{V^\perp} \neq \vec 0$, since otherwise we
would have that the vector of $f$ would be inside $V$. Note that $f_V$ is a convex combination of ReLUs with vectors in $V$, hence
 \Cref{eq:contradiction-relu} becomes
\begin{equation}\label{eq:contradiction-relu2}
2\E_{\x\sim \normal}[(f(\x)-\Pi_Vf(\x))\psi(\x))]\geq \eps\;,
\end{equation}
where we used that $\E_{\x \sim \normal}[
    (f(\x)^2]=\E_{\x \sim \normal}[
    (\mathrm{ReLU}(\vec w_V \cdot \x+\vec w_{V^\perp}\cdot \x+t))^2]$, (by adding all the convex combinations of $\Pi_V$).
Moreover, we define $\psi_{V}(\x)=\E_{\vec z \sim \normal_{\vec \subvector }}[\psi(\vec z + \x_{V})]$. By adding and subtracting $\psi_{V}(\x)$, we get the following
\begin{align*}
    \E_{\x\sim \normal}[(f(\x)-f_V(\x))\psi(\x)]&=\E_{\x\sim \normal}[(f(\x)-f_V(\x))(\psi(\x)-\psi_V(\x))]+\E_{\x\sim \normal}[(f(\x)-f_V(\x))\psi_V(\x)]
    \\&=\E_{\x\sim \normal}[(f(\x)-f_V(\x))(\psi(\x)-\psi_V(\x))]\;,
\end{align*}
 where the last equality holds because the term $\psi_V(\x)$ does not depend on the directions inside $V^{\perp}$.
It remains to bound the term  $\E_{\x\sim \normal}[(f(\x)-f_V(\x))(\psi(\x)-\psi_V(\x))]$. From \CS\ inequality, we have
\[
\E_{\x\sim \normal}[(f(\x)-f_V(\x))(\psi(\x)-\psi_V(\x))]\leq 
\sqrt{2M \E_{\x\sim \normal}[(\psi(\x)-\psi_V(\x))^2]}\;.
\]
From \Cref{lem:gaussian-smoothing}, we get that
\[
\E_{\x\sim \normal}[(f(\x)-f_V(\x))(\psi(\x)-\psi_V(\x))]\leq \sqrt{2M\E_{\x\sim \normal}[(\nabla \psi(\x)\cdot \vec \subvector)^2]}=\sqrt{2M\vec \subvector^\top\boldsymbol{\mathrm{Inf}}_\psi\vec \subvector}\;.
\]
Furthermore, using that $\|\wh{\vec M}-\boldsymbol{\mathrm{Inf}}_\psi\|_2\leq \eta/2$, we have that
\[
\E_{\x\sim \normal}[(f(\x)-f_V(\x))(\psi(\x)-\psi_V(\x))]\leq \sqrt{2M}\sqrt{\eta/2 +\vec \subvector^\top\wh{\vec M}\vec \subvector}\leq 2\sqrt{M\eta}\;.
\]
where in the last inequality we used that $\subvector$ lies in the $V^{\perp}$ and for any $\vec v\in V^{\perp}$, it holds $\E_{\x\sim \normal}[(\vec v^\top\wh{\vec M} \vec v)^2]\leq \eta$. By choosing $\eta=\eps^2/(M32)$, we have
\[
\E_{\x\sim \normal}[(f(\x)-g(\x))\psi(\x)]\leq 2\sqrt M\sqrt{\eta}\leq \eps/2\;,
\]
which from \Cref{eq:contradiction-relu2}, we get a contradiction. 
An application of \Cref{lem:dim-of-subspace} completes the proof.
\end{proof}

We require the following standard fact showing the existence of a small $\eps$-cover of ReLU activations for of the set $V$.

\begin{fact}[see, e.g., Corollary 4.2.13 of \cite{Ver18}]\label{fct:cover-relu}
For any $\eps,M>0$, there exists an explicit $\eps$-cover $\mathcal H$ of ReLUs over the $\R^d$ over the $L_2$ norm with respect the Gaussian distribution of size $\poly(M/\eps)^d$.
\end{fact}

We are now ready to prove the main theorem of this section.

\begin{proof}[Proof of \Cref{thm:agnostic-relu-mq}]
We first show that there is a set $\cal H$ of size $(M/\eps)^{\poly(1/\eps)}$
which contains tuples $(\vec u, t)$ with $\vec u\in \R^d$ and $\|\vec u\|_2\leq M$ and $t\in \R$,
such that
\[
  \E_{(\x,y)\sim \D}[(\mathrm{ReLU}(\vec u\cdot \x +t)- y)^2]
  \leq \inf_{f\in \cal{C}}\E_{(\x,y)\sim \D}[(f(\x)- y)^2] +\eps\;.
\]
First note we can assume that $1/\eps^6\leq d$, since otherwise
one can directly do a brute-force search over an $\eps$-cover
of the $d$-dimensional $M$-ball.  The runtime to perform this brute-force search
will be $(M/\eps)^{O(d)} \log(1/\delta)$ which, by the assumption that $\poly(1/\eps)> d$, is smaller
than $(M/\eps)^{\poly(1/\eps)} \log(1/\delta)$.

Let $f\in \cal C$ be such that the $\E_{(\x,y)\sim \D}[(f(\x)-y(\x))^2]$ is minimized. 

Similar with the proof of \Cref{intro-thm:non-proper-real-valued}, we can assume that $|y(\x)|\leq \valb^ {1/2}/\eps^{1/2}$ as this does not increase the error by a lot.
 Let $\psi(\x)=T_\rho y$ for $\rho=\poly(\eps/(\valb))$. Note that $\|\nabla \psi(\x)\|_2\leq M'$. From \Cref{lem:estimation}, with
$N=\poly(d/\eps)\log(1/\delta)$ queries, we get that
with probability $1-\delta/2$ a matrix $\vec M$, so that $\|\vec M-\boldsymbol{\mathrm{Inf}}^{}_\psi\|_F\leq \eps$.  Applying \Cref{prop:dimension-reduction0relu} to
the matrix $\vec M$, we get that in the subspace $V$ spanned by the eigenvectors
of the matrix $\vec M$ with
eigenvalues larger than $\eta=\poly(\eps/\valb ))$ with dimension at most $O(\poly(M',1/\eta,1/\eps))$, there exists a $\mathrm{ReLU}$ activation $h$ with vector lying in the subspace $V$ so that

\begin{equation*}
\min_{t\in \R} \E_{(\x,y)\sim \D}[(\psi(\x)-\mathrm{ReLU}(\vec v\cdot \x+t))^2]\leq \E_{(\x,y)\sim \D}[(\psi(\x)-f(\x))^2]+\eps\;.
\end{equation*}
Applying \Cref{fct:cover-relu}, we get that there exists an $\eps$-cover $\mathcal H$ of halfspaces in $V$ of size $(M/\eps)^{\poly(1/\eps)}$, so that for any $f(\x)=\mathrm{ReLU}(\vec w\cdot\x+t)$ with $\vec w\in V$ there exists $f'\in \mathcal H$ so that $\Exn[(f(\x)-f'(\x))^2]\leq \eps$. Hence, there exists $f'\in \mathcal H$ so that
\[
\min_{f'\in \mathcal H} \E_{(\x,y)\sim \D}[(\psi(\x)-f'(\x))^2]\leq \E_{(\x,y)\sim \D}[(\psi(\x)-f(\x))^2]+2\eps\;.
\]
Furthermore, from \Cref{lem:smoothing-real} we have that it also holds
\[
\min_{f'\in \mathcal H} \E_{(\x,y)\sim \D}[(y(\x)-f'(\x))^2]\leq \E_{(\x,y)\sim \D}[(y(\x)-f(\x))^2]+O(\eps)\;.
\]

To complete the proof, we show that Step~\ref{alg:emprical-outputs-relu} of \Cref{alg:agnostic-proper-relu} 
outputs the correct hypothesis. From ERM
it follows that $O(\frac{M}{\eps^2}\log({\cal H}/\delta))$
samples are sufficient to guarantee that the excess error of the chosen hypothesis is at most $\eps$
with probability at least $1-\delta/2$.
To bound the runtime of the algorithm, we note that the exhaustive search 
over an $\eps$-cover takes time $(M/\eps)^{\poly(1/\eps)} \log(1/\delta)$.
Thus, the total runtime of our algorithm in the case
where $\poly(1/\eps) \leq d$ is
\[
  \Big( \poly(d/\eps) + d(1/\eps)^{\poly(1/\eps)} \Big) \log(1/\delta) \,.
\]
This completes the proof of \Cref{thm:agnostic-relu-mq}.
\end{proof}

\end{document}